\PassOptionsToPackage{table,dvipsnames}{xcolor}
\documentclass{article}

\usepackage{iclr2027_conference,times}

\usepackage[utf8]{inputenc}
\usepackage[T1]{fontenc}
\usepackage{url}
\usepackage{amsfonts}
\usepackage{nicefrac}
\usepackage{microtype}

\usepackage{adjustbox}
\usepackage{soul}
\usepackage{subfigure}
\usepackage{capt-of}
\usepackage{wrapfig}
\usepackage{multirow}
\usepackage[table]{xcolor}
\usepackage{algpseudocode}
\usepackage{algorithm}
\usepackage{amssymb}
\usepackage{mathtools}
\usepackage{amsthm}
\usepackage{thmtools}
\usepackage{thm-restate}
\usepackage{todonotes}
\usepackage{tcolorbox}
\usepackage[pdfusetitle]{hyperref}
\usepackage[capitalize,noabbrev]{cleveref}

\makeatletter
\newcounter{Halgorithmic}
\AtBeginEnvironment{algorithmic}{\stepcounter{Halgorithmic}}
\providecommand{\theHALG@line}{}
\renewcommand{\theHALG@line}{\arabic{Halgorithmic}.\arabic{ALG@line}}
\makeatother

\theoremstyle{plain}
\newtheorem{theorem}{Theorem}[section]
\newtheorem{proposition}[theorem]{Proposition}

\theoremstyle{definition}

\theoremstyle{remark}

\usepackage[dvipsnames]{xcolor}

\usepackage{graphicx}
\usepackage{graphics}
\usepackage{booktabs}
\usepackage{pdflscape}
\usepackage{array}
\usepackage{siunitx}
\newcolumntype{H}{>{\setbox0=\hbox\bgroup}c<{\egroup}@{}} %

\usepackage{amsmath,amsfonts,bm}

\newcommand{\abs}[1]{\vert {#1} \vert}

\newcommand{\prob}[1]{\mathbb P\left[#1 \right]}
\DeclareMathOperator{\EX}{\mathbb{E}}%

\mathchardef\mhyphen="2D

\counterwithin*{theorem}{section}

\theoremstyle{definition}

\theoremstyle{definition}  

\crefname{axiom}{Axiom}{Axioms}

\def\eqref#1{equation~\ref{#1}}

\def\1{\bm{1}}

\DeclareMathAlphabet{\mathsfit}{\encodingdefault}{\sfdefault}{m}{sl}
\SetMathAlphabet{\mathsfit}{bold}{\encodingdefault}{\sfdefault}{bx}{n}

\DeclarePairedDelimiterX{\infdivx}[2]{[}{]}{%
  #1\;\delimsize\|\;#2%
}

\usepackage{xkcdcolors}

\crefname{figure}{Fig.}{Figs.}
\crefname{table}{Tab.}{Tabs.}
\crefname{section}{\S}{\S\S}
\crefname{chapter}{Ch.}{Chs.}
\crefname{appendix}{App.}{Apps.}
\crefname{equation}{Eq.}{Eqs.}
\crefname{theorem}{Thm.}{Thms.}
\crefname{proposition}{Prop.}{Props.}
\Crefname{proposition}{Prop.}{Props.}

\usepackage{xcolor}
\usepackage{soul}

\definecolor{bgOrange}{RGB}{254, 230, 200}
\definecolor{bgBlue}{RGB}{204, 229, 255}
\definecolor{bgGreen}{RGB}{218, 242, 219}

\newcommand{\hlOrange}[1]{\sethlcolor{bgOrange}\hl{#1}}
\newcommand{\hlBlue}[1]{\sethlcolor{bgBlue}\hl{#1}}
\newcommand{\hlGreen}[1]{\sethlcolor{bgGreen}\hl{#1}}

\def\states{{\mathcal{S}}}
\def\initState{s_I}
\def\finalState{s_F}
\def\trajectory{{\tau}}
\def\trajDist{{q}}
\def\noiseDist{p_\text{\mdseries src}}
\def\dataDist{p_\text{\mdseries tgt}}
\def\coupling{{\gamma}}
\def\markovDist{{\kappa}}
\def\actionSet{{\alpha}}

\newcommand{\indicator}[1]{\mathbf{1}[#1]}
\newcommand{\smath}[1]{\ensuremath{\smash{#1}}}
\newcommand{\rate}{u_t}
\newcommand{\rbrate}{\bar{u}_t}
\newcommand{\loss}{\mathcal{L}}
\DeclareMathOperator{\supp}{supp}

\def\prob{\mathbb{P}}

\hypersetup{
    colorlinks=true,
    linkcolor=xkcdDarkBlue,
    filecolor=xkcdDarkPurple,
    urlcolor=xkcdDarkBlue,
    citecolor=xkcdDarkBlue,
    anchorcolor=black
}

\title{Trajectory Matching: Discrete Flow Matching with Structured Actions}
    \iclrfinalcopy
\usepackage{etoolbox}
\makeatletter
\patchcmd{\@maketitle}
    {Published as a conference paper at ICLR 2027}
    {Preprint}
    {}
    {\PackageError{preprint}{Could not replace the ICLR publication header}{Check the conference style before distributing this PDF.}}
\makeatother
\author{Niklas Kemper \quad Marten Lienen \quad David L\"udke \quad Stephan G\"unnemann\\[3pt]
{\normalfont\small School of Computation, Information and Technology \& Munich Data Science Institute}\\
{\normalfont\small Technical University of Munich}\\
{\normalfont\small\hypersetup{urlcolor=black}\texttt{\href{mailto:niklas.kemper@tum.de}{niklas.kemper@tum.de}}}}
\hypersetup{
    pdfauthor={Niklas Kemper, Marten Lienen, David L\"udke, Stephan G\"unnemann},
    pdftitle={Trajectory Matching: Discrete Flow Matching with Structured Actions},
    pdfsubject={Preprint}
}

\date{}

\newif\ifincludeappendix
\includeappendixtrue

\begin{document}

\maketitle

\begin{abstract}
The natural units of structured discrete generation are often not independent single-component edits but richer actions: attaching a fragment to a molecule, inserting a node together with an incident edge, or growing a hierarchy by adding a child. Yet most existing discrete flow matching methods (DFM) are restricted to independent component-wise edits over fixed-size states. This restriction stems from endpoint conditioning: endpoints fix the source and target during training but not the action sequence between them, making tractable conditional rates hard to define for complex actions that cannot be factorized component-wise. We propose trajectory matching, which conditions the flow on a full sequence of actions connecting the endpoints. This yields principled, simulation-free training targets for user-defined action spaces, including actions that change state size, modify multiple components, or depend on earlier actions. We prove that the resulting objective recovers the marginal flow and that trajectory matching strictly subsumes factorized DFM, while additionally capturing actions with prerequisite structure such as hierarchy growth or fragment-based molecule generation. Empirically, on tree and molecular graph generation, trajectory matching with domain-aligned actions matches or surpasses factorized DFM baselines (lifting MOSES validity from 92.8\% to 99.2\%) and transfers without retraining to substructure extension and linking.
 Trajectory matching enlarges the design space of discrete flow matching by turning the action vocabulary into a modeling choice.
An interactive explanation of trajectory matching is available at \url{https://niklaskemper.com/blog/trajectory-matching/}.

\end{abstract}

\section{Introduction}

Many structured generation problems are most naturally expressed in terms of actions that reflect the structure of the domain rather than the structure of the representation. In molecular graph generation, for example, the natural actions are not single-component edits but operations like attaching a fragment to a partial structure or inserting an atom together with a bond.
Such actions can change the size of the state, modify several components at once, and depend on one another---an atom can only be attached to atoms that have already been inserted.
The value of such structured actions is well established in autoregressive models: fragment-based  molecule generation  \citep{junction_tree}, grammar-based code generation \citep{code_gen}, and insertion-based graph models \citep{graph_rnn} all exploit domain-aligned actions to achieve strong empirical performance.

However, existing discrete flow matching (DFM) \citep{discrete_flows_original,discrete_flows} and related non-autoregressive methods such as discrete diffusion \citep{discrete_diffusion} cannot easily express  such actions.
The reason is structural: DFM prescribes a conditional flow between sampled endpoints, and the corresponding rate must be tractable for training. To achieve this over combinatorially large state spaces, existing methods factorize the rate across components. This commits the model to fixed-size states, single-component edits, and independence between actions.
Recent extensions \citep{edit_flows,evoflows,tpps} lift the fixed-size constraint by supporting insertions and deletions in specific domains, but they do so by mapping these operations into augmented state spaces where they again reduce to factorized single-component edits. Actions that modify multiple components or depend on one another therefore remain out of reach.

We address this gap with trajectory matching: rather than conditioning on endpoints alone, we condition on a full trajectory of actions connecting them. A trajectory specifies which actions are applied and in what order, leaving the conditional flow to determine only when each action is taken. This makes the conditional rate straightforward to define for general action spaces, including actions that change state size, modify multiple components, or depend on one another. The resulting training objective is still simulation-free and provably recovers the marginal data-generating flow. To make training efficient, we require an efficiently sampleable distribution over trajectories connecting the endpoints. We give explicit constructions for several settings, including actions with hierarchical dependencies. Our framework also recovers standard factorized DFM as the special case of independent component-wise actions with uniformly sampled orderings, placing factorized updates and actions with prerequisites within a common framework.

Empirically, we show that domain-aligned action spaces translate into concrete gains on tree and molecular graph generation. On tree generation, child-attachment actions yield 100\% validity by construction across sampling budgets, while a factorized baseline collapses below 10\% at 10 Euler steps. On molecular graphs, we demonstrate the flexibility of trajectory matching across four action spaces, ranging from factorized bond edits, which match existing DFM baselines, to fragment-based insertions, which lift MOSES validity from 92.8\% to 99.2\%. Additionally, a single variable-size insertion model transfers zero-shot to substructure extension and linking, substantially outperforming a factorized baseline that requires task-specific adaptation.

Our core contributions are:
\begin{itemize}
    \item \textbf{Trajectory matching.} A discrete flow matching framework over user-defined action spaces,  including size-changing, multi-component, and dependent actions, obtained by conditioning training on full action trajectories rather than endpoints. The objective recovers the marginal flow (\cref{thm:trajectory-marginal}); a Rao-Blackwellized variant reduces gradient variance (\cref{prop:rao-blackwell-variance}).

    \item \textbf{A unifying framework.} Trajectory matching subsumes factorized DFM (\cref{prop:factorized-flow-equivalence}) while extending to partial-order action systems, capturing hierarchy growth and fragment-based molecule construction.

    \item \textbf{Empirical gains on structured generation.} On tree and molecular graph generation, trajectory matching with domain-aligned action spaces yields concrete gains over factorized DFM baselines: 100\% validity by construction on trees, MOSES validity from 92.8\% to 99.2\% on molecules, and zero-shot transfer to substructure extension and linking.
\end{itemize}

\section{Action-Based Discrete Flow Matching}
\label{sec:background}
\textbf{States and Actions}\quad We model generation as a sequence of actions on discrete states \smath{s \in \states} rather than component-wise updates. This enlarges the design space beyond standard factorized DFM. Let \smath{\mathcal{A}} be the global action vocabulary and \smath{A(s) \subseteq \mathcal{A}} the set of actions available at state \smath{s}; each \smath{a} produces a new state \smath{a(s) \neq s}.
The choice of \smath{\mathcal{A}} is itself a modeling decision: in molecules, for instance, one might edit bonds on a fixed atom set, attach atoms one at a time, or insert whole ring fragments.
Actions may modify multiple components, depend on one another, or change the size of the state. We say actions \smath{a_1,a_2} are \emph{independent} if the order of application does not matter, i.e., for every \smath{s} with \smath{a_1,a_2 \in A(s)}: \smath{a_1 \in A(a_2(s))}, \smath{a_2 \in A(a_1(s))}, and \smath{a_1(a_2(s)) = a_2(a_1(s))}.
For example, unmasking two tokens in text is independent, while insertion actions generally are not.

\textbf{CTMCs}\quad Our construction is built on continuous-time Markov chains (CTMC) \smath{(S_t)_{t \in [0,1]}} on the state space \smath{\states}. Rather than parameterizing the CTMC by transition rates between states, we parameterize it by an \emph{action rate} \smath{\rate(s,a) \ge 0} of taking action \smath{a \in A(s)} at state \smath{s} and time \smath{t}. The infinitesimal transition probabilities for \smath{s' \neq s} are
\begin{align*}
\resizebox{\linewidth}{!}{$\displaystyle
  \prob[S_{t+h}=s' \mid S_t=s] = h \sum_{a(s) = s'} \rate(s,a) + o(h), \quad \prob[S_{t+h} = s \mid S_t = s] = 1 - h \sum_{a \in A(s)}\rate(s,a) + o(h).
$}
\end{align*}
We say \smath{\rate} \emph{generates} the associated marginal \emph{probability path} \smath{p_t(s) = \prob[S_t = s]}.

\textbf{Discrete Flow Matching}\quad The goal of discrete flow matching \citep{discrete_flows_original,discrete_flows} is to learn the rate \smath{\rate} of a CTMC that transports the source distribution \smath{\noiseDist} to the target distribution \smath{\dataDist}. Directly specifying \smath{\rate} is in general intractable. Instead, flow matching specifies, for each pair of endpoints \smath{(\initState, \finalState)}, an endpoint-conditioned CTMC \smath{(S_t \mid \initState, \finalState)_t} with prescribed rates \smath{\rate(s, a \mid \initState, \finalState)} that generate a conditional path \smath{p_t(s \mid \initState, \finalState)} satisfying
\[
p_0(s \mid \initState, \finalState) = \indicator{s = \initState},
\qquad
p_1(s \mid \initState, \finalState) = \indicator{s = \finalState}.
\]
Training uses a coupling distribution \smath{\coupling(\initState,\finalState)} for the endpoints with marginals \smath{\noiseDist} and \smath{\dataDist}, respectively. The marginal rate and path are then given by:
\begin{align*}
\rate(s,a) = \EX_{(\initState,\finalState) \sim \coupling} [\rate(s,a \mid \initState,\finalState) \mid S_t = s],\quad
p_t(s) = \EX_{(\initState,\finalState) \sim \coupling} p_t(s \mid \initState,\finalState)
\end{align*}
where by construction \smath{p_0 = \noiseDist} and \smath{p_1 = \dataDist}. The key insight of flow matching is that the marginal rate \smath{\rate} generates the marginal path \smath{p_t} \citep{flow1}. The conditional rate \smath{\rate(\cdot \mid \initState,\finalState)} thus serves as a tractable target for an unconditional model \smath{\rate^\theta}, trained by minimizing a Bregman divergence \smath{D} \citep{generator_matching}:
\begin{align}
\label{eq:flow-matching-objective}
\EX_{(\initState,\finalState) \sim \coupling}
\EX_{t}
\EX_{s \sim p_t(\cdot \mid \initState,\finalState)}
D\!\left(
\rate(s,\cdot \mid \initState,\finalState),
\rate^\theta(s,\cdot)
\right)
\end{align}

\textbf{Factorized Flows}
\label{subsec:factorized}
Specifying the conditional rate \smath{\rate(s, a \mid \initState, \finalState)} is the central design choice in discrete flow matching: it must be efficiently computable and generate an efficiently sampleable path connecting \smath{\initState} and \smath{\finalState}. In general action spaces, however, many different action sequences --- potentially with dependencies between actions --- may connect \smath{\initState} to \smath{\finalState}, making it difficult to specify such rates from the endpoints alone. Existing discrete flow matching methods therefore make strong assumptions on the action space: 1) initial, final, and all intermediate states decompose into components \smath{s = (s^1,\dots,s^d)} with the same fixed size \smath{d}, 2) each action modifies exactly one component from its initial value \smath{\initState^i} to its final value \smath{\finalState^i}, and 3) conditioned on the endpoints, components evolve independently. Under these assumptions, the conditional rate factorizes across components: for an action \smath{a} acting on component \smath{i}, it suffices to specify a rate \smath{\rate(s, a \mid \initState, \finalState)} that depends only on \smath{s^i}, \smath{\initState^i}, and \smath{\finalState^i}. Each component then evolves independently between its initial and final values. A common choice is the linear schedule
\begin{align}
\label{eq:factorized-flow-rate}
\rate(s,a \mid \initState,\finalState)
= \indicator{a^i(s^i) = \finalState^i} \cdot 1/(1-t),
\end{align}
which generates the factorized path
\begin{align*}
\resizebox{\linewidth}{!}{$\displaystyle
p_t(s \mid \initState,\finalState)
= \prod_{i=1}^{d} p_t^i(s^i \mid \initState^i,\finalState^i), \quad \text{with} \quad
p_t^i(s^i \mid \initState^i,\finalState^i)
= (1-t)\,\indicator{s^i = \initState^i}
+ t\,\indicator{s^i = \finalState^i}.
$}
\end{align*}
Other time schedules yield analogous constructions. This is the most widely used discrete flow matching formulation, which we will refer to as \emph{factorized flows}.

\section{Trajectory Matching}
\label{sec:trajectory-matching}
Factorized flows are simple to define and have shown strong empirical performance on a variety of tasks. However, their underlying assumptions rule out many natural modeling choices in structured domains: actions cannot change the size of the state (e.g., adding an atom to a molecule), cannot modify multiple components at once (e.g., inserting a full fragment into a molecular graph), and cannot depend on one another (e.g., a bond insertion requires both endpoints to already exist). We lift these restrictions by defining flows over general action spaces. As argued in \cref{subsec:factorized}, specifying valid rates from endpoints alone is difficult in this setting, because many action sequences may connect the same endpoint pair. We therefore condition not only on endpoints but also on a full \emph{trajectory}: a concrete sequence of actions transforming \smath{\initState} into \smath{\finalState}. Conditioning on a trajectory makes the conditional rate straightforward to define, even for size-changing, multi-component, or dependent actions.

\subsection{Trajectory-conditioned flows}
A trajectory \smath{\trajectory} from an initial state \smath{\initState} to a final state \smath{\finalState} is a sequence of states and actions 
\begin{align*}
    \trajectory: \initState = s_0 \xrightarrow{a_1} s_1 \xrightarrow{a_2} \cdots \xrightarrow{a_n}s_n = \finalState.
\end{align*}
For ease of notation, we assume that no state \smath{s_i} appears more than once in one trajectory. 

Next, we define a trajectory-conditioned CTMC \smath{(S_t \mid \trajectory)_{t \in [0,1]}}. The trajectory fixes \emph{the order of actions and states} while the CTMC randomizes \emph{when} each action is taken. We again specify the process via its rate \smath{\rate(s,a \mid \trajectory)}, requiring it to generate a path with \smath{p_0(s \mid \trajectory) = \indicator{s = \initState}} and \smath{p_1(s \mid \trajectory) = \indicator{s = \finalState}}. A natural choice is to assume equal expected times between actions, i.e., the execution times \smath{t_i} of actions \smath{a_i} 
(\smath{ 0 < t_1 < \cdots < t_n < 1}) are distributed as the order statistics of \smath{n} i.i.d. uniform random variables on \smath{[0,1]}. 
This corresponds to the rate
\begin{align}
\label{eq:trajectory-rate}
\rate(s,a \mid \trajectory) = \sum_{i=0}^{n-1} \indicator{s=s_i \land a = a_{i+1}} \cdot \frac{n-i}{1-t}
\end{align}
in which only the next action in the trajectory has positive rate, proportional to the number of remaining actions. The generated path has a simple closed form:
\begin{restatable}[Trajectory-conditioned path]{proposition}{trajectorypathprop}
\label{prop:trajectory-path}
Under rate \cref{eq:trajectory-rate}, the generated probability path is
\[
p_t(s \mid \trajectory) = \sum\nolimits_{i=0}^{n} \indicator{s = s_i} \cdot \prob[K_t = i],
\]
where \smath{K_t \sim \text{Bin}(n, t)} counts the number of actions taken by time \smath{t}.
\end{restatable}
In particular, \smath{p_t(\cdot \mid \trajectory)} is easy to sample from. Other choices of conditional rates are discussed in \cref{app:kappa-schedules}. Importantly, conditioning on the trajectory lets us define a CTMC between initial and final state even for settings with complex actions where the endpoint-conditioned rate cannot be factorized.

\subsection{Trajectory matching}
\begin{figure}[tbp]
    \centering
    \includegraphics[width=1\textwidth]{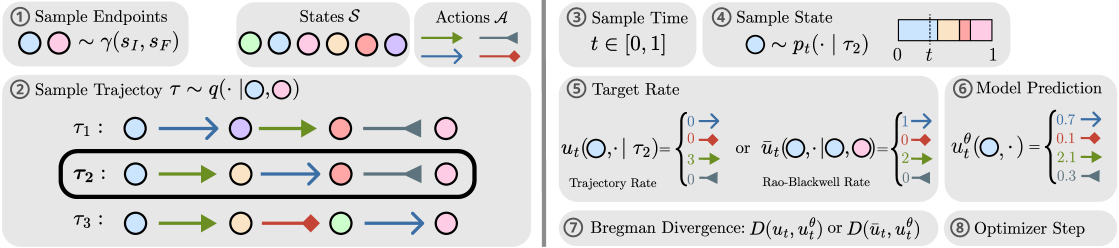}
    \vspace{-0.7cm}
    \caption{Training step of trajectory matching. Conditioning on a full trajectory (rather than endpoints alone) yields tractable training targets for general action spaces (including multi-component, dependent, or size-changing actions). Pseudocode in \cref{alg:trajectory_matching_training}.}
    \label{fig:training}
    \vspace{-0.4cm}
\end{figure}
Having defined trajectory-conditioned CTMCs, we now describe how to use their rates as training targets for an unconditional model \smath{\rate^\theta(s,\cdot)} that depends only on the current state and time. For any endpoint pair \smath{(\initState, \finalState)}, many valid trajectories may connect them. Therefore, in addition to the coupling \smath{\coupling} over endpoints, we also specify a conditional distribution 
$\trajDist(\trajectory \mid \initState, \finalState)$
over trajectories from \smath{\initState} to \smath{\finalState}. We assume \smath{\trajDist(\cdot \mid \initState, \finalState)} is supported on valid trajectories for all \smath{(\initState, \finalState)} in the support of \smath{\coupling}, and that sampling from \smath{\trajDist} is efficient. Intuitively, \smath{\trajDist} 
encodes an inductive bias over how \smath{\initState} should be transformed into \smath{\finalState}: it determines which trajectories 
contribute to training and with what weight. Analogously to endpoint-conditioned discrete 
flow matching, one can show: 
\begin{restatable}{theorem}{trajectorymarginalthm}
\label{thm:trajectory-marginal}
Let \smath{\rate(s, a \mid \trajectory)} be a conditional rate that generates the trajectory-conditioned path \smath{p_t(s \mid \trajectory)}. Define the marginal probability path and marginal rate as
\begin{align*}
p_t(s) := \EX_{\trajectory \sim \trajDist}\, p_t(s \mid \trajectory), \qquad \rate(s, a) := \EX_{\trajectory \sim \trajDist}[\, \rate(s, a \mid \trajectory) \mid S_t = s\,],
\end{align*}
where \smath{\EX_{\trajectory \sim \trajDist}} is shorthand for \smath{\EX_{(\initState, \finalState) \sim \coupling}\, \EX_{\trajectory \sim \trajDist(\cdot \mid \initState, \finalState)}}. Then \smath{\rate} generates \smath{p_t}.
\end{restatable}
In particular, we have \smath{p_0 = \noiseDist} and \smath{p_1 = \dataDist}, so the marginal rate \smath{\rate} transports noise to data.

\textbf{Training.}\quad
Similarly to endpoint-conditioned flow matching, \cref{thm:trajectory-marginal} enables efficient simulation-free training using a Bregman divergence \smath{D}:
\begin{tcolorbox}[colback=gray!15, colframe=gray!15, boxrule=0pt, arc=2pt, left=4pt, right=4pt, top=0pt, bottom=0pt, boxsep=0pt]
\setlength{\abovedisplayskip}{0pt}
\begin{align}
\label{eq:trajectory_matching_objective}
\loss(\theta) := \EX_{\trajectory \sim \trajDist}\, \EX_{t}\, \EX_{s \sim p_t(\cdot \mid \trajectory)}\, D\!\left(\rate(s, \cdot \mid \trajectory),\, \rate^\theta(s, \cdot)\right).
\end{align}
\end{tcolorbox}
This loss is a valid proxy for the marginal matching loss in two senses. First, its gradients agree, so optimizing \smath{\loss(\theta)} is equivalent to optimizing the marginal target:
\begin{restatable}[Consistent gradients]{proposition}{traingradientsprop}
\label{prop:traj-gradients}
For any parametric family \smath{\rate^\theta},
\begin{align*}
\nabla_\theta \loss(\theta) \;=\; \nabla_\theta \EX_t\, \EX_{s \sim p_t}\, D\!\left(\rate(s, \cdot),\, \rate^\theta(s, \cdot)\right).
\end{align*}
\end{restatable}
Second, with a sufficiently expressive family, the minimizer recovers the data-generating flow:
\begin{restatable}[Consistent optima]{corollary}{trainoptimacor}
\label{cor:traj-optima}
If the family \smath{\rate^\theta} contains the marginal rate \smath{\rate}, then any minimizer of \smath{\loss(\theta)} equals \smath{\rate} for almost all $(s, t)$ with $p_t(s) > 0$ and generates the marginal probability path \smath{p_t}.
\end{restatable}
\cref{fig:training} summarizes the resulting training procedure. Importantly, trajectory matching training remains simulation-free as \smath{p_t(s \mid \trajectory)} can be efficiently sampled from (\cref{prop:trajectory-path}) and \smath{\rate(s, a \mid \trajectory)} can be computed in closed-form. Sampling from \smath{\coupling} and \smath{\trajDist} is application-specific; we discuss efficient instantiations in \cref{sec:instantiations}. In our experiments, overlapping preprocessing with GPU training makes the additional wall-clock cost small (\cref{app:runtime}).

\textbf{Sampling.}\quad
Sampling proceeds as in endpoint-conditioned DFM: draw \smath{\initState \sim \noiseDist} and integrate the learned CTMC with rates \smath{\rate^\theta} over \smath{t \in [0, 1]}. With Euler integration, each step of size \smath{h} samples each available action \smath{a} independently with probability \smath{h\,\rate^\theta(s, a)} and applies the sampled compatible actions to the current state; see \cref{app:training_sampling}.
Although the trajectory-conditioned targets \smath{\rate(\cdot \mid \trajectory)} are sparse---only the next action in the sampled trajectory has positive rate---the learned model \smath{\rate^\theta} approximates the \emph{marginal} rate (\cref{thm:trajectory-marginal}), which can place mass on many currently available actions. This allows multiple actions per Euler step, so sampling does not degenerate to sequential generation.

In summary, trajectory matching separates three modeling decisions: i) \textbf{the state and action space}, which determines what transformations are possible; ii) \textbf{the source and target distributions \smath{\noiseDist, \dataDist} and coupling \smath{\coupling}}, which specify \emph{what} is transported; and iii) \textbf{the trajectory distribution \smath{\trajDist(\trajectory \mid \initState, \finalState)}}, which specifies \emph{how} endpoints are connected through action sequences.

\subsection{Variance reduction via Rao–Blackwellization}
\label{sec:rao-blackwell}
As we have seen, the trajectory-matching objective \cref{eq:trajectory_matching_objective} recovers the correct marginal flow by supervising the model on a single next action per sampled trajectory. We can sharpen the training signal by observing that multiple trajectories in \smath{\trajDist(\cdot \mid \initState, \finalState)} may pass through the same state \smath{s} at time \smath{t}, each prescribing a potentially different next action. We now show how to supervise the model on all such actions simultaneously, reducing gradient variance while preserving the optimum. To this end, define the posterior distribution over trajectories given the current state \smath{s} and endpoints \smath{(\initState,\finalState)}:
\begin{align*}
    p_t(\trajectory \mid s, \initState, \finalState) := \frac{p_t(s \mid \trajectory) \cdot \trajDist(\trajectory \mid \initState, \finalState)}{\sum_{\trajectory'} p_t(s \mid \trajectory') \cdot \trajDist(\trajectory' \mid \initState, \finalState)}
\end{align*}
The \emph{Rao-Blackwellized target} is then the expectation of \smath{\rate(s,a \mid \trajectory)} under this posterior:
\begin{align*}
    \rbrate(s,a \mid \initState, \finalState) := \EX_{\trajectory \sim p_t(\cdot \mid s, \initState, \finalState)} \rate(s,a \mid \trajectory).
\end{align*}
When this expectation can be computed efficiently (as we show for several trajectory families in \cref{sec:instantiations}), the Rao--Blackwellized rate can replace \smath{\rate(\cdot \mid \trajectory)} as the training target in \cref{eq:trajectory_matching_objective}. This preserves the expected gradient and the optimum while reducing gradient variance (\cref{prop:rao-blackwell-consistency,prop:rao-blackwell-variance}). Our ablation suggests that this yields faster convergence to comparable final performance (\cref{app:rb-ablation}).

In summary, conditioning on the full trajectory gives a general recipe for defining valid conditional flows over arbitrary action spaces, including settings where endpoint-conditioned rates cannot be specified directly. This yields a simulation-free training objective whose optimum recovers the 
marginal data-generating flow. 
Gradient variance can be further reduced by Rao-Blackwellizing the target over all trajectories consistent with the current state and endpoints.

\section{Instantiations of Trajectory Matching}
\label{sec:instantiations}
\suppressfloats[t]
\begin{figure}[t]
    \centering
    \setlength{\abovecaptionskip}{2pt}
    \includegraphics[width=1\textwidth]{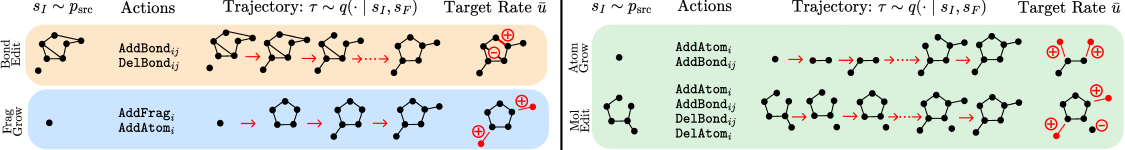}
    \caption{Four trajectory-matching instantiations on molecular graphs, \textsc{BondEdit} (factorized bond edits), \textsc{FragGrow} (fragment-based growth), \textsc{AtomGrow} (atom/bond growth from a seed), and \textsc{MolEdit} (noisy-initial editing), each shown with a sampled source $\initState$, action set, example trajectory $\trajectory$, and Rao-Blackwellized target rate $\rbrate$. Full definitions in \cref{app:instantiations}, more instantiations in \cref{fig:connections-overview}} 
    \label{fig:overview_mol}
    \vspace{-0.3cm}
\end{figure}
We now instantiate trajectory matching at increasing levels of generality. The two central design choices are the action space and the trajectory distribution \smath{\trajDist}. We illustrate these choices through a running example of molecular graph generation (\cref{fig:overview_mol}), moving from independent actions to settings outside the scope of factorized flows.
We begin with the most common case: for each endpoint pair, trajectories differ only in the \emph{order} of actions, while the action set itself is fixed.

\begin{restatable}[Fixed-action distribution]{definition}{fixedactiondistributiondef}
    \label{def:fixed-action-distribution}
    \smath{\trajDist} is a \emph{fixed-action distribution} if for every
    \smath{(\initState, \finalState)}, all
    trajectories
    \smath{\trajectory \in \supp\trajDist(\cdot \mid \initState, \finalState)}
    use the same set of actions, denoted
    \smath{\actionSet(\initState, \finalState)}.
\end{restatable}

When the same property holds at intermediate states, we write \smath{\actionSet(s, \finalState)} for the set of actions remaining after \smath{s}: that is, all trajectories passing through \smath{s} on their way to
\smath{\finalState} share the same remaining action set, independent of how \smath{s} was reached.

\hlOrange{\textbf{Independent fixed-action distribution}}

The simplest case is when the required actions are pairwise independent and the trajectory distribution is uniform over their orderings. For example, in \textsc{BondEdit}, which mirrors DeFoG \citep{defog}, atoms are fixed by the coupling \smath{\coupling} and only edges can be inserted or deleted. Any ordering of these independent edge edits yields a valid trajectory.
\begin{restatable}[Independent fixed-action distribution]{definition}{independentfixedactiondistributiondef}
    \label{def:independent-fixed-action-distribution}
    A fixed-action distribution \smath{\trajDist} is \emph{independent} if for every \smath{(\initState, \finalState)}, the actions in \smath{\actionSet(\initState, \finalState)} are pairwise independent, and \smath{\trajDist(\cdot \mid \initState, \finalState)} is uniform over their orderings.
\end{restatable}
Sampling from \smath{\trajDist} is efficient whenever computing \smath{\actionSet} is. The  rate $\rbrate$ also admits a closed form:
\begin{restatable}{proposition}{independentratesprop}
    \label{prop:independent_rates}
    Let \smath{\trajDist} be an independent fixed-action distribution. Then for a state \smath{s} on a trajectory from \smath{\initState} to \smath{\finalState} we have: \smath{\rbrate(s,a \mid \initState, \finalState) = \indicator{a \in \actionSet(s, \finalState)} \cdot 1/(1-t)}.
\end{restatable}
Thus, all remaining actions receive the same rate, reflecting the order-invariance. Importantly, factorized flow matching arises as the special case of trajectory matching with order-invariant, independent component-wise actions:
\begin{restatable}{theorem}{factorizedflowequivalenceprop}
    \label{prop:factorized-flow-equivalence}
    Consider a factorized flow as defined in \cref{eq:factorized-flow-rate} with single component actions \smath{a^i_v(s) = (s^1, \ldots, v, \ldots, s^d)}. Let \smath{\trajDist} be the corresponding independent fixed action distribution with \smath{\actionSet(\initState, \finalState) =\{a^{i}_v : v = \finalState^i \neq \initState^i\}}. Then:
    \begin{enumerate}
        \item The trajectory-matching loss (\cref{eq:trajectory_matching_objective}) and the factorized flow matching loss (\cref{eq:flow-matching-objective}) have the same gradient in \smath{\theta} and hence the same minimizers.
        \item The Rao-Blackwellized loss (\cref{prop:independent_rates}, \cref{prop:rao-blackwell-consistency}) is identical to the factorized loss (\cref{eq:flow-matching-objective}).
    \end{enumerate}
\end{restatable}

\hlBlue{\textbf{Dependent fixed-action distribution}}

Having established that trajectory matching subsumes factorized flows, we now relax the independence assumption. Dependencies between actions arise naturally in structured domains: in hierarchy generation, adding a child requires its parent to exist; in molecular generation without a fixed atom set, inserting a bond requires both endpoint atoms to exist. A natural way to model such dependencies is through a directed acyclic graph (DAG) over actions.
\begin{restatable}{definition}{actiondependencydagdef}
\label{def:action-dependency-dag}
We equip the action set \smath{\actionSet(s, \finalState)} with a DAG whose vertices are the actions and whose edges \smath{a_i \to a_j} indicate that \smath{a_i} is a prerequisite of \smath{a_j}. We require two actions to be independent if neither is reachable from the other. A trajectory \smath{\trajectory} is consistent with \smath{\actionSet(\initState, \finalState)} if its action sequence is a topological ordering of this DAG.
\end{restatable}
Similarly to the independent fixed-action distribution, we define a dependent fixed-action distribution as uniform over all \emph{consistent} trajectories:
\begin{restatable}[Dependent fixed-action distribution]{definition}{dependentfixedactiondistributiondef}
    \label{def:dependent-fixed-action-distribution}
    We say \smath{\trajDist} is a \emph{dependent fixed-action distribution} if for all \smath{\initState, \finalState} every trajectory in the support of \smath{\trajDist(\cdot \mid \initState, \finalState)} is consistent with the dependency DAG \smath{\actionSet(\initState, \finalState)} and \smath{\trajDist(\cdot \mid \initState, \finalState)} is uniform over all such trajectories.
\end{restatable}
This family interpolates between two familiar regimes: an edgeless DAG recovers the independent fixed-action distribution, while a linear chain DAG recovers fixed-order autoregressive generation. General DAGs lie between these extremes and capture partially ordered action systems such as our \textsc{GrowHierarchy}
(a tree grown by attaching children to existing nodes) and
\textsc{FragGrow} (a molecule grown by attaching ring fragments to a
root atom); see \cref{fig:overview_mol} and \cref{app:instantiations}.

Unlike the independent case, the Rao--Blackwellized target \smath{\rbrate} can assign different rates to the remaining actions. Available actions on which many future actions depend receive higher rates, encouraging them to occur earlier. For forest-structured dependency DAGs, trajectories can be sampled and Rao--Blackwellized rates computed efficiently:
\begin{restatable}{proposition}{dependentfixedactionforestratesprop}
    \label{prop:dependent-fixed-action-forest-rates}
    Let \smath{\trajDist} be a dependent fixed-action distribution whose dependency DAG \smath{\actionSet(\initState, \finalState)} is a forest. Then for a state \smath{s} on a trajectory from \smath{\initState} to \smath{\finalState} we have:
    \begin{align*}
    \rbrate(s,a \mid \initState, \finalState) = \indicator{a \in \textup{roots}(\actionSet(s, \finalState))} \cdot \textup{desc}_{\actionSet(s, \finalState)}(a) \cdot 1/(1-t)
    \end{align*}
    where \smath{\textup{roots}} is the set of roots of the DAG and \smath{\textup{desc}(a)} is the size of the tree rooted at \smath{a} (\smath{a} included).
\end{restatable}
The resulting target interpolates between the independent and fully dependent extremes. With no dependencies between actions, it reduces to the uniform rate from \cref{prop:independent_rates}, supervising every remaining action with equal weight. With a fully sequential chain of dependencies, all weight concentrates on the single next action, recovering fixed-order autoregressive supervision. Between these extremes, every currently available action receives a rate proportional to the size of its remaining subtree.
Rao--Blackwellized rates for general dependency DAGs are given in \cref{prop:dependent-fixed-action-general-rates}, but computing them exactly is intractable in general. For the forest-plus-sink dependencies arising in \textsc{AtomGrow} and \textsc{MolEdit}, we derive an alternative conditional path with exact, tractable rates (\cref{app:non-forest-approx}).

Dependent fixed action distributions extend trajectory matching to partially ordered action systems that factorized endpoint-conditioned flows cannot express. Conditioning on concrete trajectories yields targets that span the full spectrum from unordered to fully sequential actions in a unified way.

\hlGreen{\textbf{Beyond fixed action sets}}

\label{sec:beyond-fixed}
So far we assumed that all trajectories between a given endpoint pair use the same set of actions. In some settings, however, different trajectories between \smath{\initState} and \smath{\finalState} involve different action sets, for example because the target admits multiple decompositions. In \textsc{AtomGrow}, a given ring bond may be created when a new atom is attached by \texttt{AddAtom} or by a separate \texttt{AddBond} action after both endpoints exist. One way to obtain tractable trajectory distributions in such settings is to introduce a latent variable \smath{z} specifying the decomposition. We express the trajectory distribution as a mixture, which can be sampled efficiently whenever both the latent distribution and the conditional trajectory distribution can be sampled efficiently:
\begin{align}
\label{eq:beyond-fixed}
\trajDist(\trajectory \mid \initState, \finalState) = \EX_{z \sim p(z \mid \initState, \finalState)} \trajDist(\trajectory \mid z, \initState, \finalState).
\end{align}
For \textsc{AtomGrow}, \smath{z} specifies a spanning tree rooted at the seed atom: tree edges are created by \texttt{AddAtom} actions, while other edges require \texttt{AddBond} actions. Conditional on \smath{z}, the action set and its dependencies are fixed again, forming a forest-plus-sink structure for which we have a tractable conditional path with exact rates (\cref{app:non-forest-approx}). \textsc{MolEdit} uses a similar construction (\cref{app:instantiations}).

Importantly, we can also understand the existing insertion- and deletion-based text model EditFlows \citep{edit_flows} as an instantiation of trajectory matching over a variable action set; see \cref{app:editflows}.

In summary, trajectory matching provides a flexible and principled framework for designing discrete flows over domain-specific action spaces. It subsumes classical endpoint-conditioned factorized flows as a special case, while naturally extending to complex actions, partial dependencies, and structured generation processes for which direct endpoint-conditioned targets are difficult to specify.

\section{Empirical Results}
\label{sec:results}
\suppressfloats[t]
\begin{table*}[t]
\setlength{\abovecaptionskip}{0pt}
\centering
\newsavebox{\treecomparisonbox}
\newsavebox{\moleculecomparisonbox}
\providecommand{\std}[1]{$_{\pm #1}$}
\providecommand{\gc}{\cellcolor{gray!15}}
\sbox{\treecomparisonbox}{%
  \setlength{\tabcolsep}{3pt}%
  \renewcommand{\arraystretch}{1.07542}%
  \begin{tabular}{clcc}
    \toprule
    & \multicolumn{3}{c}{Tree} \\
    \cmidrule(lr){2-4}
    & Model & V.U.N. $\uparrow$ & Ratio $\downarrow$ \\
    \midrule
    \multirow{5}{*}{\rotatebox[origin=c]{90}{\footnotesize Fixed-size}} & DiGress                   & 90.0  & \textbf{1.6} \\
     & EDGE                        & 0.0   & 850.7 \\
     & BwR                & 0.0   & 11.4 \\
     & HSpectre (one-shot) & 82.5 & 2.1 \\
     & DeFoG                       & $96.5_{\pm 2.6}$ & $\mathbf{1.6}_{\pm 0.4}$ \\
    \midrule
    \multirow{6}{*}{\rotatebox[origin=c]{90}{\footnotesize Variable-size}} & GraphRNN & 0.0 & 607.0 \\
     & GRAN & 0.0 & 607.0 \\
     & BiGG & 75.0 & 5.2 \\
     & GraphGen & 95.0 & 33.2 \\
     & HSpectre & \textbf{100} & 4.0 \\
     & \cellcolor{gray!15} GrowHierarchy & \cellcolor{gray!15} $\textbf{100.0}_{\pm 0.0}$ & \cellcolor{gray!15} $2.5_{\pm 0.6}$ \\
    \bottomrule
\end{tabular}
}
\sbox{\moleculecomparisonbox}{%
  \setlength{\tabcolsep}{4pt}%
  \renewcommand{\arraystretch}{1.0}%
  \begin{tabular}{cl ccccc ccccccc}
\toprule
& & \multicolumn{5}{c}{ZINC} 
& \multicolumn{7}{c}{MOSES} \\
\cmidrule(lr){3-7}\cmidrule(lr){8-14}
& Model 
& Val.\,$\uparrow$ & Uniq.\,$\uparrow$ & FCD\,$\downarrow$ & NSPDK\,($\times 10^3$)\,$\downarrow$ & Scaf\,$\uparrow$
& Val.\,$\uparrow$ & Uniq.\,$\uparrow$ & Nov.\,$\uparrow$ & Filt.\,$\uparrow$ & FCD\,$\downarrow$ & SNN\,$\uparrow$ & Scaf\,$\uparrow$ \\
\midrule
& Training set 
& 100.0 & 100.0 & 0.19 & 0.1 & 0.48 
& 100.0 & 100.0 & 0.0 & 100.0 & 0.01 & 0.64 & 99.1 \\
\midrule
\multirow{6}{*}{\rotatebox[origin=c]{90}{\footnotesize Fixed-size}}
& DiGress~\citep{digress} 
& 95.0 & -- & 3.48 & 2.1 & 0.42 
& 85.7 & 100.0 & 95.0 & 97.1 & \underline{1.19} & 0.52 & 14.8 \\
& DisCo~\citep{disco} 
& -- & -- & -- & -- & -- 
& 88.3 & 100.0 & \underline{97.7} & 95.6 & 1.44 & 0.50 & \underline{15.1} \\
& Cometh~\citep{cometh} 
& -- & -- & -- & -- & -- 
& 90.5 & 99.9 & 92.6 & \underline{99.1} & 1.27 & \underline{0.54} & \textbf{16.0} \\
& GBD~\citep{gbd} 
& 97.9 & -- & 2.25 & 1.8 & 0.50 
& -- & -- & -- & -- & -- & -- & -- \\
& DeFoG~\citep{defog} 
& \textbf{99.2\std{0.1}} & \underline{99.99\std{0.01}} & 1.43\std{0.02} & 0.8\std{0.1} & \textbf{0.59\std{0.01}} 
& \underline{92.8} & 99.9 & 92.1 & 98.9 & 1.95 & \textbf{0.55} & 14.4 \\
& \gc BondEdit 
& \gc 98.8\std{0.1} & \gc 99.98\std{0.02} & \gc 1.41\std{0.02} & \gc \underline{0.5\std{0.0}} & \gc \underline{0.56\std{0.01}} 
& \gc 84.9 & \gc 100.0 & \gc 97.0 & \gc 96.6 & \gc 1.38 & \gc 0.50 & \gc 14.8 \\
\midrule
\multirow{4}{*}{\rotatebox[origin=c]{90}{\footnotesize Variable-size}}
& GrIDDD~\citep{griddd}
& 85.3 & 99.95 & 4.30 & 2.3 & 0.34 
& 79.6 & 99.9 & 98.5 & 84.7 & 1.67 & 0.48 & 14.9 \\
& \gc MolEdit 
& \gc 89.6\std{0.1} & \gc \underline{99.99\std{0.00}} & \gc 3.79\std{0.02} & \gc 4.9\std{0.0} & \gc 0.36\std{0.01} 
& \gc 76.7 & \gc 100.0 & \gc \textbf{99.3} & \gc 93.1 & \gc 2.09 & \gc 0.48 & \gc 7.7 \\
& \gc AtomGrow
& \gc 98.2\std{0.1} & \gc \textbf{100.00\std{0.00}} & \gc \underline{1.24\std{0.03}} & \gc 0.6\std{0.1} & \gc 0.52\std{0.01} 
& \gc 82.0 & \gc 100.0 & \gc 94.6 & \gc 98.1 & \gc \textbf{0.94} & \gc 0.53 & \gc 13.0 \\
& \gc FragGrow
& \gc \underline{98.9\std{0.0}} & \gc \textbf{100.00\std{0.00}} & \gc \textbf{1.22\std{0.03}} & \gc \textbf{0.4\std{0.0}} & \gc 0.51\std{0.01} 
& \gc \textbf{99.2} & \gc 99.9 & \gc 90.7 & \gc \textbf{99.3} & \gc 1.21 & \gc \textbf{0.55} & \gc 13.7 \\
\bottomrule
\end{tabular}
}
\pgfmathsetmacro{\comparisonScale}{(\the\textwidth-8pt)/(\the\wd\treecomparisonbox+\the\wd\moleculecomparisonbox)}
\pgfmathsetlengthmacro{\treecomparisonwidth}{\comparisonScale*\the\wd\treecomparisonbox}
\pgfmathsetlengthmacro{\moleculecomparisonwidth}{\comparisonScale*\the\wd\moleculecomparisonbox}
\begin{minipage}[t]{\treecomparisonwidth}
\caption{Graph generation on Tree. Full results in \cref{tab:tree_full}.}
\label{tab:tree-comparison}
\end{minipage}\hfill
\begin{minipage}[t]{\moleculecomparisonwidth}
\caption{Comparison of DFM and diffusion molecular graph generative models on ZINC~\citep{zinc} and MOSES~\citep{moses}. Best results in \textbf{bold}, second best \underline{underlined}.}
\label{tab:mol-comparison}
\end{minipage}
\par\vspace{2pt}\nointerlineskip
\scalebox{\comparisonScale}{\usebox{\treecomparisonbox}}\hfill
\scalebox{\comparisonScale}{\usebox{\moleculecomparisonbox}}
\vspace{-0.5cm}
\end{table*}

We evaluate trajectory matching on tree and molecular graph generation 
and show that: (i) domain-aligned action spaces can enforce \emph{structural 
validity by construction}, which factorized flows cannot; (ii) our 
framework supports \emph{diverse action spaces} over the same domain, where 
domain-aligned ones can outperform factorized baselines; and (iii) the 
resulting models \emph{transfer without retraining} to structured downstream 
tasks such as substructure extension and linking. All instantiations 
share the same training and sampling procedure and graph transformer 
backbone~\citep{digress}; only the action space, source distribution, 
and trajectory distribution differ. Setup details 
are in \cref{app:setup}.

\textbf{Hierarchy generation.}\label{sec:exp-hierarchy}
Our \textsc{GrowHierarchy} instantiation generates rooted trees using a
single action type that attaches a new child to an existing node.
Because the action space only contains tree-preserving operations, every
reachable state is a connected tree and \emph{validity is guaranteed by
construction}, without filtering or projection. On the Tree dataset
\citep{tree}, all evaluated samples are also unique and novel, yielding
100\% V.U.N. \textsc{GrowHierarchy} achieves a Ratio of 2.5, the best
distributional graph-statistics score among the evaluated variable-size
methods (\cref{tab:tree-comparison}). The advantage of our domain-aligned
actions is particularly clear at low sampling budgets: at just 10 Euler
steps, well below the number of actions per trajectory, \textsc{GrowHierarchy}
maintains a comparable Ratio and 100\% validity, while DeFoG's validity drops to 8\%
(\cref{fig:tree-validity-main,fig:tree-step-sweep}).

\textbf{Molecular graph generation.}\label{sec:exp-molecules}
We instantiate trajectory matching with four different action spaces 
over the same domain, illustrating the flexibility of the framework: 
\textsc{BondEdit} (bond edits on a fixed-size atom set, mirroring 
existing factorized approaches), \textsc{AtomGrow} (atom and bond 
insertions from a seed atom), \textsc{FragGrow} (ring fragment 
insertions from a dataset-specific vocabulary), and \textsc{MolEdit} 
(insertions and deletions on noisy initial molecules); see 
\cref{fig:overview_mol}.
\cref{tab:mol-comparison} reports results on ZINC and
MOSES against fixed-size and variable-size baselines.
QM9 results are in \cref{tab:qm9}.

Three findings stand out. First, \textsc{BondEdit}, our factorized 
instantiation, matches existing factorized DFM and discrete diffusion 
baselines (despite using no sampling-time optimizations), showing that 
the generality of trajectory matching comes at no cost in the 
factorized regime (\cref{prop:factorized-flow-equivalence}). Second, 
\textsc{AtomGrow} and \textsc{FragGrow} outperform the variable-size
diffusion baseline GrIDDD~\citep{griddd} on most metrics, suggesting that
directly modeling size-changing actions offers advantages over augmenting
factorized diffusion with auxiliary size prediction. Third,
\textsc{FragGrow}, which encodes the strongest inductive bias by 
treating ring fragments as atomic actions, achieves very strong 
DFM results on ZINC and MOSES, most notably lifting MOSES validity 
from 92.8\% to 99.2\%. This empirically supports the value of aligning 
the action space with the structure of the domain.

\suppressfloats[t]
\begin{table*}[t]
\setlength{\abovecaptionskip}{0pt}
\begin{minipage}[t]{0.29\textwidth}
\captionof{figure}{Tree validity by sampling steps.}
\label{fig:tree-validity-main}
\end{minipage}\hfill
\begin{minipage}[t]{0.69\textwidth}
\caption{Substructure extension and linking. Scores are valid-and-connected percentages ($\uparrow$). Full results in \cref{fig:extension-heatmap,fig:linking-heatmap}.}
\label{tab:downstream}
\label{tab:downstream-linking}
\end{minipage}
\par\nointerlineskip
\makebox[0.29\textwidth][c]{%
\raisebox{-2pt}[0pt][0pt]{\adjustbox{valign=b}{%
\includegraphics[width=0.275\textwidth]{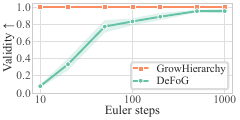}}}}\hfill
\adjustbox{valign=b}{%
\begin{minipage}{0.69\textwidth}
\centering
\newcommand{\downstreamimage}[2][]{%
  \tikz[baseline=(molecule.base)]{%
    \node[inner sep=0pt,outer sep=0pt,blend mode=multiply] (molecule)
      {\includegraphics[#1]{#2}};}}
\setlength{\tabcolsep}{1.5pt}
\renewcommand{\arraystretch}{1.0}
\small
\resizebox{\linewidth}{!}{%
\begin{tabular}{@{}cccc@{\hspace{9pt}}cccc@{}}
\toprule
\multicolumn{4}{c}{\textbf{Substructure extension}} & \multicolumn{4}{c}{\textbf{Substructure linking}} \\
\cmidrule(lr){1-4}\cmidrule(lr){5-8}
Seed & \textsc{BondEdit} & \textsc{AtomGrow} & Sample \textsc{AtomGrow} & Seed & \textsc{BondEdit} & \textsc{AtomGrow} & Sample \textsc{AtomGrow} \\
\midrule
\adjustbox{valign=m}{\downstreamimage[width=0.0861\linewidth]{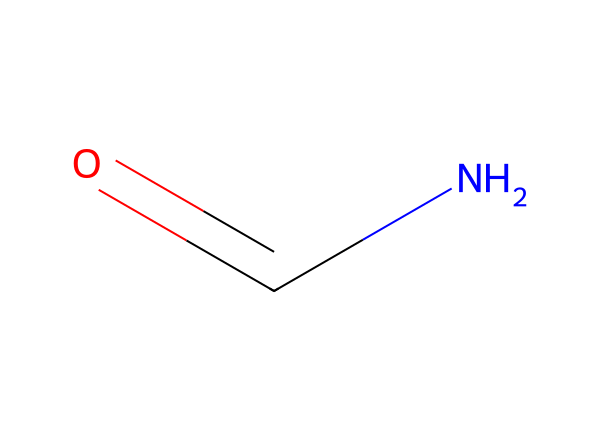}}
 & 50.4 & \textbf{98.4} & \adjustbox{valign=m}{\downstreamimage[width=0.20\linewidth,trim={0 25bp 0 25bp},clip]{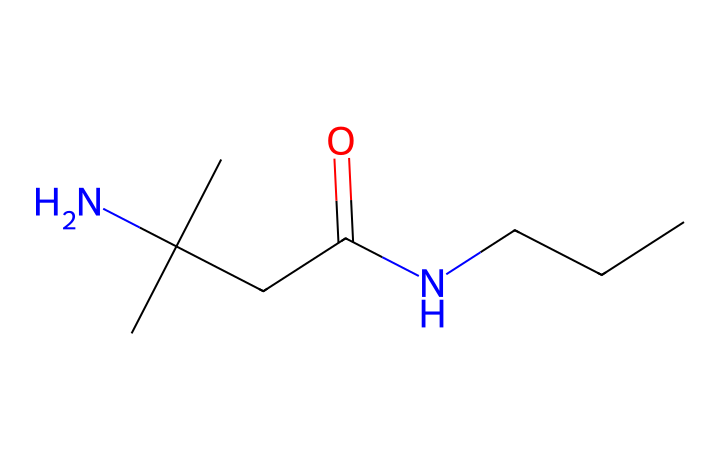}} & \makebox[0.0812\linewidth][c]{\adjustbox{valign=m}{\downstreamimage[width=0.0812\linewidth,viewport=12 646 220 750,clip]{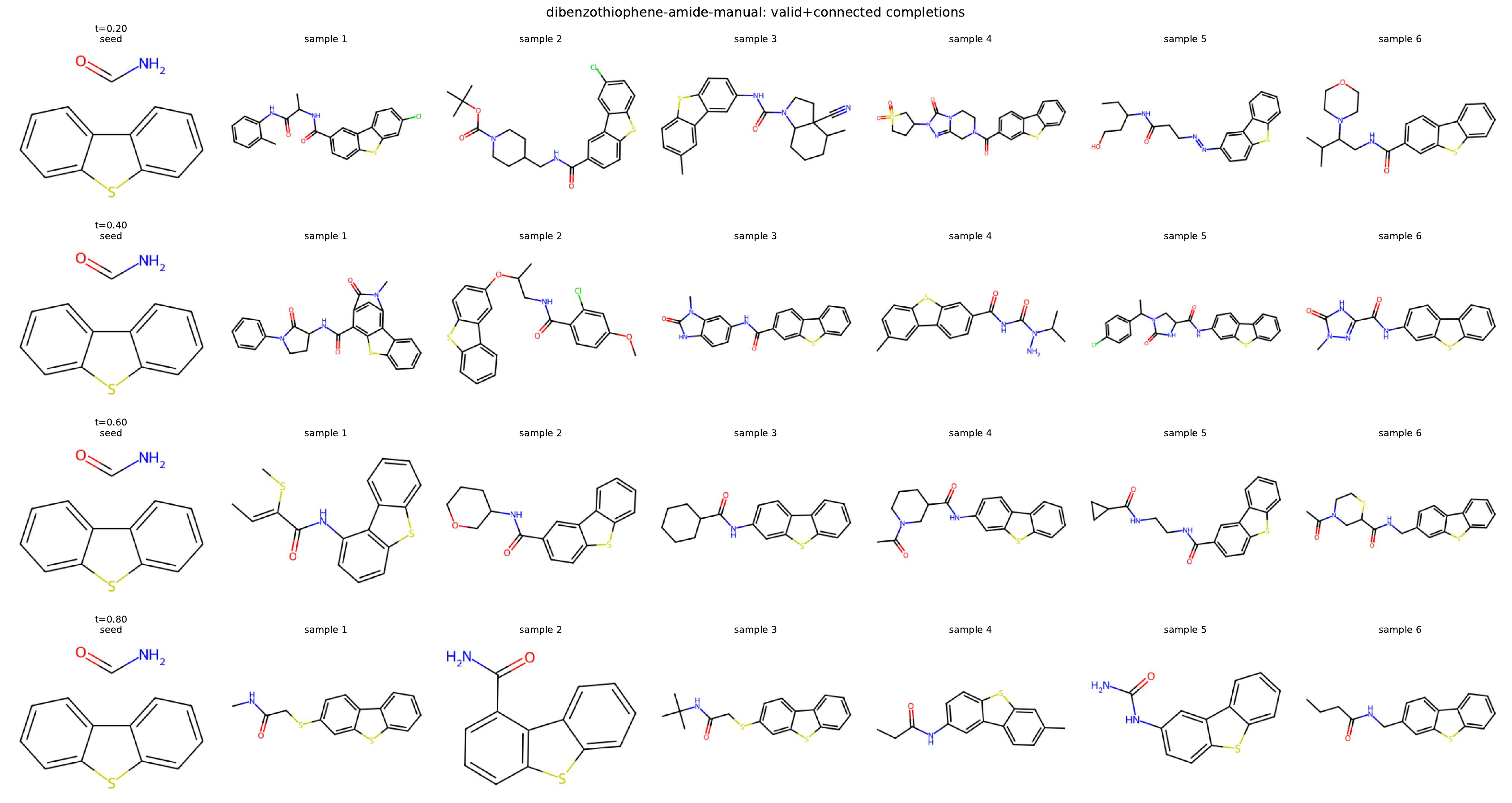}}}\,$+$\,\adjustbox{valign=m}{\downstreamimage[width=0.0431\linewidth,viewport=77 768 179 812,clip]{figures/substructure/dibenzothiophene-amide_linking_grid.pdf}}
 & 39.8 & \textbf{71.1} & \adjustbox{valign=m}{\downstreamimage[width=0.20\linewidth,viewport=240 672 456 768,clip]{figures/substructure/dibenzothiophene-amide_linking_grid.pdf}} \\[-8pt]
\adjustbox{valign=m}{\downstreamimage[width=0.0861\linewidth]{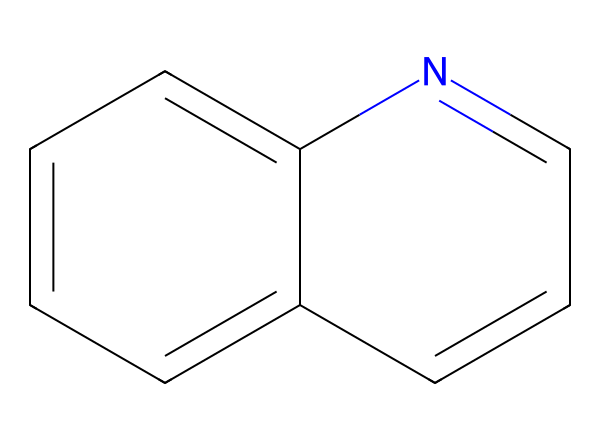}}
 & 31.3 & \textbf{91.8} & \adjustbox{valign=m}{\downstreamimage[width=0.20\linewidth,trim={0 25bp 0 25bp},clip]{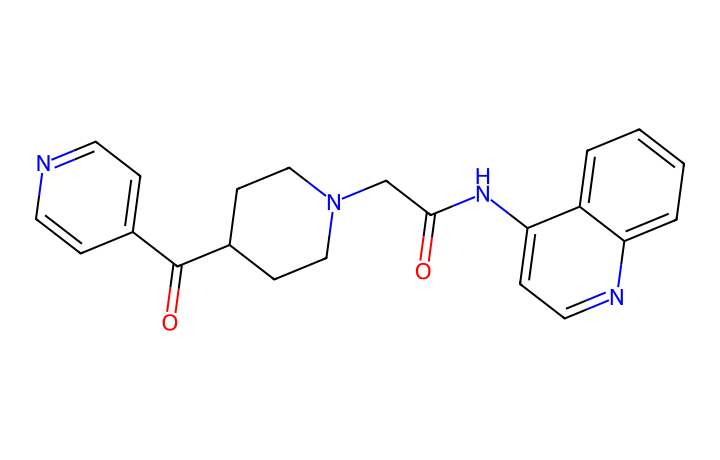}} & \makebox[0.0812\linewidth][c]{\adjustbox{valign=m}{\downstreamimage[width=0.0685\linewidth,viewport=48 728 186 812,clip]{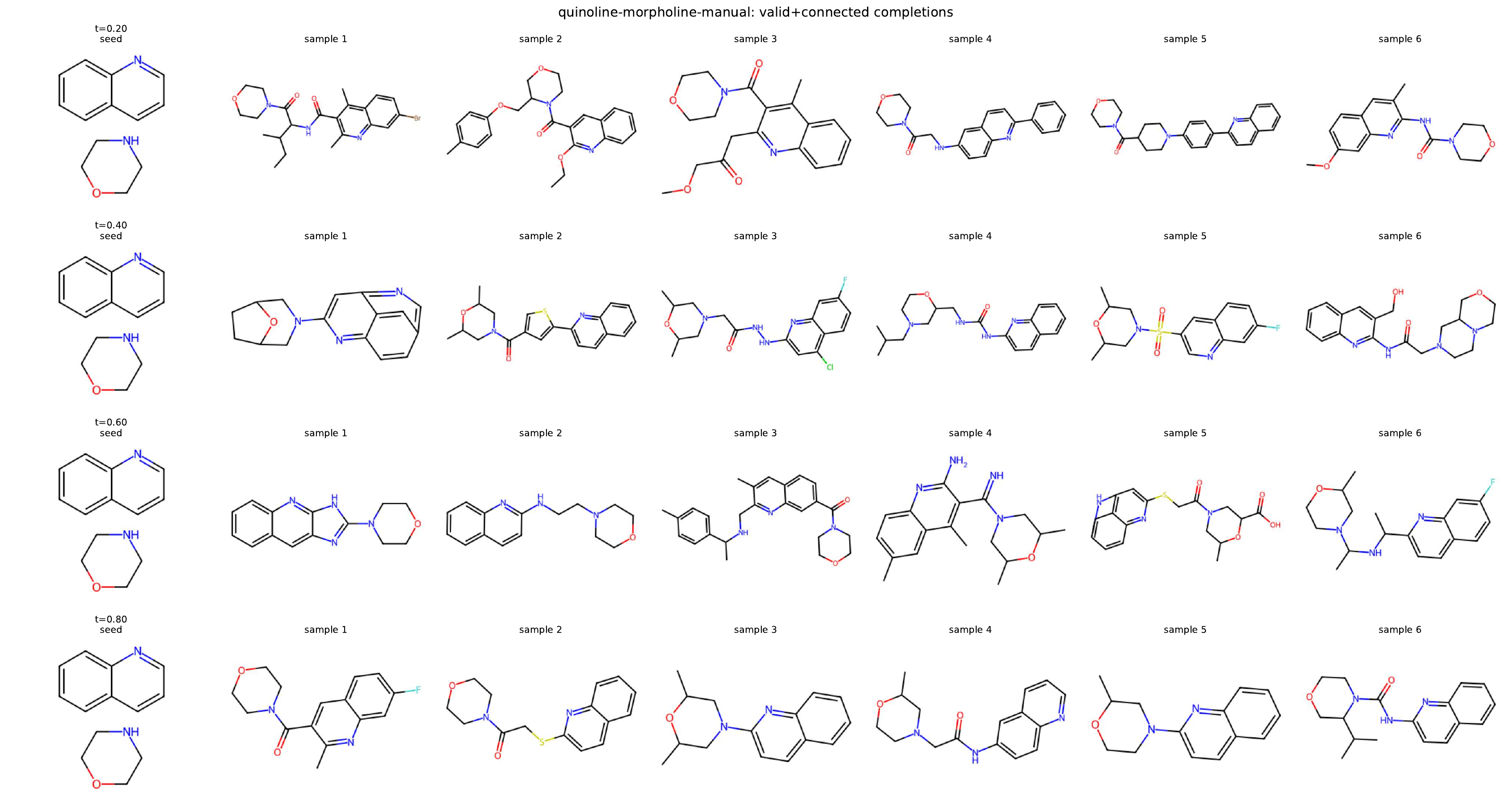}}}\,$+$\,\adjustbox{valign=m}{\downstreamimage[width=0.0431\linewidth,viewport=77 646 159 720,clip]{figures/substructure/quinoline-morpholine-linking_grid.pdf}}
 & 60.9 & \textbf{76.6} & \adjustbox{valign=m}{\downstreamimage[width=0.20\linewidth,viewport=934 672 1152 782,clip]{figures/substructure/quinoline-morpholine-linking_grid.pdf}} \\[-1pt]
\midrule
All 11 seeds & 33.8 & \textbf{94.6} & & All 6 pairs & 57.0 & \textbf{78.8} & \\
\bottomrule
\end{tabular}%
}

\end{minipage}}
\par\vspace{-8pt}
\end{table*}

\textbf{Downstream applications.}
Trajectory matching naturally supports zero-shot transfer to structured
downstream tasks by allowing generation to start from meaningful partial
states. We demonstrate this on two tasks: \textit{substructure extension}
(extend a partial substructure into a valid molecule containing it) and
\textit{substructure linking} (join two disjoint substructures into a
single valid molecule). A single \textsc{AtomGrow} model trained on ZINC
with a random number of root atoms accommodates
both directly: extension by setting \smath{s_I} to the given substructure,
and linking by setting \smath{s_I} to the union of two disconnected
substructures (see \cref{app:atomgrow,app:downstream}).
While \textsc{AtomGrow} can generate extensions and linkers without a fixed
target size, adapting the factorized \textsc{BondEdit} baseline requires
sampling a target size in advance and freezing the conditioning
substructure(s). \textsc{AtomGrow} substantially outperforms the factorized baseline
on both tasks, achieving a valid-and-connected fraction of 94.6\% vs.\ 33.8\%
on extension across 11 substructures, including some not present during
training, and 78.8\% vs.\ 57.0\% on linking across 6 fragment pairs
(see \cref{tab:downstream}).
This works because \textsc{AtomGrow}'s trajectories pass
through valid partial molecules, precisely the states on which these
downstream tasks condition. More generally, trajectory matching makes the
path between endpoints 
a modeling choice: training intermediates can be
chosen to align with downstream conditional tasks, rather than being fixed
by a factorized endpoint-conditioned construction.

\vspace{-0.2cm}
\section{Related Work}
\vspace{-3pt}
\label{sec:related_work}

\textbf{Action-based generation.}
Several lines of work generate structured objects through action sequences, using different training objectives and generation paradigms:
GFlowNets \citep{gflownets,gflownets2} learn distributions over action trajectories on a DAG, but target reward-proportional sampling and are trained via balance constraints rather than conditional matching.
Autoregressive models handle complex actions naturally by generating in a fixed order \citep{junction_tree,grammar_vae,graph_rnn}, but commit to that order and to sequential generation.
Non-autoregressive edit-based models such as DiffusER \citep{diffuser} predict a restricted set of single-token edit operations in parallel, but rely on heuristic reconstruction targets without a principled guarantee that the learned model samples from the data distribution.

\textbf{Discrete flow matching.}
Discrete diffusion \citep{discrete_diffusion} and DFM \citep{discrete_flows,discrete_flows_original}  provide a principled framework for non-autoregressive generation that provably matches the data distribution, and have been successfully applied to text \citep{discrete_diffusion}, graphs \citep{defog,digress}, and molecules \citep{molecule_flow}.  However, these methods share a common template restricting the action space: the conditional dynamics are determined by endpoints alone, factorize across components, and are restricted to single-component updates on a fixed-size state. 
ARDM \citep{ardm} is an exception to endpoint-conditioning: training over random unmasking orders 
implicitly conditions on a trajectory, but is restricted to fixed-size,  single-token unmasking operations.

A small body of work relaxes the fixed-size constraint: Jump diffusion \citep{jump_diffusion} combines continuous-state diffusion with a separate jump process and an auxiliary insertion head that predicts new component values and locations. GRIDDD \citep{griddd} similarly augments fixed-size graph diffusion with an auxiliary model that predicts how many nodes to insert at each step.
Edit Flows \citep{edit_flows}, the closest prior work, handles single-token insertions, substitutions, and deletions for text by lifting them into an augmented state space, where they reduce to factorized single-component updates. 
Variants of this construction have been applied to temporal point processes \citep{tpps}, 3D molecule generation \citep{branching_flows}, and protein generation \citep{evoflows}.
None of these approaches supports multi-component 
actions (e.g., fragment attachment) or general actions with explicit 
prerequisites (e.g., edge insertions requiring both endpoints).
Trajectory matching addresses this gap through full-trajectory
conditioning, providing, to our knowledge, the first general DFM
construction for such actions. It also recovers factorized DFM
and Edit Flow rates as special cases, yielding
a unified framework.

\vspace{-0.2cm}
\section{Limitations}
\vspace{-3pt}
\label{sec:limitations}
Trajectory matching requires specifying both an action space \smath{\mathcal{A}} and a trajectory distribution \smath{\trajDist}, making its performance dependent on domain-specific design choices. The framework is most useful when the actions encode meaningful structure, such as fragment insertions in molecular graphs.  In domains where independent component-wise edits already capture the relevant structure, complex actions may offer no advantage. In this case, trajectory matching can fall back to factorized DFM.

The tractability of trajectory matching depends on being able to sample from \smath{\trajDist} efficiently and, ideally, to compute the Rao-Blackwellized rates in closed form. For many natural choices---independent fixed-action distributions and forest-structured dependencies---both can be done in linear time. Outside these families, both may require domain-specific constructions (e.g., \cref{eq:beyond-fixed}), so applying trajectory matching to a new domain can itself require non-trivial methodological work.

Our empirical evaluation covers trees and molecular graphs across four datasets. While these settings exercise diverse action structures, evaluating the framework in further domains remains future work. %

\vspace{-0.2cm}
\section{Conclusion}
\vspace{-3pt}
We introduced trajectory matching, a discrete flow matching framework that operates over user-defined action spaces by conditioning on full action trajectories rather than on endpoints alone. The framework subsumes factorized DFM and extends it to actions that change state size, modify multiple components, or depend on one another. Empirically, domain-aligned action spaces translate into concrete gains across tree and molecular graph generation. By turning the action vocabulary into a modeling choice, trajectory matching opens up the design space of DFM.

\subsection*{AI use statement}

The idea of this work, the general conceptual framework and its specific instantiations were developed by the authors without any substantial use of generative AI.
The main text was polished with substantial use of generative AI, which provided feedback on clarity, suggested alternative formulations of sentences and paragraphs, and helped to find typos, inconsistencies and errors. Text in the appendix was for some sections (mainly the experimental details and some proofs) in parts drafted by generative AI based on provided implementation details or proof sketches by the authors. E.g., a  draft for the precise description of the experimental setup was generated by generative AI based on the provided implementation. Again, all text generated by generative AI was carefully checked, partly revised, and verified by the authors.

The code for the general framework was written by the authors. The implementation of specific instantiations and experiments was done with the help of AI Coding Assistants. Additionally, we used generative AI to provide feedback on the experiments, find bugs, and interpret some of the results. We did not use AI to generate datasets or clean them.
 We extensively tested the code and use the existing evaluation framework used by other baselines (e.g. DeFoG).

We take responsibility for the final content of this work,
including text, claims or artifacts produced with the aid of generative AI.

\subsection*{Ethics statement}

This work studies foundational aspects of generative modeling of structured discrete objects. Our experiments include molecular graph generation, which can support beneficial applications such as drug discovery but also carries risks of misuse, including the design of harmful compounds. We do not identify ethical concerns specific to our method beyond those associated with generative modeling of molecules in general.

\subsection*{Reproducibility statement}

The mathematical assumptions and proofs supporting our framework are provided in \cref{app:proofs}, with the consistency of the implemented decomposed objective established in \cref{app:setup}. Training and sampling procedures are described in \cref{app:training_sampling}, and the action spaces, source distributions, and trajectory constructions for each instantiation are specified in \cref{app:instantiations}. \Cref{app:setup} documents the datasets and splits, model architecture, training hyperparameters, hardware, sampling budgets, and evaluation protocols. Additional ablations and the downstream evaluation procedure are provided in \cref{app:experimental_details}.
We commit to releasing the code publicly upon publication.

\newpage
\bibliography{references/references}

\newpage
\ifincludeappendix
    \appendix
    \section*{Appendix Contents}
    \small
    \begin{itemize}
        \item \hyperref[app:extended_related_work]{Appendix \ref*{app:extended_related_work}: Extended Related Work}
        \item \hyperref[app:proofs]{Appendix \ref*{app:proofs}: Proofs}
        \item \hyperref[app:kappa-schedules]{Appendix \ref*{app:kappa-schedules}: General Time Schedules}
        \item \hyperref[app:training_sampling]{Appendix \ref*{app:training_sampling}: Training and Sampling}
        \item \hyperref[app:instantiations]{Appendix \ref*{app:instantiations}: Action Space and Trajectory Distribution Instantiations}
        \item \hyperref[app:connections]{Appendix \ref*{app:connections}: Connection to existing methods}
        \item \hyperref[app:editflows]{Appendix \ref*{app:editflows}: Edit Flows as Trajectory Matching}
        \item \hyperref[app:experimental_details]{Appendix \ref*{app:experimental_details}: Additional Experimental Results}
    \end{itemize}
    \normalsize
    \section{Extended Related Work}
\label{app:extended_related_work}

\paragraph{Flow and generator matching.}
Flow matching~\citep{flow1,flow2,flow3} learns marginal dynamics by regressing against tractable conditional dynamics. Generator Matching~\citep{generator_matching} generalizes this conditional-to-marginal principle to broad classes of Markov processes through their infinitesimal generators, encompassing deterministic flows, diffusions, and continuous-time Markov chains on continuous, discrete, and manifold-valued state spaces. Discrete Flow Matching (DFM)~\citep{discrete_flows_original,discrete_flows} can be viewed as an instance of this framework on discrete spaces and has been applied to text, proteins, graphs, and molecules~\citep{discrete_flows,discrete_flows_original,defog}. Widely used DFM constructions define endpoint-conditioned paths on fixed product spaces, factorize the conditional probability paths across components, and allow individual transitions to modify only one categorical component. Within this product-space setting, \citet{kinetic_flows} derive kinetic-optimal rates for general  categorical probability paths within each component.

Trajectory Matching lies within the Generator Matching framework \citep{generator_matching}, with the complete action trajectory serving as the conditioning variable. Our contribution is a general construction of tractable conditional generators over discrete action spaces without requiring a fixed product-space representation or componentwise actions, together with efficient Rao--Blackwellized targets for several important trajectory families.

\paragraph{Variable-size discrete flows.}
A few diffusion and flow models relax the fixed-size assumption using specialized size-changing processes. Jump Diffusion~\citep{jump_diffusion} combines continuous diffusion with a dimension-reducing deletion process and learns its time reversal, including the rate and transition kernel for inserting new components. GRIDDD~\citep{griddd} augments graph diffusion with insertions and deletions and an auxiliary model for predicting graph size changes. 
FlexMDM~\citep{flexmdm} augments masked diffusion for text with a dedicated insertion process: an additional head predicts how many mask tokens to insert between existing positions, while the standard unmasking posterior subsequently predicts their values. Generalized Poisson Flow~\citep{poisson_flows} similarly decomposes variable-size generation into a generalized-Poisson insertion process, which learns the insertion rate and the distribution of newly inserted components, and a within-length flow that refines the current state.

Edit Flows~\citep{edit_flows}, the closest predecessor to our construction, defines a CTMC on variable-length sequences with token insertions, substitutions, and deletions. To obtain tractable conditional dynamics, it lifts sequences to an augmented alignment space, where insertions and deletions are represented as transitions between tokens and a special empty symbol. The resulting conditional probability path factorizes over aligned positions, allowing the standard product-space DFM construction to be applied, with each transition modifying only one component of the augmented state.
Notably, trajectory matching recovers the Edit Flow probability
path and marginal rate by uniformly ordering the alignment-induced
edits for a fixed alignment distribution (\cref{prop:ef-equiv}).

Related size-changing constructions have been applied to temporal point processes \citep{tpps}, proteins \citep{evoflows}, and 3D molecules~\citep{branching_flows}. In particular, \citet{branching_flows} use a specialized atom split-and-deletion process conditioned on a latent forest that determines the ancestry of split atoms. This construction can be seen as a special case of our dependent fixed-action formulation, which allows general action sets and dependency structures.

In summary, these methods provide specialized constructions for domain-specific size-changing operations, rather than a general recipe for training flows with arbitrary multi-component actions and dependencies.

\paragraph{Generation order.}
Generation order is an important modeling choice in discrete generative models. Autoregressive models typically fix a single generation order, whereas Autoregressive Diffusion Models~\citep{ardm} train over multiple componentwise unmasking orders. Standard factorized DFM similarly permits components to be modified in arbitrary order; in our framework, it corresponds to independent actions ordered uniformly at random. Trajectory Matching makes this trajectory distribution explicit by specifying $\trajDist(\trajectory\mid \initState,\finalState)$ over complete domain-level action trajectories, allowing generation orders to encode domain knowledge and dependencies between actions. Jointly learning this trajectory distribution---as recently explored for token insertion orders by LoFlexMDM~\citep{loflexmdm}---is a promising direction for future work.

\paragraph{Fragment-based molecular generation.}
Fragment-based molecular generators construct molecules from chemically meaningful substructures rather than individual atoms. Early approaches include junction-tree and hierarchical variational autoencoders~\citep{junction_tree,junction_tree_similar} and autoregressive models that select and attach structural motifs~\citep{moler}. These methods demonstrate that fragment-level representations can improve chemical validity, scalability, and control over molecular structure, but rely on autoregressive or latent-variable constructions.

Fragments have also been incorporated into diffusion and flow-based models. Fragment-based molecular diffusion and flow matching approaches typically generate a fragment graph based on a fixed vocabulary~\citep{fragment_diffusion,coarse_to_fine} and reconstruct atom-level graphs through seperate coarse-to-fine post-processing. Similarly, \citet{molecule_flow} generate 3D molecular structures by generating a point cloud of rigid motifs.

 These approaches use fragments primarily to define a coarser state representation on which a conventional diffusion or flow process operates. In contrast, our \textsc{FragGrow} instantiation retains an atom-level molecular graph and treats attaching an entire fragment as a single structured action, directly producing the updated atom-level graph without a separate coarse-to-fine decoder.

\section{Proofs}
\label{app:proofs}

Throughout the proofs we use the Kolmogorov forward equations, which characterize when a rate $\rate$ generates a probability path $p_t$. In our action-based parametrization, they take the following form:

\begin{theorem}[Kolmogorov forward equations]
\label{thm:kolmogorov-forward}
A rate $\rate$ generates a probability path $p_t$ if and only if, for every state $s$ and every $t \in [0,1]$,
\[
\frac{d}{dt} p_t(s)
\;=\; \sum_{\substack{s',\,a:\\ a(s')=s}} p_t(s')\, \rate(s', a)
\;-\; p_t(s) \sum_{a \in A(s)} \rate(s, a).
\]
\end{theorem}
\begin{proof}
This is the standard Kolmogorov forward equation for CTMCs , specialized to our action parametrization via $Q_t(s', s) = \sum_{a:\, a(s') = s} \rate(s', a)$ for $s' \neq s$.
\end{proof}

\subsection{Proof of Proposition 3.1}

\trajectorypathprop*

\begin{proof}
Let $\beta_i(t) := \binom{n}{i} t^i (1-t)^{n-i} = \Pr[K_t = i]$. We show that
\[
p_t(s \mid \trajectory) \;=\; \sum_{i=0}^n \indicator{s = s_i}\, \beta_i(t)
\]
satisfies the Kolmogorov forward equations of \cref{thm:kolmogorov-forward} with the trajectory-conditioned rate \cref{eq:trajectory-rate} and initial condition $p_0(\cdot \mid \trajectory) = \indicator{\,\cdot = s_0}$; the claim then follows by uniqueness.

Since the states $s_0, \dots, s_n$ in $\trajectory$ are distinct, $p_t(s \mid \trajectory) = 0$ for $s \notin \{s_0, \dots, s_n\}$, and the forward equation holds trivially at such $s$. For $s = s_i$, the only action flowing in is $a_i$ from $s_{i-1}$ (when $i \geq 1$) and the only action with positive outflow rate is $a_{i+1}$ (when $i \leq n-1$). The forward equation thus reduces to
\[
\frac{d}{dt}\beta_i(t)
\;=\; \frac{n - i + 1}{1 - t}\,\beta_{i-1}(t)
\;-\; \frac{n - i}{1 - t}\,\beta_i(t),
\]
with the convention $\beta_{-1} \coloneq \beta_{n+1} \coloneq 0$. Using the identity $(n-i+1)\binom{n}{i-1} = i\binom{n}{i}$, both sides equal
\[
\binom{n}{i}\Bigl[\,i\, t^{i-1}(1-t)^{n-i} \;-\; (n-i)\, t^{i}(1-t)^{n-i-1}\Bigr],
\]
which verifies the equation. The initial condition $\beta_i(0) = \indicator{i = 0}$ also holds.
\end{proof}

\paragraph{Trajectories with repeated states.}
The assumption that a trajectory never revisits a state is only for notational convenience. For repeated states, we additionally condition on the progress index $K_t=i$, distinguishing visits through the augmented state $(s_i,i)$. The next action $a_{i+1}$ then has rate $(n-i)/(1-t)$, with zero outgoing rate at $i=n$. Averaging these rates over the conditional distribution of $K_t$ given $S_t=s$ and $\trajectory$ yields a rate on the original state space. The resulting rate defines a process on the original state space with the same probability of each state at every time. In particular, it still starts at $\initState$ and ends at $\finalState$.

\trajectorymarginalthm*

\begin{proof}
By \cref{thm:kolmogorov-forward}, since $\rate(\cdot \mid \trajectory)$ generates $p_t(\cdot \mid \trajectory)$ for each $\trajectory$,
\[
\frac{d}{dt} p_t(s \mid \trajectory)
\;=\; \sum_{\substack{s',\,a:\\ a(s')=s}} p_t(s' \mid \trajectory)\, \rate(s', a \mid \trajectory)
\;-\; p_t(s \mid \trajectory) \sum_{a \in A(s)} \rate(s, a \mid \trajectory).
\]
Taking $\EX_{\trajectory \sim \trajDist}$ on both sides,
\[
\frac{d}{dt} p_t(s)
\;=\; \sum_{\substack{s',\,a:\\ a(s')=s}} \EX_\trajectory\!\left[p_t(s' \mid \trajectory)\, \rate(s', a \mid \trajectory)\right]
\;-\; \sum_{a \in A(s)} \EX_\trajectory\!\left[p_t(s \mid \trajectory)\, \rate(s, a \mid \trajectory)\right].
\]
By the definition of $\rate(s, a)$ as a conditional expectation,
\[
\EX_\trajectory\!\left[p_t(s \mid \trajectory)\, \rate(s, a \mid \trajectory)\right] \;=\; p_t(s)\, \rate(s, a),
\]
where on $\{p_t(s) = 0\}$ both sides vanish and $\rate(s, a)$ may be defined arbitrarily. Substituting yields the Kolmogorov forward equation for $p_t$ with rate $\rate$, so by \cref{thm:kolmogorov-forward}, $\rate$ generates $p_t$.
\end{proof}

\subsection{Proof of \cref{prop:traj-gradients}}

\traingradientsprop*

\begin{proof}
Expanding $\loss(\theta)$ and the marginal matching loss, the claim is
\[
\nabla_\theta \EX_\trajectory\, \EX_t\, \EX_{s \sim p_t(\cdot \mid \trajectory)}\, D\!\left(\rate(s, \cdot \mid \trajectory),\, \rate^\theta(s, \cdot)\right)
\;=\; \nabla_\theta \EX_t\, \EX_{s \sim p_t}\, D\!\left(\rate(s, \cdot),\, \rate^\theta(s, \cdot)\right).
\]
Since $D$ is a Bregman divergence, $D(a, b) = \phi(a) - \phi(b) - \langle a - b,\, \nabla \phi(b)\rangle$ for some strictly convex $\phi$. The terms $\phi(a)$ on each side do not depend on $\theta$ and vanish under $\nabla_\theta$. The remaining $\theta$-dependent terms are
\[
- \phi(\rate^\theta(s, \cdot)) - \langle a - \rate^\theta(s, \cdot),\, \nabla \phi(\rate^\theta(s, \cdot))\rangle,
\]
whose $\theta$-gradient is affine in $a$ for fixed $(s,t)$.
By \cref{thm:trajectory-marginal},
\[
\EX[\rate(s,\cdot\mid\trajectory)\mid S_t=s,t]
=
\rate(s,\cdot).
\]
Consequently,
\[
\EX\!\left[
\nabla_\theta D\!\left(
\rate(s,\cdot\mid\trajectory),\,\rate^\theta(s,\cdot)
\right)
\,\middle|\, S_t=s,t
\right]
=
\nabla_\theta D\!\left(
\rate(s,\cdot),\,\rate^\theta(s,\cdot)
\right).
\]
Taking expectations over $t$ and $s\sim p_t$ proves the claim,
with differentiation and expectation interchanged under the usual
regularity conditions.
\end{proof}

\subsection{Proof of \cref{cor:traj-optima}}

\trainoptimacor*

\begin{proof}
By \cref{prop:traj-gradients}, $\loss(\theta)$ and the marginal
matching loss have the same gradient in $\theta$, hence the same
minimizers. The marginal matching loss is the expectation under
$\EX_t\,\EX_{s \sim p_t}$ of
$D(\rate(s,\cdot),\,\rate^\theta(s,\cdot))$.
Since $D$ is a Bregman divergence, $D(a,b)\geq 0$ with equality
iff $a=b$, so the integrand is pointwise minimized by
$\rate^\theta(s,\cdot)=\rate(s,\cdot)$ for almost every $(s,t)$
with $p_t(s)>0$.

If the family contains $\rate$, this pointwise minimum is
attained, so any minimizer $\theta^\star$ satisfies
$\rate^{\theta^\star}(s,\cdot)=\rate(s,\cdot)$
for almost every $(s,t)$ with $p_t(s)>0$.
Consequently, the probability fluxes
$p_t(s)\rate^{\theta^\star}(s,a)$ and $p_t(s)\rate(s,a)$
agree almost everywhere in time. Therefore, $p_t$ also
satisfies the Kolmogorov forward equations, in integral form,
for $\rate^{\theta^\star}$. Since both processes start from
$\noiseDist$, uniqueness implies that
$\rate^{\theta^\star}$ generates the same marginal path $p_t$.
\end{proof}

\subsection{Rao--Blackwellization: consistency and variance reduction}
\label{app:rao-blackwell}

We establish the guarantees for the Rao--Blackwellized training target introduced in \cref{sec:rao-blackwell}. Averaging over trajectories conditional on the current state, time, and endpoints preserves the expected gradient and the optimum of the trajectory-matching objective, while reducing gradient variance.

\begin{restatable}[Consistency of the Rao-Blackwellized loss]{proposition}{raoblackwellconsistencyprop}
\label{prop:rao-blackwell-consistency}
\cref{prop:traj-gradients} and \cref{cor:traj-optima} continue to hold when \smath{\rate(\cdot \mid \trajectory)} is replaced by \smath{\rbrate(\cdot \mid \initState, \finalState)} in \cref{eq:trajectory_matching_objective}.
\end{restatable}

\begin{proof}
By the tower property, $\rate(\cdot \mid \trajectory)$ and $\rbrate(\cdot \mid \initState, \finalState)$ share the same conditional expectation given $S_t = s$, namely the marginal rate $\rate(s, \cdot)$. The proof of \cref{prop:traj-gradients} uses $\rate(\cdot \mid \trajectory)$ only through this conditional expectation, so the same gradient identity holds with $\rbrate(\cdot \mid \initState, \finalState)$ in place of $\rate(\cdot \mid \trajectory)$. \cref{cor:traj-optima} then follows by the same argument.
\end{proof}

\begin{restatable}[Variance reduction]{proposition}{raoblackwellvarianceprop}
\label{prop:rao-blackwell-variance}
Let \smath{\hat{g}_\text{TM}(\theta) = \nabla_\theta D(\rate(s, \cdot \mid \trajectory), \rate^\theta(s, \cdot))} and \smath{\hat{g}_\text{RB}(\theta) = \nabla_\theta D(\rbrate(s, \cdot \mid \initState, \finalState), \rate^\theta(s, \cdot))} be the gradient estimators under the trajectory-conditioned and Rao-Blackwellized targets, sampled from the distribution defining \smath{\loss(\theta)}. Then
\begin{align*}
\EX[\hat{g}_\text{RB}] = \EX[\hat{g}_\text{TM}] \quad \text{and} \quad \mathrm{Var}(\hat{g}_\text{RB}) \leq \mathrm{Var}(\hat{g}_\text{TM}) \text{ componentwise}.
\end{align*}
\end{restatable}

\begin{proof}
Since $D$ is a Bregman divergence, $\nabla_\theta D(a, \rate^\theta(s, \cdot))$ is linear in $a$. Therefore, conditioning on $(s, t, \initState, \finalState)$ and using the definition of the Rao-Blackwellized target $\rbrate(s, a \mid \initState, \finalState) = \EX[\rate(s, a \mid \trajectory) \mid s, t, \initState, \finalState]$,
\[
\EX[\hat{g}_\text{TM} \mid s, t, \initState, \finalState] \;=\; \nabla_\theta D\!\left(\EX[\rate(s, \cdot \mid \trajectory) \mid s, t, \initState, \finalState],\, \rate^\theta(s, \cdot)\right) \;=\; \hat{g}_\text{RB}.
\]
Equality of expectations follows by the tower property: $\EX[\hat{g}_\text{TM}] = \EX[\EX[\hat{g}_\text{TM} \mid s, t, \initState, \finalState]] = \EX[\hat{g}_\text{RB}]$. The variance inequality follows from the law of total variance applied componentwise:
\[
\mathrm{Var}(\hat{g}_\text{TM}) \;=\; \EX\bigl[\mathrm{Var}(\hat{g}_\text{TM} \mid s, t, \initState, \finalState)\bigr] + \mathrm{Var}\bigl(\EX[\hat{g}_\text{TM} \mid s, t, \initState, \finalState]\bigr) \;\geq\; \mathrm{Var}(\hat{g}_\text{RB}). \qedhere
\]
\end{proof}

\subsection{Proof of \cref{prop:independent_rates}}

\independentratesprop*

\begin{proof}
Fix $\initState, \finalState$ and a state $s$ on some trajectory from $\initState$ to $\finalState$ in the support of $\trajDist$. Let $m := |\actionSet(s, \finalState)|$ be the number of actions remaining at $s$, and let $\ell := |\actionSet(\initState, \finalState)| - m$ be the number of actions already taken to reach $s$. Since $\trajDist$ is an independent fixed-action distribution, every trajectory passing through $s$ on its way to $\finalState$ takes the same set $\actionSet(s, \finalState)$ of remaining actions; only their ordering varies, and the distribution over orderings is uniform.

By symmetry, the posterior $p_t(\trajectory \mid s, \initState, \finalState)$ is uniform over orderings of $\actionSet(s, \finalState)$, so each $a \in \actionSet(s, \finalState)$ is equally likely to be the next action: $\Pr[a_{\ell+1} = a \mid s, \initState, \finalState] = 1/m$.

The trajectory-conditioned rate at $s$ (position $\ell$ on a trajectory of length $\ell + m$) is, by \cref{eq:trajectory-rate},
\[
\rate(s, a \mid \trajectory) = \indicator{a = a_{\ell+1}} \cdot \frac{m}{1-t}.
\]
Taking the posterior expectation,
\begin{align*}
\rbrate(s, a \mid \initState, \finalState) &\;=\; \EX_{\trajectory \sim p_t(\cdot \mid s, \initState, \finalState)}\, \rate(s, a \mid \trajectory) \\
&\;=\; \indicator{a \in \actionSet(s, \finalState)} \cdot \frac{1}{m} \cdot \frac{m}{1-t} \\
&\;=\; \indicator{a \in \actionSet(s, \finalState)} \cdot \frac{1}{1-t}. \qedhere
\end{align*}
\end{proof}

\subsection{Proof of \cref{prop:factorized-flow-equivalence}}

\factorizedflowequivalenceprop*

\begin{proof}
We first prove Part 2; Part 1 then follows from \cref{prop:rao-blackwell-consistency}.
\textbf{Part 2.} By \cref{prop:independent_rates}, the Rao-Blackwellized rate is
\[
    \rbrate(s, a^i_v \mid \initState, \finalState)
    = \indicator{a^i_v \in \actionSet(s, \finalState)} \cdot \frac{1}{1-t}
    = \indicator{s^i \neq v = \finalState^i} \cdot \frac{1}{1-t},
\]
which is exactly the factorized flow rate from \cref{eq:factorized-flow-rate}.
It remains to check that the sampling distribution over $(s, t, \initState, \finalState)$
agrees in the two losses. Both draw $(\initState, \finalState) \sim \coupling$ and
$t \sim U[0,1]$; it suffices to show the conditional paths agree, i.e.,
\[
\EX_{\trajectory \sim \trajDist(\cdot \mid \initState, \finalState)} p_t(s \mid \trajectory)
= \prod_{i=1}^d p_t^i(s^i \mid \initState^i, \finalState^i).
\]

Fix $(\initState, \finalState)$ and let $F = \{i : \initState^i \neq \finalState^i\}$
with $n_F = \abs{F}$. Every $\trajectory$ in the support of
$\trajDist(\cdot \mid \initState, \finalState)$ is an ordering of
$\{a^i_{\finalState^i} : i \in F\}$. By \cref{prop:trajectory-path},
conditional on $\trajectory$ and $K_t = k$, the current state has flipped exactly the
components corresponding to the first $k$ actions in $\trajectory$. Averaging over
$\trajectory$ uniformly over orderings, the identity of the $k$ flipped components is
uniform over size-$k$ subsets of $F$. Since
$K_t \sim \text{Bin}(n_F, t)$, this is equivalent to flipping each $i \in F$
independently with probability $t$, which yields the factorized path
$\prod_{i=1}^d p_t^i(s^i \mid \initState^i, \finalState^i)$.

\textbf{Part 1.} Follows from Part 2 and \cref{prop:rao-blackwell-consistency}.
\end{proof}

\subsection{Rao-Blackwellized rates for general dependent fixed-action distributions}
\label{app:general-dag-rates}

We will first prove a more general version of \cref{prop:dependent-fixed-action-forest-rates}, which applies to any dependent fixed-action distribution, not just those with forest-structured dependency DAGs. 

Let $L(G)$ denote the number of linear extensions (topological orderings) of a DAG $G$, and let $G \setminus a$ denote the DAG obtained by removing vertex $a$
together with its incident edges.

\begin{restatable}{proposition}{dependentfixedactiongeneralratesprop}
    \label{prop:dependent-fixed-action-general-rates}
    Let $\trajDist$ be a dependent fixed-action distribution with dependency DAG
    on $\actionSet(\initState, \finalState)$. For a state $s$ on a trajectory
    from $\initState$ to $\finalState$,
    \begin{align*}
        \rbrate(s, a \mid \initState, \finalState)
        = \indicator{a \in \textup{roots}(\actionSet(s, \finalState))}
          \cdot \frac{L(\actionSet(s, \finalState) \setminus a)}{L(\actionSet(s, \finalState))}
          \cdot \frac{m}{1-t},
    \end{align*}
    where $m = |\actionSet(s, \finalState)|$ is the number of remaining actions.
\end{restatable}

\begin{proof}
Fix $\initState, \finalState$ and a state $s$ on a trajectory in the support of
$\trajDist(\cdot \mid \initState, \finalState)$. Let
$\ell = |\actionSet(\initState, \finalState)| - m$ be the number of actions
already taken to reach $s$, so every trajectory through $s$ has length
$\ell + m$ and reaches $s$ at position $\ell$.

Since $\trajDist$ is uniform over topological orderings of
$\actionSet(\initState, \finalState)$, the posterior
$p_t(\trajectory \mid s, \initState, \finalState)$ is uniform over those orderings
whose first $\ell$ actions reach $s$. Each such ordering decomposes into a
prefix of length $\ell$ (an ordering of the actions already taken) and a
suffix of length $m$ (a topological ordering of $\actionSet(s, \finalState)$).
Marginalizing out the prefix, the suffix is uniform over linear extensions of
$\actionSet(s, \finalState)$. Hence for any root $a$ of $\actionSet(s, \finalState)$,
\[
    \Pr[a_{\ell+1} = a \mid s, \initState, \finalState]
    = \frac{L(\actionSet(s, \finalState) \setminus a)}{L(\actionSet(s, \finalState))},
\]
since orderings with $a$ first are in bijection with linear extensions of
$\actionSet(s, \finalState) \setminus a$. For non-roots this probability is
zero.

The trajectory-conditioned rate at $s$ is, by \cref{eq:trajectory-rate},
$\rate(s, a \mid \trajectory) = \indicator{a = a_{\ell+1}} \cdot m / (1-t)$.
Taking the posterior expectation,
\[
    \rbrate(s, a \mid \initState, \finalState)
    = \indicator{a \in \textup{roots}(\actionSet(s, \finalState))}
      \cdot \frac{L(\actionSet(s, \finalState) \setminus a)}{L(\actionSet(s, \finalState))}
      \cdot \frac{m}{1-t}. \qedhere
\]
\end{proof}

We note that counting linear extensions of a general DAG is computationally intractable in general, so
\cref{prop:dependent-fixed-action-general-rates} does not yield an efficient
training target in general. However, if the dependency DAG is a forest, the linear-extension ratio admits a
closed form, recovering:

\dependentfixedactionforestratesprop*

\begin{proof}
By \cref{prop:dependent-fixed-action-general-rates}, it suffices to show that
for any root $a$ of a forest $G$ on $m$ vertices, the probability that $a$
comes first in a uniform linear extension of $G$ equals $\textup{desc}(a)/m$.

Let $T_1, \ldots, T_r$ be the trees of $G$ with roots $\rho_1, \ldots, \rho_r$
and sizes $m_1, \ldots, m_r$. Any linear extension of $G$ is uniquely
determined by a linear extension of each $T_j$ together with an interleaving
specifying, for each of the $m$ output positions, which tree contributes that
position. The interleaving is a sequence in $[r]^m$ containing $m_j$ copies of
$j$, and conversely every such interleaving combined with linear extensions of
the $T_j$ yields a valid linear extension of $G$. Under a uniform linear
extension of $G$, both the per-tree extensions and the interleaving are
independently uniform.

By symmetry over the $m$ positions, the interleaving assigns position $1$ to
tree $T_j$ with probability $m_j / m$. Conditional on position $1$ coming
from $T_j$, the first vertex is the minimum of $T_j$'s linear extension,
which is the root $\rho_j$ (the unique minimum of any rooted tree's partial
order). Hence
\[
    \Pr[\rho_j \text{ first}] = \frac{m_j}{m} = \frac{\textup{desc}(\rho_j)}{m}.
\]
Since $a$ is the root of some $T_j$, substituting into
\cref{prop:dependent-fixed-action-general-rates} gives the claim.
\end{proof}

For the forest-plus-sink dependencies in \textsc{AtomGrow} and \textsc{MolEdit}, \cref{app:non-forest-approx} gives an alternative conditional path with exact tractable rates.

    \section{General Time Schedules}
\label{app:kappa-schedules}

Factorized flows are commonly used with a general time schedule $\kappa : [0,1] \to [0,1]$ satisfying $\kappa_0 = 0$, $\kappa_1 = 1$, and $\dot\kappa_t > 0$. The conditional rate of \cref{eq:factorized-flow-rate} then takes the form
\begin{equation}
\rate(s, a \mid \initState, \finalState) \;=\; \indicator{a^i(s^i) = s_F^i} \cdot \frac{\dot\kappa_t}{1 - \kappa_t},
\label{eq:factorized-kappa}
\end{equation}
which generates the factorized path with marginals $\kappa_t \indicator{s^i = s_F^i} + (1 - \kappa_t) \indicator{s^i = s_I^i}$. This general schedule is used, for example, by Edit Flows \citep{edit_flows} and schedules like the polydec schedule have been shown to outperform the linear schdedule for some tasks \citep{defog}.

\paragraph{Reinterpretation as a time reparametrization.} The factor $\dot\kappa_t / (1 - \kappa_t)$ is precisely what one obtains by substituting $t \mapsto \kappa_t$ in the linear-schedule rate $1/(1-t)$ and applying the chain rule. More generally, if $\rate$ generates a probability path $p_t$ on $[0,1]$, then by the chain rule applied to the Kolmogorov forward equation (\cref{thm:kolmogorov-forward}),
\begin{equation*}
\tilde u_t(s, a) := \dot\kappa_t \cdot u_{\kappa_t}(s, a) \quad\text{generates}\quad \tilde p_t(s) := p_{\kappa_t}(s),
\end{equation*}
and endpoints are preserved since $\kappa_0 = 0$ and $\kappa_1 = 1$. Applying this transformation to the linear-schedule trajectory rate of \cref{eq:trajectory-rate} yields the natural $\kappa$-scheduled trajectory rate.

\paragraph{$\kappa$-scheduled trajectory rate.} Probabilistically, this corresponds to drawing the action execution times $0 < t_1 < \cdots < t_n < 1$ as the order statistics of $n$ i.i.d.\ random variables with CDF $\kappa_t$ (rather than $\mathrm{U}[0,1]$). The resulting trajectory-conditioned rate is
\begin{equation}
\rate(s, a \mid \trajectory)
\;=\; \sum_{i=0}^{n-1} \indicator{s = s_i \,\wedge\, a = a_{i+1}} \cdot (n - i) \cdot \frac{\dot\kappa_t}{1 - \kappa_t},
\label{eq:traj-rate-kappa}
\end{equation}
which generates the conditional path
\begin{equation}
p_t(s \mid \trajectory) \;=\; \sum_{i=0}^{n} \indicator{s = s_i} \cdot \mathbb{P}[K_t = i],
\qquad K_t \sim \mathrm{Bin}(n, \kappa_t).
\label{eq:traj-path-kappa}
\end{equation}
The linear schedule $\kappa_t = t$ recovers \cref{eq:trajectory-rate,prop:trajectory-path}.

\paragraph{All results extend.} Because the schedule enters \cref{eq:traj-rate-kappa} only through the scalar prefactor $\dot\kappa_t / (1 - \kappa_t)$, every result in \cref{sec:trajectory-matching,sec:instantiations} carries over by replacing $1/(1-t)$ with $\dot\kappa_t / (1 - \kappa_t)$ throughout.
The $\kappa$ schedule slots into trajectory matching exactly as it does into factorized flows, without changing the structure of any proof.

    \section{Training and Sampling}
\label{app:training_sampling}

\begin{figure}[tbp]
    \centering
    \begin{minipage}[t]{0.41\textwidth}
        \refstepcounter{algorithm}
        \label{alg:trajectory_matching_training}
        \small
        \textbf{Algorithm \thealgorithm.} Training procedure for trajectory matching

        {\algtext*{EndWhile}\input{algorithms/training}}
    \end{minipage}
    \hfill
    \begin{minipage}[t]{0.55\textwidth}
        \refstepcounter{algorithm}
        \label{alg:trajectory_matching_sampling}
        \small
        \textbf{Algorithm \thealgorithm.} Sampling procedure for trajectory matching

        \input{algorithms/sampling}
    \end{minipage}
    \caption{Pseudocode for training and sampling in trajectory matching.}
    \label{fig:trajectory_matching_algorithms}
\end{figure}

\cref{alg:trajectory_matching_training,alg:trajectory_matching_sampling} give the
training and sampling procedures for trajectory matching. We note a few
practical considerations.

\paragraph{Training.}
The training loop is written in terms of the trajectory $\trajectory$
to make the connection to the theory explicit, but the trajectory itself
does not necessarily need to be materialized. In practice, we use a single
sampler that takes $(\initState, \finalState) \sim \coupling$ and
directly returns a tuple $(t, s, \rate(s, \cdot \mid \trajectory))$ —
or $(t, s, \rbrate(s, \cdot \mid \initState, \finalState))$ when using
the Rao-Blackwellized target.
Depending on the chosen trajectory distribution, the sampler can often be efficiently implemented without explicitly constructing the full trajectory $\trajectory$.
For most instantiations of interest, the
target rate at $s$ depends only on quantities such as the remaining
action set $\actionSet(s, \finalState)$ or the residual dependency DAG,
which can be computed directly from the endpoints and the sampled state.

Because the sampling routine is independent of the model, it is well
suited to run in CPU dataloader workers in parallel with the forward and
backward passes on the GPU. In our QM9 experiments, the observed loader
wait is below 1 ms per training step, so sampling and target construction
add very little exposed wall-clock time (\cref{app:runtime}).

\paragraph{Sampling.}
At each Euler step, the inner loop over $\mathcal{A}(s)$ in
\cref{alg:trajectory_matching_sampling} can be executed in parallel: each available
action is independently retained with probability
$\min\{\rate^\theta(s, a) \cdot h, 1\}$, and these draws are independent
Bernoullis. In most of our settings the sampled actions are also
\emph{compatible} — they can be applied simultaneously to $s$ without
conflict — either by design of the action space or as the typical case.
For \textsc{BondEdit}, \textsc{AtomGrow}, \textsc{FragmentGrow}, and \textsc{GrowHierarchy}, every action in $A(s)$ is independent of every other,
so any subset can be applied in parallel and the result does not depend on the application order.

For \textsc{MolEdit}, the action space mixes deletions and additions and a
sampled subset can occasionally be incompatible — for example, an
\texttt{AddBond} action targeting an atom that another sampled
\texttt{DelAtom} action removes. In practice this is rare, and we
resolve it deterministically by applying all deletions before all
additions, and skipping any action that is no longer in $A(s)$
by the time it would be applied. This keeps the parallel structure of
the inner loop while guaranteeing a well-defined update.

    \section{Action Space and Trajectory Distribution Instantiations}
\label{app:instantiations}

In this appendix we give precise definitions of the action spaces, source
distributions, couplings, and trajectory distributions for the four settings
visualized in \cref{fig:training}: \textsc{BondEdit}, \textsc{FragmentGrow},
\textsc{AtomGrow}, and \textsc{MolEdit}. We deliberately abstract away from atom
and bond features in this beginning of this section to simplify the expostion and treat all graphs as unlabeled. The
extension to feature-labeled graphs is discussed subsequently.

\paragraph{Atom-id correspondence.}
Throughout, we identify the correspondence between atoms in $\initState$ and $\finalState$ by index.
Atoms with the same index are assumed to transform into each other along a training trajectory.  Thus, the atom ids given by the coupling $\coupling$ decide what
 unambiguously what atom is ``kept'',
``deleted'', or ``added'', and what bond is ``already present in $\initState$'' or
``required in $\finalState$''.
This is the standard setting in discrete graph diffusion and flow matching approaches \citep{defog,digress}.

We represent a (partial) molecular graph as $s = (V(s), E(s))$ with $V(s)$ a set of
atom ids and $E(s)$ a set of unordered pairs $\{i,j\} \subseteq V(s)$.

\paragraph{Two recurring source distributions.}
Two source distributions appear repeatedly across the four settings:
\begin{itemize}
\item \emph{Single-atom source} ($\noiseDist^{\text{atom}}$): samples a graph
$(\{i_0\}, \emptyset)$ consisting of a single atom. Used by \textsc{FragmentGrow}
and \textsc{AtomGrow}.
\item \emph{Marginal-graph source} ($\noiseDist^{\text{marg}}$): first samples
the number of atoms $n$ from the empirical distribution of $|V(\cdot)|$ in the
training set, then samples each potential edge $\{i,j\}$ independently with
probability $\bar p_{\mathrm{edge}}$, the average edge density in the training
set. Used by \textsc{BondEdit} and \textsc{MolEdit}.
\end{itemize}

\subsection{BondEdit}
\label{app:bondedit}

\textsc{BondEdit} is the simplest setting: the atom set is fixed by the coupling and
only edges are inserted or deleted. It mirrors the action space of DeFoG~\citep{defog}.

\paragraph{States.} For every $(\initState, \finalState)$ in the support of $\coupling$
we have $V(\initState) = V(\finalState)$. Every intermediate state $s$ on a trajectory
satisfies $V(s) = V(\initState) = V(\finalState)$.

\paragraph{Source distribution.}
We use the marginal-graph source $\noiseDist = \noiseDist^{\text{marg}}$.

\paragraph{Coupling.}
$\coupling$ is the size-matched coupling: $\initState \sim \noiseDist^{\text{marg}}$
and $\finalState$ is sampled uniformly from training molecules with
$|V(\finalState)| = |V(\initState)|$, with atom ids matched arbitrarily (e.g.\ in
the order they appear in each graph).

\paragraph{Action space.} At a state $s$, the available actions are
\[
\mathcal{A}(s) \;=\;
\bigl\{\,\texttt{AddBond}_{ij} \;:\; \{i,j\} \notin E(s)\,\bigr\}
\;\cup\;
\bigl\{\,\texttt{DelBond}_{ij} \;:\; \{i,j\} \in E(s)\,\bigr\},
\]
with
$\texttt{AddBond}_{ij}(s) = (V(s), E(s) \cup \{\{i,j\}\})$
and
$\texttt{DelBond}_{ij}(s) = (V(s), E(s) \setminus \{\{i,j\}\})$.

\paragraph{Required actions.} The set of actions transforming $s$ into $\finalState$ is
\[
\actionSet(s, \finalState) \;=\;
\bigl\{\,\texttt{AddBond}_{ij} : \{i,j\} \in E(\finalState) \setminus E(s)\,\bigr\}
\,\cup\,
\bigl\{\,\texttt{DelBond}_{ij} : \{i,j\} \in E(s) \setminus E(\finalState)\,\bigr\}.
\]

So the adding all bonds that are still missing and deleting all bonds that are not present in $\finalState$.

\paragraph{Trajectory distribution.}
Any two actions in $\actionSet(\initState, \finalState)$ act on disjoint edges, so
they are pairwise independent in the sense of \cref{sec:background}. We take
$\trajDist(\cdot \mid \initState, \finalState)$ to be the \emph{independent fixed-action
distribution} (\cref{def:independent-fixed-action-distribution}) over
$\actionSet(\initState, \finalState)$ and use the Rao-Blackwellized rates of \cref{prop:independent_rates}.

\subsection{FragmentGrow}
\label{app:fragmentgrow}

\textsc{FragmentGrow} produces a molecule from a single seed atom by attaching
fragments which capture all ring substructures and individual atoms. We assume a fixed fragment library $\mathcal{F}$
(rings and ring systems extracted from the training set; in practice
$|\mathcal{F}| < 300$). Molecules whose ring system cannot be expressed using
$\mathcal{F}$ are skipped during training.
The fragment vocabulary covers 94\% of QM9, 99\% of ZINC, and 90\% of MOSES
training molecules; the remaining 6\%, 1\%, and 10\%, respectively, are skipped
during FragGrow training. Validation and test sets are not filtered: all
evaluations use the complete standard reference splits, including molecules
not representable by the fragment vocabulary.

\paragraph{States.} The source $\initState$ is a single-atom graph
$(\{i_0\}, \emptyset)$ for some seed atom $i_0$ that is also present in $\finalState$.
Every intermediate state is a connected subgraph of $\finalState$ containing $i_0$.

\paragraph{Source distribution.}
We use the single-atom source $\noiseDist = \noiseDist^{\text{atom}}$.

\paragraph{Coupling.}
$\finalState$ is sampled from the training distribution and the seed atom $i_0$
is sampled uniformly from $V(\finalState)$.

\paragraph{Action space.} At a state $s$, the available actions are
\begin{itemize}
\item $\texttt{AddAtom}_i$ for each $i \in V(s)$, which adds a new atom $j \notin V(s)$
together with a single bond $\{i, j\}$ to the parent atom $i$.
\item $\texttt{AddFragment}^F_i$ for each $i \in V(s)$ and each fragment $F \in \mathcal{F}$
that admits an atom-attachment, which inserts a copy of $F$ joined to $s$ at atom $i$.
\item $\texttt{AddFragment}^F_{ij}$ for each $\{i,j\} \in E(s)$ and each fragment
$F \in \mathcal{F}$ that admits a bond-attachment, which inserts a copy of $F$ joined
to $s$ along the bond $\{i,j\}$.
\end{itemize}
Each of these actions strictly grows the state. 

\paragraph{Decomposition.}
\begin{figure}[tb]
\centering
\subfigure[Decomposition]{\includegraphics[width=0.19\textwidth]{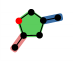}}
\hfill
\subfigure[Dependency tree]{\includegraphics[width=0.19\textwidth]{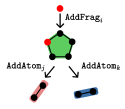}}
\caption{FragmentGrow decompostion and dependency tree: Given the decomposition of $\finalState$ into fragments and atom/bonds, there is one unique set of actions to generate $\finalState$ from $\initState$ (marked in red), and the dependencies between these actions form a tree.}
\label{fig:fragmentgrow-structure}
\end{figure}
Given $\finalState$ and a chosen root atom $i_0 \in V(\finalState)$, there is a
\emph{unique hierarchical decomposition} of $\finalState$ into a structural action tree,
computed as follows:
\begin{enumerate}
\item Initialize the ring units as the SSSR (smallest set of smallest rings) rings of $\finalState$.
\item Repeatedly merge ring units whose overlap is more than a single shared atom
or a single shared bond, and merge any remaining cycle in the ring-overlap graph
into a single larger ring fragment. After this step, distinct ring fragments share
at most one atom or one bond.
\item Build the structural action tree by a BFS-like traversal from $i_0$:
non-ring bonds become \texttt{AddAtom} children, and ring fragments incident to
the current frontier become \texttt{AddFragment} children attached either at an
atom (if the ring shares one atom with the parent) or at a bond
(if the ring shares one bond with the parent).
\end{enumerate}

\paragraph{Required actions.}
The structural action tree above determines
$\actionSet(\initState, \finalState)$ (one action per non-root unit of the
decomposition). The structural action tree also indicated the dependency relationships.
This procedure is the key to obtaining a tree  of actions:
because each ring system is absorbed into a single atomic action, no action has
two prerequisites and there is one unique set of actions to generate $\finalState$ from $\initState$.

\paragraph{Trajectory distribution.}
We take $\trajDist(\cdot \mid \initState, \finalState)$ to be the \emph{dependent
fixed-action distribution} (\cref{def:dependent-fixed-action-distribution}) over
$\actionSet(\initState, \finalState)$ with the dependency tree described above.
Then, we can derive the Rao-Blackwellized rates efficiently and use for training

\subsection{AtomGrow}
\label{app:atomgrow}

\textsc{AtomGrow} also grows from a single seed atom, but uses only single-atom
and single-bond actions; rings are produced by \texttt{AddBond} actions that close
cycles between previously inserted atoms.

\paragraph{States.} As in \textsc{FragmentGrow}, $\initState = (\{i_0\}, \emptyset)$
for a seed atom $i_0 \in V(\finalState)$, and every intermediate state is a
connected subgraph of $\finalState$ containing $i_0$.

\paragraph{Source distribution.}
We use the single-atom source $\noiseDist = \noiseDist^{\text{atom}}$.

\paragraph{Coupling.}
As in \textsc{FragmentGrow}: $\finalState$ is sampled from the training
distribution, $i_0$ uniformly from $V(\finalState)$, and
$\initState = (\{i_0\}, \emptyset)$.

\paragraph{Action space.} At state $s$,
\[
\mathcal{A}(s) \;=\;
\{\texttt{AddAtom}_i : i \in V(s)\}
\;\cup\;
\{\texttt{AddBond}_{ij} : i, j \in V(s),\; \{i,j\} \notin E(s)\}.
\]
$\texttt{AddAtom}_i$ adds a new atom $j$ together with the bond $\{i,j\}$;
$\texttt{AddBond}_{ij}$ inserts the bond $\{i,j\}$ between two existing atoms.

\paragraph{Latent decomposition.}
Different orderings of \texttt{AddAtom} / \texttt{AddBond} actions can yield the
same final molecule via genuinely different action sets — for example, the same
edge of a ring may be inserted either by an \texttt{AddAtom} that introduces its
second endpoint or by an \texttt{AddBond} that closes the cycle. To resolve this,
we introduce a latent
\[
z \;=\; T,
\]
a rooted spanning tree of $\finalState$ rooted at the seed atom $i_0$. We use the
factored form of \cref{sec:beyond-fixed} with $p(z \mid \initState, \finalState)$
a distribution over rooted spanning trees of $\finalState$ with root $i_0$. In practice we use for $z$ a random BFS tree of $\finalState$ rooted at $i_0$, which can be sampled efficiently and yields good results.

\paragraph{Required actions given $z$.}
Given $z = T$, the action set decomposes into two groups:
\begin{itemize}
\item For each non-root atom $v \in V(\finalState) \setminus \{i_0\}$, an action
$\texttt{AddAtom}_{\mathrm{parent}_T(v)}$ that produces $v$.
\item For each non-tree edge $\{i,j\} \in E(\finalState) \setminus E(T)$, an action
$\texttt{AddBond}_{ij}$.
\end{itemize}

\paragraph{Dependency structure.}
The \texttt{AddAtom} actions form a tree mirroring $T$:
$\texttt{AddAtom}_{\mathrm{parent}_T(v)}$ is a prerequisite of
$\texttt{AddAtom}_{\mathrm{parent}_T(w)}$ whenever $\mathrm{parent}_T(w) = v$.
Each $\texttt{AddBond}_{ij}$ has \emph{two} prerequisites — the
\texttt{AddAtom} actions producing $i$ and $j$ — and has no successors. 

\paragraph{Trajectory distribution.}
Conditional on $z$, we use the forest-plus-sink construction of
\cref{app:non-forest-approx} for the DAG above. It specifies the trajectory
distribution $\trajDist(\cdot \mid z, \initState, \finalState)$ together with
the conditional timing law and exact rates.

\paragraph{Multi-root variant.}
The construction extends naturally to sources $\initState$ consisting 
of $k$ disjoint seed atoms (used in our substructure-linking 
experiments, \cref{app:downstream}). The coupling samples $\finalState$ 
from the training distribution, runs BFS from $k$ random seeds 
$i_1, \dots, i_k \in V(\finalState)$ to partition $V(\finalState)$ 
into $k$ components, and sets $\initState = (\{i_1, \dots, i_k\}, 
\emptyset)$. The latent $z$ becomes a rooted spanning forest of 
$\finalState$ with one tree per seed; everything else is unchanged. 
The choice of $k$ (or a distribution over $k$) can be set flexibly. 
Empirically, we found that this matches single-root performance on 
standard generation while enabling zero-shot transfer to downstream 
tasks like substructure linking.

\subsection{MolEdit}
\label{app:moledit}

\textsc{MolEdit} is the most general setting: the source $\initState$ is itself a
(noisy) molecule, and trajectories edit it into $\finalState$ using both deletions
and additions of edges \emph{and} nodes.

\paragraph{States.} $\initState$ and $\finalState$ are molecular graphs related by
the atom-id correspondence: an atom $i \in V(\initState)$ is \emph{kept} if
$i \in V(\finalState)$ and \emph{deleted} otherwise. We write
\[
V_{\mathrm{kept}} = V(\initState) \cap V(\finalState), \quad
V_{\mathrm{del}} = V(\initState) \setminus V(\finalState), \quad
V_{\mathrm{new}} = V(\finalState) \setminus V(\initState),
\]
and assume $V_{\mathrm{kept}} \neq \emptyset$.

\paragraph{Source distribution.}
We use the marginal-graph source $\noiseDist = \noiseDist^{\text{marg}}$.

\paragraph{Coupling.}
The base coupling $\coupling$ is independent:
$\initState \sim \noiseDist^{\text{marg}}$ and
$\finalState \sim \dataDist$ are drawn independently. In practice, however, we
found it beneficial to maximize the edge overlap between $\initState$ and $\finalState$ to reduce the number of required edits.
Concretely, after sampling $(\initState, \finalState) \sim \noiseDist^{\text{marg}}
\otimes \dataDist$, we relabel the atoms of $\initState$ by greedily aligning
them to the atoms of $\finalState$ so as to maximize the size of the edge overlap
$|E(\initState) \cap E(\finalState)|$ (under the resulting atom-id
correspondence). This relabeling does not change $\noiseDist$ or $\dataDist$, but it concentrates $\coupling$ on
endpoint pairs with shorter trajectories.

\paragraph{Action space.} At state $s$,
\begin{align*}
\mathcal{A}(s) \;=\;\;
& \{\texttt{AddAtom}_i : i \in V(s)\}
\;\cup\; \{\texttt{AddBond}_{ij} : i,j \in V(s),\, \{i,j\} \notin E(s)\} \\
& \;\cup\; \{\texttt{DelBond}_{ij} : \{i,j\} \in E(s)\}
\;\cup\; \{\texttt{DelAtom}_i : i \in V(s),\, \deg_s(i) = 0\}.
\end{align*}
$\texttt{DelAtom}_i$ is only available once atom $i$ has no incident bonds.

\paragraph{Latent decomposition.}
We introduce a latent
\[
z \;=\; T,
\]
a rooted spanning forest of $\finalState$ whose roots are exactly
$V_{\mathrm{kept}}$. Each component contains one kept atom, and every new
atom $v \in V_{\mathrm{new}}$ has a unique parent $\mathrm{parent}_T(v)$
to which it attaches when inserted. Thus, the forest edges $E(T)$ are
exactly the bonds created by atom additions. Such a forest can be
constructed by a multi-source BFS initialized with all kept atoms.

\paragraph{Required actions given $z$.}
Given $z = T$, $\actionSet(\initState, \finalState \mid z)$ consists of:
\begin{itemize}
\item \textbf{Bond deletions.} $\texttt{DelBond}_{ij}$ for every
$\{i,j\} \in E(\initState)$ that is either incident to a deleted atom
($\{i,j\} \cap V_{\mathrm{del}} \neq \emptyset$) or is a kept-kept bond not
in $\finalState$ ($\{i,j\} \subseteq V_{\mathrm{kept}}$ and
$\{i,j\} \notin E(\finalState)$).
\item \textbf{Atom deletions.} $\texttt{DelAtom}_i$ for every $i \in V_{\mathrm{del}}$.
\item \textbf{Atom additions.} $\texttt{AddAtom}_{\mathrm{parent}_T(v)}$ for every
$v \in V_{\mathrm{new}}$.
\item \textbf{Bond additions.} $\texttt{AddBond}_{ij}$ for every
\[
\{i,j\} \in
\bigl(E(\finalState) \setminus E(\initState)\bigr)
\setminus E(T).
\]
These are all missing target bonds except those created by atom
additions, including missing bonds between kept atoms and bonds
connecting different forest components.
\end{itemize}

\paragraph{Dependency structure.}
The dependencies are:
\begin{itemize}
\item For each $i \in V_{\mathrm{del}}$, every $\texttt{DelBond}_{ij}$ incident to
$i$ is a prerequisite of $\texttt{DelAtom}_i$.
\item For each $v \in V_{\mathrm{new}}$, the action inserting $v$
depends on the action inserting $\mathrm{parent}_T(v)$ if that
parent is also new; otherwise, it has no prerequisites.
\item Each required $\texttt{AddBond}_{ij}$ depends on the
atom-insertion action for each endpoint in $V_{\mathrm{new}}$.
Endpoints in $V_{\mathrm{kept}}$ introduce no prerequisites.
\end{itemize}
The atom-insertion dependencies form a forest rooted at insertions
whose parents are kept atoms. Bond additions and atom deletions
have no successors, so the resulting action DAG retains the
forest-plus-sink structure.

\paragraph{Trajectory distribution.}
Conditional on $z$, we use the forest-plus-sink construction of
\cref{app:non-forest-approx} for the DAG above, which specifies the
trajectory distribution, conditional timing law, and exact rates.

\subsection{GrowHierarchy}
\label{app:growhierarchy}
 
To illustrate that trajectory matching applies beyond molecular graph generation, we
include a non-molecular setting: generating rooted trees by repeatedly attaching
children to existing nodes. This is the simplest non-trivial instance of a
forest-structured dependent fixed-action distribution --- a single action type,
no latent, and a closed-form Rao-Blackwellized rate from \cref{prop:dependent-fixed-action-forest-rates}.
 
\paragraph{States.} A state is a rooted tree $s = (V(s), E(s), r)$ with $V(s)$ a set of
node ids, $E(s) \subseteq V(s) \times V(s)$ a set of directed parent-to-child
edges, and $r \in V(s)$ the root. The intermediate states along a trajectory
are subtrees of $\finalState$ containing its root.
 
\paragraph{Source distribution.}
$\noiseDist$ samples a single-node tree $(\{r\}, \emptyset, r)$. If the training
distribution is over rooted trees, $r$ is the designated root of the sampled
$\finalState$; otherwise $r$ is sampled uniformly from $V(\finalState)$.
 
\paragraph{Coupling.}
$\finalState$ is sampled from the training distribution and $\initState$ is
the single-node tree at its root.
 
\paragraph{Action space.} At a state $s$,
\[
\mathcal{A}(s) \;=\; \{\texttt{AddChild}_i \;:\; i \in V(s)\},
\]
where $\texttt{AddChild}_i$ adds a new node $j \notin V(s)$ together with the
directed edge $(i, j)$.
 
\paragraph{Required actions.}
The set of actions transforming $\initState$ into $\finalState$ is
\[
\actionSet(\initState, \finalState) \;=\;
\bigl\{\,\texttt{AddChild}_{\mathrm{parent}_{\finalState}(v)} \;:\;
v \in V(\finalState) \setminus \{r\}\,\bigr\},
\]
i.e.\ one action per non-root node of $\finalState$.
 
\paragraph{Dependency structure.}
The dependency DAG is exactly $\finalState$: the action producing $v$ is a
prerequisite of the action producing any child of $v$. This DAG is a tree (in
particular a forest), so \cref{prop:dependent-fixed-action-forest-rates} applies.
 
\paragraph{Trajectory distribution.}
$\trajDist(\cdot \mid \initState, \finalState)$ is the dependent fixed-action
distribution over $\actionSet(\initState, \finalState)$, uniform over the
topological orderings of the tree.

\subsection{Handling Node and Edge Features}
\label{app:features}
 
The instantiations above abstract away from node and edge features (atom types,
bond orders, charges, etc.). Two natural strategies extend them to feature-labeled
graphs.
 
\paragraph{Strategy A: joint prediction at insertion.}
Insertion actions are parameterized by features: $\texttt{AddAtom}_i(v)$ inserts
a new atom with feature $v$ in a single action, and analogously for
$\texttt{AddBond}_{ij}(b)$ and $\texttt{AddFragment}$. The action vocabulary
grows by a factor of the feature alphabet, but no separate feature-edit actions
are needed.
 
\paragraph{Strategy B: separate feature-edit actions.}
Insertion actions produce nodes and bonds with a random feature drawn from a
fixed distribution; features are then modified by dedicated
$\texttt{ChangeAtomFeature}_i(v)$ and $\texttt{ChangeBondFeature}_{ij}(b)$ actions.
For settings such as \textsc{BondEdit} that already work with a fixed atom set,
$\texttt{ChangeAtomFeature}$ and $\texttt{ChangeBondFeature}$ become additional
edits on the existing structure.
 
\paragraph{Tradeoffs and our default.}
Strategy~A reduce the number of actions needed, but lock features in at insertion time. Strategy~B keeps features editable throughout the trajectory, which lets the model correct earlier feature choices during sampling, and it composes cleanly with all instantiations above without enlarging fragment libraries by feature combinations. 
The best choice depend on the task and setting, but we found strategy~B to perform best in most experiments.

The dependency structures can be easily adapted to accomodate these additional feature-wise edit actions.

\subsection{Exact conditional paths for forest-plus-sink dependencies}
\label{app:non-forest-approx}

The dependency structures of \textsc{AtomGrow} and \textsc{MolEdit} motivate a tractable construction beyond forests. Required \texttt{AddBond} actions may require both endpoint atoms to be inserted, and \textsc{MolEdit}'s \texttt{DelAtom} actions require their incident bonds to be deleted. These multi-parent actions are sinks: they have prerequisites but no successors. We construct a conditional path by evolving the forest independently and scheduling each sink after its prerequisites.

\paragraph{Dependency structure.}
Fix endpoints $(\initState,\finalState)$ and a latent decomposition $z$, and write $G:=\actionSet(\initState,\finalState\mid z)$ for the resulting action dependency DAG. Every topological ordering of $G$ defines a valid action sequence between these endpoints. We consider DAGs consisting of a forest $F$ and a set $Z$ of multi-parent sinks whose prerequisites lie in $F$. This partition is fixed throughout the construction.

\paragraph{Sampling the conditional path.}
We use the linear time schedule and sample firing times as follows:
\begin{enumerate}
    \item Sample a uniform topological ordering of $F$ and assign its actions the ordered values of $|F|$ independent $\mathrm{Uniform}(0,1)$ draws. This is the existing forest construction underlying \cref{prop:dependent-fixed-action-forest-rates}; denote the resulting firing times by $t_a$.
    \item For each sink $b\in Z$, let $t_b^{\min}=\max_{a\in\operatorname{prereq}(b)}t_a$ and sample
    \[
        t_b=t_b^{\min}+(1-t_b^{\min})U_b, \qquad U_b\sim\mathrm{Uniform}(0,1),
    \]
    using draws $U_b$ that are independent across sinks and independent of the forest draws.
    \item Obtain $S_t$ by applying all actions with firing time at most $t$ to $\initState$, in increasing firing-time order.
\end{enumerate}
This procedure enables efficient sampling of intermediate states. When $Z=\varnothing$, it recovers exactly the forest construction of \cref{prop:dependent-fixed-action-forest-rates}.

\paragraph{Exact conditional rates.}
We now show that the constructed path admits tractable, closed-form conditional rates. Fixing the endpoints and latent decomposition $z$, let $G_s:=\actionSet(s,\finalState\mid z)$ denote the remaining action dependency DAG at state $s$. As in the preceding fixed-action constructions, we assume that $s$ determines the remaining actions under this conditioning. Let $F_s$ be the forest induced by the remaining forest actions, and let $\operatorname{desc}_{F_s}(a)$ denote the size of the subtree rooted at $a$, including $a$. The rates are given by the following proposition.

\begin{proposition}[Exact forest-plus-sink path]
\label{prop:forest-sink-path}
The sampling construction defines a conditional probability path starting from $\initState$ and ending at $\finalState$, generated, for $0\leq t<1$, by
\begin{equation}
\label{eq:forest-sink-rate}
u_t^{\mathrm{alt}}(s,a\mid\initState,\finalState,z)
=\frac{1}{1-t}
\begin{cases}
    \operatorname{desc}_{F_s}(a), & a\in\operatorname{roots}(F_s),\\
    1, & a\in Z\cap\operatorname{roots}(G_s),\\
    0, & \text{otherwise}.
\end{cases}
\end{equation}
\end{proposition}

\begin{proof}
The forest process has the rates of \cref{prop:dependent-fixed-action-forest-rates}: every available forest action $a$ fires at rate $\operatorname{desc}_{F_s}(a)/(1-t)$. Sinks do not affect this process. Their firing histories depend only on past forest events and independent sink draws, so observing them does not change the forest's future rates given its current completed actions.

A sink $b$ cannot fire before $t_b^{\min}$. Conditional on its prerequisites' firing times, $t_b$ is uniform on $[t_b^{\min},1]$. Once enabled and while unexecuted, its hazard is
\[
    \frac{1/(1-t_b^{\min})}{(1-t)/(1-t_b^{\min})}=\frac{1}{1-t}.
\]
This expression is independent of $t_b^{\min}$, so the same rate holds after averaging over past prerequisite times. Thus all transition intensities depend only on the completed action set and $t$, yielding the stated CTMC and rates.

Forest actions respect their topological ordering and fire before $1$ almost surely. Each sink fires strictly after its latest prerequisite and before $1$ almost surely. The full firing order is therefore a topological ordering of $G$ and executes every required action. By assumption it transforms $\initState$ into $\finalState$.
\end{proof}

\paragraph{Connection to trajectory matching.}
Trajectory matching permits any valid conditional path between the prescribed endpoints. Since \cref{eq:forest-sink-rate} gives exact rates for such a path, the conditional-to-marginal argument of \cref{thm:trajectory-marginal} applies with the endpoints and $z$ as conditioning variables. Under realizability, matching these rates recovers the corresponding marginal rate and transports the source distribution to the target distribution.

\paragraph{Comparison with the uniform-topological path.}
The difference between the two paths lies in the back-coupling of sinks to their prerequisites. In the uniform-topological construction, all actions share a joint timing law: a sink must occur after its prerequisites, so its presence can accelerate those prerequisites, leaving time for the sink to execute. Our construction removes this back-coupling by sampling the forest independently and then scheduling sinks after their prerequisites. Once a sink is enabled, its rate is $1/(1-t)$ in both constructions.

For an available forest action, the relative rate difference is controlled by the number of still-disabled sinks with prerequisites in its residual subtree, relative to the size of that subtree. If there are no such sinks, the rates agree exactly. Both constructions respect the same dependencies and reach the same endpoints. Their difference concerns the intermediate probability path, which is a modeling choice in trajectory matching.

    \section{Connection to existing methods}
\label{app:connections}

\begin{figure}[tbp]
    \centering
    \includegraphics[width=1\textwidth]{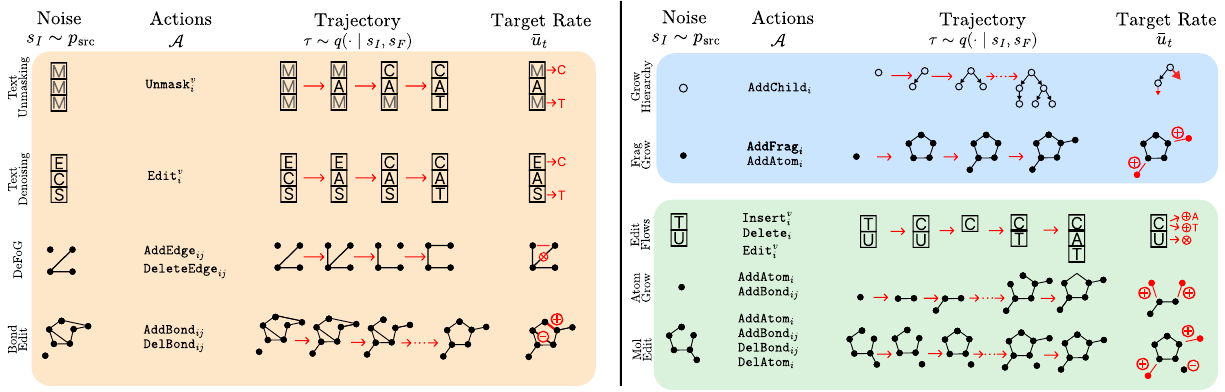}
    \caption{Examples of different existing methods that can be seen as instances of Trajectory Matching (text unmasking, text denoising, DeFog GraphGeneration \citealp{defog} and EditFlows \citealp{edit_flows}) and our new instantiations of Trajectory Matching (AtomGrow, FragGrow, MolEdit, BondEdit, and GrowHierarchy). for each instantiation, we show an initial state $\initState \sim \noiseDist$, the action space $\mathcal{A}$,a sample trajectory $\tau$ from $\initState$ to a final state $\finalState \sim \dataDist$ and the Rao-Blackwellized target rate $\rbrate(s, a \mid \initState, \finalState)$ of an intermediate state $s$.}
    \label{fig:connections-overview}
\end{figure}

We know show that Trajectory Matching is a general framework that encompasses many existing methods as special cases. In \cref{fig:connections-overview}, we give an overview of how different existing methods can be seen as instances of Trajectory Matching, and how our new instantiations of Trajectory Matching fit into this picture. We discuss these connections in more detail in the following sections.

Any factorized flow matching method, including text unmasking \citep{discrete_diffusion}, text denoising and fixed size graph generation \citep{defog} can be directly seen as instance of Trajectory Matching with an independent fixed-action distribution.   

We will next show that we can also see Edit Flows \citep{edit_flows} as an instance of Trajectory Matching.

    \subsection{Edit Flows as Trajectory Matching}
\label{app:editflows}

We finally relate our framework to Edit Flows. Edit Flows define a CTMC directly
on variable-length sequences using edit operations: insertions, deletions, and
substitutions. The main difficulty is that, for two sequences, there is in
general no unique sequence of edits transforming one into the other. Edit Flows
handle this by introducing an auxiliary alignment. Given a source sequence and a
target sequence, an alignment produces two padded sequences $(z_0,z_1)$ of the
same length, where the blank symbol $\epsilon$ records insertions and deletions.
Each aligned coordinate $j$ with $z_0^j\neq z_1^j$ then specifies one elementary
edit: if $z_0^j=\epsilon$ the edit inserts $z_1^j$, if $z_1^j=\epsilon$ the edit
deletes $z_0^j$, and otherwise the edit substitutes $z_0^j$ by $z_1^j$.

This fits naturally into trajectory matching by viewing the alignment as an
additional latent variable. In the notation of Sec.~4, we take
\[
z := (z_0,z_1),
\]
and first sample $z$ from the alignment distribution conditional on the endpoint
pair. Conditional on $z$, define the set of indices that need to change as
\[
E(z) := \{j : z_0^j \neq z_1^j\}.
\]
A trajectory is then a permutation of this index set,
\[
\trajectory = (j_1,\dots,j_m),
\qquad
m := |E(z)|,
\]
and represents the order in which the alignment-induced edits are performed.
For a subset $S\subseteq E(z)$ of edits that have already been applied, let
$s_S$ be the current observable sequence obtained by applying the edits in $S$
and removing blanks. For each remaining index $j\in E(z)\setminus S$, let
\[
A_z(j,S)
\]
denote the edit action on the current sequence $s_S$ corresponding to applying
the aligned edit at coordinate $j$. We define
\[
\trajDist(\trajectory\mid z)
\]
to be uniform over all permutations of $E(z)$. Equivalently, conditional on
$(z,S)$, the next index in the trajectory is uniform over the remaining set
$E(z)\setminus S$.

It is important that this is not, in general, a fixed-action distribution in the
sense of independent component-wise edits. Although the latent index set
$E(z)$ is fixed once the alignment is fixed, the concrete edit action induced by
an index can depend on the set $S$ of edits that have already been applied,
because insertions and deletions change positions in the observable sequence.
For example, starting from the aligned source $\epsilon\,\epsilon$ and
target $A\,B$, the ordering that inserts $A$ before $B$ corresponds to
\[
\mathrm{insert}_0(A) \;\to\; \mathrm{insert}_1(B),
\]
whereas the other ordering is
\[
\mathrm{insert}_0(B) \;\to\; \mathrm{insert}_0(A).
\]
Thus the same latent coordinate system defines the remaining edits, but the
actual observable edit action $A_z(j,S)$ is state-dependent. We next
show that, despite this dependence, marginalizing over the uniform ordering
recovers exactly the Edit Flow probability path and marginal rate.

\begin{proposition}[Edit Flows are uniform-order trajectory matching]
\label{prop:ef-equiv}
For a fixed alignment distribution over $(z_0, z_1)$, let $\rate^{\mathrm{EF}}$
denote the Edit Flow rate with a linear time schedule, and let $\rate^{\mathrm{TM}}$ denote the marginal
trajectory-matching rate obtained by sampling, conditional on $(z_0, z_1)$, a
uniform ordering of all alignment-compatible edits. Then Edit Flows and this
trajectory-matching construction induce the same marginal probability path
and the same marginal rate:
\[
    p_t^{\mathrm{EF}}(s) = p_t^{\mathrm{TM}}(s),
    \qquad
    \rate^{\mathrm{EF}}(s, a) = \rate^{\mathrm{TM}}(s, a).
\]
\end{proposition}

\begin{proof}
Fix an alignment $(z_0, z_1)$ and abbreviate
\[
    E := \{j : z_0^j \neq z_1^j\},
    \qquad
    m := |E|.
\]
For $S \subseteq E$, let $s_S$ be the observable state obtained by applying
the edits in $S$ and removing blanks. We do not assume that the map
$(z_0, z_1, S) \mapsto s_S$ is injective.

\emph{Probability paths.} In trajectory matching, a trajectory is a uniform
ordering $\trajectory = (j_1, \dots, j_m)$ of the edits in $E$, and under the
linear schedule the firing times are i.i.d.\ uniform on $[0, 1]$ and
independent of $\trajectory$. In particular, the number of edits completed by
time $t$ satisfies $K_t \sim \mathrm{Binomial}(m, t)$ independently of
$\trajectory$. Hence, for any fixed $S \subseteq E$ with $|S| = k$,
\[
    \begin{aligned}
        \mathbb{P}^{\mathrm{TM}}(S_t = S \mid z_0, z_1)
        &=
        \mathbb{P}(K_t = k)\,
        \mathbb{P}(\{j_1, \dots, j_k\} = S)
        \\
        &=
        \binom{m}{k} t^k (1-t)^{m-k}
        \cdot
        \frac{1}{\binom{m}{k}}
        \\
        &=
        t^k (1-t)^{m-k}.
    \end{aligned}
\]
In Edit Flows, each aligned edit is independently applied with probability
$t$, so
\[
    \mathbb{P}^{\mathrm{EF}}(S_t = S \mid z_0, z_1)
    =
    t^{|S|} (1-t)^{m-|S|}.
\]
Therefore the two constructions induce the same distribution over
$(z_0, z_1, S_t)$. Since the observable state is the deterministic pushforward
$X_t = s_{S_t}$ in both,
\[
    p_t^{\mathrm{TM}}(s \mid z_0, z_1)
    =
    p_t^{\mathrm{EF}}(s \mid z_0, z_1),
    \qquad
    p_t^{\mathrm{TM}}(s)
    =
    p_t^{\mathrm{EF}}(s).
\]

\emph{Marginal rates.} Fix $(z_0, z_1, S)$ and write $R := E \setminus S$.
For $j \in R$, let $A(j, s_S)$ denote the $\states$-edit action at $s_S$
corresponding to applying the aligned edit $j$. The Edit Flow conditional
rate is (in our action-based formulation)
\[
    \rate^{\mathrm{EF}}(s_S, a \mid S, z_0, z_1)
    =
    \sum_{j \in R}
    \frac{1}{1-t}\,
    \indicator{a = A(j, s_S)}.
\]
The trajectory-matching rate, conditioned on an ordering $\trajectory$
compatible with $S$ at $s_S$ (i.e.\ $\{j_1, \dots, j_k\} = S$ where $k = |S|$),
places all rate on the next edit:
\[
    \rate^{\mathrm{TM}}(s_S, a \mid \trajectory, z_0, z_1)
    =
    \frac{|R|}{1-t}\,
    \indicator{a = A(j_{k+1}, s_S)}.
\]
Conditioning a uniform ordering of $E$ on the event that the prefix-set is
$S$ leaves the suffix uniform over orderings of $R$, so
$\mathbb{P}(j_{k+1} = j \mid S, z_0, z_1) = 1/|R|$ for every $j \in R$.
Averaging over the trajectory posterior gives
\[
    \begin{aligned}
        \mathbb{E}_{\trajectory \mid S, z_0, z_1}
        \left[
            \rate^{\mathrm{TM}}(s_S, a \mid \trajectory, z_0, z_1)
        \right]
        &=
        \frac{|R|}{1-t}
        \sum_{j \in R}
        \frac{1}{|R|}\,
        \indicator{a = A(j, s_S)}
        \\
        &=
        \rate^{\mathrm{EF}}(s_S, a \mid S, z_0, z_1).
    \end{aligned}
\]
This equality holds pointwise for every latent triple $(z_0, z_1, S)$.

To obtain the marginal rate at an observable state $s$, we average over the
posterior of latent triples compatible with $X_t = s$. By definition,
\[
    \rate^{\mathrm{EF}}(s, a)
    =
    \mathbb{E}^{\mathrm{EF}}_{z_0, z_1, S \mid X_t = s}
    \left[
        \rate^{\mathrm{EF}}(s, a \mid S, z_0, z_1)
    \right],
\]
\[
    \rate^{\mathrm{TM}}(s, a)
    =
    \mathbb{E}^{\mathrm{TM}}_{z_0, z_1, S, \trajectory \mid X_t = s}
    \left[
        \rate^{\mathrm{TM}}(s, a \mid \trajectory, z_0, z_1)
    \right].
\]
Applying the tower property,
\[
    \rate^{\mathrm{TM}}(s, a)
    =
    \mathbb{E}^{\mathrm{TM}}_{z_0, z_1, S \mid X_t = s}
    \left[
        \mathbb{E}^{\mathrm{TM}}_{\trajectory \mid S, z_0, z_1, X_t = s}
        \left[
            \rate^{\mathrm{TM}}(s, a \mid \trajectory, z_0, z_1)
        \right]
    \right].
\]
Given $(z_0, z_1, S)$, the suffix ordering $\sigma_{k+1:m}$ is independent of
both $K_t$ and the prefix-set under uniform $\sigma$ and the linear schedule;
the event $X_t = s$ is a function of $(z_0, z_1, S)$ alone (it just selects
latent triples with $s_S = s$), so conditioning further on it does not change
the suffix law. Hence
\[
    \mathbb{E}^{\mathrm{TM}}_{\trajectory \mid S, z_0, z_1, X_t = s}
    \left[
        \rate^{\mathrm{TM}}(s, a \mid \trajectory, z_0, z_1)
    \right]
    =
    \rate^{\mathrm{EF}}(s, a \mid S, z_0, z_1).
\]
By the path equivalence already shown, the posterior of $(z_0, z_1, S)$ given
$X_t = s$ is the same under TM and EF, so
\[
    \rate^{\mathrm{TM}}(s, a)
    =
    \mathbb{E}^{\mathrm{EF}}_{z_0, z_1, S \mid X_t = s}
    \left[
        \rate^{\mathrm{EF}}(s, a \mid S, z_0, z_1)
    \right]
    =
    \rate^{\mathrm{EF}}(s, a).
    \qedhere
\]
\end{proof}

The equivalence extends to any shared time schedule $\kappa$ satisfying the
conditions in \cref{app:kappa-schedules}, by applying the same time
reparametrization to both constructions: $\tilde p_t = p_{\kappa_t}$ and
$\tilde u_t = \dot\kappa_t u_{\kappa_t}$.

Importantly, this shows that the Edit Flow loss and our trajectory matching loss share the same optimum and gradients. 

    \section{Additional Experimental Results}
\label{app:experimental_details}
 
This appendix first describes the empirical setup. Additionally, it complements the empirical results of
\cref{sec:results} with: (i) extended baseline comparisons on
the Tree (\cref{app:tree-extended}) and molecular graph
benchmarks (\cref{app:mol-extended}); (ii) experiments on the
validity--compute trade-off as a function of the Euler-integration
budget; (iii) an ablation of the trajectory-conditioned versus
Rao-Blackwellized training target predicted by
\cref{prop:rao-blackwell-variance} (\cref{app:rb-ablation}); and
(iv) an ablation of the trajectory distribution (\cref{app:q-ablation}).
Throughout, we use the same training hyperparameters, model backbone,
and sampling procedure as described in \cref{app:setup}; only the
quantity being varied changes.
 
\subsection{Experimental Setup}
\label{app:setup}
 
All instantiations of trajectory matching presented in this work
share a common training and sampling pipeline. The differences
between instantiations are isolated to the action space, source
distribution, and trajectory distribution as described in \cref{app:instantiations};
the architecture, loss function, datasets, and hyperparameters
described in this section are held fixed across instantiations.
 
\paragraph{Model architecture.}
We use a graph-transformer backbone proposed by DiGress \citep{digress} to predict a rate for every
node and every node pair in the current state $s$, processed
jointly by a stack of attention layers with node and pair
representations updated in tandem. We use the additional positional encoding by
DeFoG \citep{defog}. 
We adopt the graph transformer unchanged in structure, with two practical adaptations:
(i) we additionally include the current node count $|V(s)|$ as an
input feature, encoded with a radial-basis-function (RBF) embedding,
because our variable-size instantiations operate over states of
varying size and the appropriate rate depends on how much of the
generation has already been completed; and (ii) for the Tree
benchmark we use a substantially smaller backbone than for the
molecular datasets, while all molecular datasets share the same
backbone size as in DeFoG. The architecture is
permutation-equivariant on nodes and node pairs by construction, so
the predicted rates are also permutation-equivariant.
 
\paragraph{Loss decomposition.}
The trajectory-conditioned target $\rate(\cdot \mid \trajectory)$ is sparse:
only the next action in $\trajectory$ has positive rate at the current
state $s$, while the Rao-Blackwellized target
$\rbrate(\cdot \mid \initState, \finalState)$ has positive rate only on the small
set of remaining actions $\actionSet(s, \finalState)$. To predict such a sparse rate vector, we follow OneFlow \citep{oneflow} and decompose the
parametric rate into a positivity probability and a (log-)rate
magnitude:
\begin{equation}
  \rate^{\theta}(s, a) \;=\; h^{\theta}_t(s, a) \cdot \exp\bigl(g^{\theta}_t(s, a)\bigr),
  \label{eq:rate-decomposition}
\end{equation}
where $h^{\theta}_t(s, a) := \sigma(f^{\theta}_t(s, a)) \in [0, 1]$
is produced by a sigmoid head and represents the probability that
$a$ has positive rate at $(s, t)$, and $g^{\theta}_t(s, a) \in
\mathbb{R}$ is a raw log-rate head. Both heads share the same
graph-transformer backbone.
 
We decompose the loss correspondingly. Fix $(t,s,a)$, write
$m^\theta=\exp(g_t^\theta(s,a))$, and let $R$ denote the training rate
target, either trajectory-conditioned or Rao--Blackwellized. With
$Y=\indicator{R>0}$, the per-action objective is
\begin{equation}
\ell_{\mathrm{dec}}(\theta)
=\mathbb{E}\!\left[\operatorname{BCE}(Y,h^\theta)+Y D(R,m^\theta)
\mid t,s,a\right].
\label{eq:decomposed-loss}
\end{equation}
Importantly, this practical implementation choice retains all central consistency guarantees of trajectory matching. In particular the population optimum of \cref{eq:decomposed-loss} recovers the same marginal probability path as the original Bregman rate-matching objective, as formalized below:

\begin{proposition}[Consistency of the decomposed objective]
\label{prop:decomposed-consistency}
Define $h_*=\Pr(R>0\mid t,s,a)$ and, when $h_*>0$,
$m_*=\mathbb{E}[R\mid t,s,a,R>0]$. Then
\[
\nabla_\theta\ell_{\mathrm{dec}}(\theta)
=\nabla_\theta\!\left[\operatorname{BCE}(h_*,h^\theta)
+h_*D(m_*,m^\theta)\right].
\]
Under head-wise realizability, the population optimum recovers
$\rate^\theta(s,a)=h_*m_*=\mathbb{E}[R\mid t,s,a]=\rate(s,a)$
on the support of the training distribution, and hence generates the
same marginal probability path. When $h_*=0$, the magnitude target
may be chosen arbitrarily and the optimal rate is zero.
\end{proposition}
\begin{proof}
BCE is affine in its target, and the model gradient of a Bregman
divergence is affine in its first argument. Taking conditional
expectations therefore gives the gradient identity. The corresponding
marginal objective is minimized at $h^\theta=h_*$ and
$m^\theta=m_*$ on positive support. Their product equals the marginal
rate by conditional expectation, for either choice of $R$; the path
claim follows from \cref{thm:trajectory-marginal}.
\end{proof}

\Cref{tab:loss-decomposition-ablation} compares the decomposed objective
with direct Bregman rate matching on \textsc{GrowHierarchy}. Both
objectives achieve similar performance in this experiment.
\begin{table}[htbp]
\centering
\small
\caption{Objective ablation on Tree with \textsc{GrowHierarchy}.
Entries report means and standard deviations over sampling seeds
from a single trained model per objective.}
\label{tab:loss-decomposition-ablation}
\begin{tabular}{lcccc}
\toprule
Objective & Degree $\downarrow$ & Spectral $\downarrow$ & Wavelet $\downarrow$ & Validity (\%) $\uparrow$ \\
\midrule
Decomposed & $0.0006\pm0.0002$ & $0.0072\pm0.0006$ & $0.0033\pm0.0002$ & $100.0\pm0.0$ \\
Direct Bregman & $0.0008\pm0.0002$ & $0.0085\pm0.0009$ & $0.0034\pm0.0004$ & $100.0\pm0.0$ \\
\bottomrule
\end{tabular}
\end{table}
 
\paragraph{Datasets.}
We use the standard dataset splits as in DeFoG \citep{defog} to
ensure direct comparability with that line of work; we summarize
each dataset briefly here and refer to DeFoG App. F.2 for full
preprocessing details.
FragGrow uses the representable subset of each training split, as detailed in
\cref{app:fragmentgrow}. We explicitly do not filter the validation or test
splits, so evaluation uses the same complete standard reference splits as
the other models.
For the molecular benchmarks, ZINC is distributed under a custom
license and is free to use for everyone, while MOSES is released
under the MIT license.
 
\textit{Tree.} The synthetic random-tree dataset introduced in \citet{tree}, consisting of trees with $64$
nodes. Used in \cref{sec:exp-hierarchy,app:tree-extended} to
evaluate hierarchy generation.
Validity here means that the generated graph is a connected acyclic
graph; the Ratio metric captures distributional fit through normalized distribution-style graph statistics. We note that the Ratio metric is very unstable as it divides through a very small number.
 
\textit{QM9 (without hydrogens).} QM9~\citep{qm9} contains small organic
molecules with up to $9$ heavy atoms (C, N, O, F), with hydrogens
removed. 
 
\textit{ZINC250K.} A subset of approximately $250$K drug-like
molecules from the ZINC database. Tests scaling to larger and more
chemically diverse molecules and is the substrate for our
downstream substructure extension and linking experiments
(\cref{app:downstream}).
 
\textit{MOSES.} A standardized benchmark for de novo molecule
generation comprising approximately $1.9$M drug-like molecules,
with the standard MOSES train/test split as used by DeFoG. Includes explicit
metrics for distributional fit such as FCD, SNN similarity, and
scaffold similarity, and is the most challenging of our molecular
benchmarks for both quality and coverage.
 
\paragraph{Training.}
We implement all experiments in PyTorch \citep{pytorch}.
All experiments use a single NVIDIA A100 GPU per training run.
We train with Adam at learning rate $2 \cdot 10^{-4}$, no weight
decay, and a fixed (no schedule) learning rate. Batch sizes are
$64$ for QM9 and ZINC and $32$ for MOSES (smaller due to memory).
Training durations are $5\,000$ steps for Tree, $200\,000$ steps for
QM9 and ZINC, and $500\,000$ steps for MOSES. We use exponential
moving average (EMA) on model weights with decay $0.9999$, started
after step $1\,000$ and updated every step. Gradients are clipped
at $L_2$-norm $1.0$. The training-time distribution over $t$ is
uniform on $[0, 1]$; the conditional-flow time schedule
$\markovDist$ (\cref{app:kappa-schedules}) is either linear or the polydec schedule of
DeFoG, with polydec yielding better results on most datasets.
Checkpoints are selected by lowest validation FCD for the molecular
datasets and lowest ratio on tree.

\paragraph{Computational cost.}
\label{app:runtime}
We compare the computational cost of trajectory matching with the
factorized DFM baseline DeFoG on QM9, with measurements averaged after warmup.
\Cref{tab:runtime-preprocessing} reports both raw CPU time for trajectory
sampling and target construction and the loader wait observed by the
training loop. These preprocessing operations run in parallel CPU
dataloader workers, overlapping with GPU training
(\cref{app:training_sampling}). Despite differences in raw construction
time, the observed loader wait remains below 1 ms per training step,
making preprocessing essentially free in exposed wall-clock cost in
this setup. \textsc{FragGrow}'s CPU preprocessing cost could be reduced
further by precomputing and caching the training molecules' fragment
decompositions (\cref{app:fragmentgrow}), although this need not yield
a noticeable end-to-end speedup when preprocessing is already overlapped.

\begin{table}[htbp]
\centering
\small
\caption{CPU preprocessing time and observed loader wait on QM9.
Preprocessing overlaps with GPU training; loader wait is measured per
training step and is not the total overhead relative to DeFoG.}
\label{tab:runtime-preprocessing}
\begin{tabular}{lrr}
\toprule
Method & \shortstack{Trajectory sampling +\\target construction (ms)} & \shortstack{\textbf{Observed loader}\\\textbf{wait (ms)}} \\
\midrule
\textsc{AtomGrow} & 4.56 & 0.75 \\
\textsc{FragGrow} & 188.83 & 0.94 \\
\textsc{MolEdit} & 30.55 & 0.91 \\
\bottomrule
\end{tabular}
\end{table}

\Cref{tab:runtime-throughput} compares end-to-end training and sampling
throughput. Relative to DeFoG, training throughput is approximately
10--14\% lower and sampling throughput is 22--25\% lower at the same
500-step sampling budget.

\begin{table}[htbp]
\centering
\small
\caption{End-to-end throughput on QM9 using one A100 GPU, averaged
after warmup. Sampling uses 500 integration steps for every method.}
\label{tab:runtime-throughput}
\begin{tabular}{lrr}
\toprule
Method & \shortstack{Training throughput\\(steps/s)} & \shortstack{Sampling throughput\\(samples/s)} \\
\midrule
DeFoG & 13.79 & 10.22 \\
\textsc{AtomGrow} & 11.89 & 7.96 \\
\textsc{FragGrow} & 12.01 & 7.63 \\
\textsc{MolEdit} & 12.45 & 7.69 \\
\bottomrule
\end{tabular}
\end{table}
 
The inductive bias of structured actions can offset these higher per-step
training and sampling costs by enabling smaller models, fewer training
steps, or fewer sampling steps. On Tree, training takes less than 1 h,
compared with approximately 8 h reported for DeFoG \citep{defog}.
We attribute this advantage in part to the tree-preserving inductive
bias, which allows us to use a substantially smaller model and train
for fewer steps (5\,000 in our setup). Separately, our sampling-budget
experiments (\cref{fig:tree-step-sweep,fig:zinc-step-sweep}) show that
structured actions can preserve generation quality with fewer integration
steps, potentially reducing total sampling time.
On the molecular datasets, approximate per-instantiation training times
are $\sim\!8$ h on QM9, $\sim\!24$ h on ZINC, and $\sim\!48$ h on MOSES.

\paragraph{Sampling.}
We use the Euler integration scheme of \cref{alg:trajectory_matching_sampling} with
$N = 1000$ steps for the Tree benchmark and $N = 500$ steps for
QM9, ZINC, and MOSES, matching DeFoG's defaults. We do not employ
any of DeFoG's sampling-time optimizations (target guidance,
stochasticity, time distortion). For evaluation we generate $40$
samples per run on Tree, $10\,000$ samples on QM9 and ZINC, and
$25\,000$ samples on MOSES. Step-budget sweeps in
\cref{app:tree-extended,app:mol-extended} vary
$N$ around these defaults to study the quality--compute trade-off.
 
\paragraph{Reported metrics and variability.}
Unless otherwise stated, results are means and standard deviations over
$5$ independent sampling seeds from a single trained model; for MOSES,
we follow the evaluation setup of DeFoG and DiGress
\citep{defog,digress} and report a single run of $25\,000$ samples.
To assess sampling variability, we additionally evaluated \textsc{AtomGrow}
and \textsc{FragGrow} using three independent sampling seeds per fixed
trained model, obtaining validity standard deviations of $0.1$ and $0.02$
percentage points and FCD standard deviations of $0.008$ and $0.006$,
respectively.
These are small relative to \textsc{FragGrow}'s validity improvement and
\textsc{AtomGrow}'s FCD improvement over DeFoG, supporting that these gains
are unlikely to be explained by sampling randomness.

\subsection{Hierarchy Generation: Extended Results}
\label{app:tree-extended}
 
\paragraph{Full baseline comparison.}
The expanded tree table
extends \cref{tab:tree-comparison} from the main
text with additional baselines. The picture is consistent
with the main-text findings: \textsc{GrowHierarchy} achieves $100\%$
V.U.N., with validity guaranteed by construction, while methods that do not encode tree
structure in their action space either underperform on V.U.N. or
collapse entirely. The only exception is HSpectre, which combines a diffusion and an autoregressive approach but cannot guarantee validity by construction.

While our main comparisons focus on diffusion and flow-matching methods
to isolate the effect of structured actions, Table~\ref{tab:tree_full}
also includes four purely autoregressive baselines: GraphRNN, GRAN,
BiGG, and GraphGen. \textsc{GrowHierarchy} achieves $100\%$ V.U.N.,
compared with the best autoregressive result of $95\%$, and a Ratio
of $2.5$, compared with $5.2$.

\newcommand{\pmval}[1]{{\scriptsize\,$\pm$#1}}

\begin{table}[ht]
\centering
\setlength{\tabcolsep}{3pt}
\renewcommand{\arraystretch}{1.15}
\caption{Graph generation results on the Tree dataset.}
\resizebox{\textwidth}{!}{%
\begin{tabular}{llcccccccccc}
\toprule
 & & \multicolumn{10}{c}{Tree Dataset} \\
\cmidrule(lr){3-12}
Model & Class & Deg.\,$\downarrow$ & Clus.\,$\downarrow$ & Orbit\,$\downarrow$ & Spec.\,$\downarrow$ & Wavelet\,$\downarrow$ & Ratio\,$\downarrow$ & Valid\,$\uparrow$ & Unique\,$\uparrow$ & Novel\,$\uparrow$ & V.U.N.\,$\uparrow$ \\
\midrule
\textit{Train set} & --- & 0.0001 & 0.0000 & 0.0000 & 0.0075 & 0.0030 & 1.0 & 100 & 100 & 0.0 & 0.0 \\
\midrule
GraphRNN \citep{graph_rnn}                          & Autoregressive & --- & --- & --- & --- & --- & 607.0 & --- & --- & --- & 0.0 \\
GRAN \citep{gran}                                     & Autoregressive & 0.1884 & 0.0080 & 0.0199 & 0.2751 & 0.3274 & 607.0 & 0.0   & 100  & 100  & 0.0 \\
DiGress \citep{digress}                             & Diffusion      & 0.0002 & 0.0000 & 0.0000 & 0.0113 & 0.0043 & \textbf{1.6} & 90.0 & 100 & 100 & 90.0 \\
EDGE \citep{edge}                                   & Diffusion      & 0.2678 & 0.0000 & 0.7357 & 0.2247 & 0.4230 & 850.7 & 0.0   & 7.5  & 100  & 0.0 \\
BwR \citep{bwr}                                     & Diffusion      & 0.0016 & 0.1239 & 0.0003 & 0.0480 & 0.0388 & 11.4  & 0.0   & 100  & 100  & 0.0 \\
BiGG \citep{bigg}                                     & Autoregressive & 0.0014 & 0.0000 & 0.0000 & 0.0119 & 0.0058 & 5.2   & 100   & 87.5 & 50.0 & 75.0 \\
GraphGen \citep{graphgen}                             & Autoregressive & 0.0105 & 0.0000 & 0.0000 & 0.0153 & 0.0122 & 33.2  & 95.0  & 100  & 100  & 95.0 \\
HSpectre (one-shot) \citep{tree}                  & Diffusion      & 0.0004 & 0.0000 & 0.0000 & 0.0080 & 0.0055 & 2.1   & 82.5  & 100  & 100  & 82.5 \\
HSpectre \citep{tree}                             & Autoregressive + Diffusion      & 0.0001 & 0.0000 & 0.0000 & 0.0117 & 0.0047 & 4.0   & 100   & 100  & 100  & 100 \\
DeFoG \citep{defog}                                 & Flow & 0.0002\pmval{0.0001} & 0.0000\pmval{0.0000} & 0.0000\pmval{0.0000} & 0.0108\pmval{0.0028} & 0.0046\pmval{0.0004} & \textbf{1.6}\pmval{0.4} & 96.5\pmval{2.6}  & 100.0\pmval{0.0} & 100.0\pmval{0.0} & 96.5\pmval{2.6} \\
\midrule
GrowHierarchy     & Flow & 0.0006\pmval{0.0002} & 0.0000\pmval{0.0000} & 0.0000\pmval{0.0000} & \textbf{0.0072}\pmval{0.0006} & \textbf{0.0033}\pmval{0.0002} & 2.5\pmval{0.6} & \textbf{100.0}\pmval{0.0} & 100.0\pmval{0.0} & 100.0\pmval{0.0} & \textbf{100.0}\pmval{0.0} \\
\bottomrule
\end{tabular}%
}
\label{tab:tree_full}
\end{table}

\paragraph{Validity--compute trade-off.}
A single-step-count comparison can mask large differences in how a
model degrades when the sampling budget is reduced.
\cref{fig:tree-step-sweep} sweeps the number of Euler steps from
$10$ to $10^{3}$ for both \textsc{GrowHierarchy} and DeFoG, reporting
validity (left) and the Ratio metric (right). The validity panel (left) shows the central effect of a domain-aligned
action space. \textsc{GrowHierarchy} maintains $100\%$ validity across
the entire range, including regimes where the step budget is far
smaller than the number of actions per trajectory. The Ratio metric shows similar results for both methods across the step sizes.
However, DeFoG's validity
falls steeply as the budget shrinks: from
$\sim\!96\%$ at $10^{3}$ steps to $\sim\!8\%$ at $10$ steps. This underlines the advantage that domain-aligned actions can bring in low-sampling regions.
 
\begin{figure}[t]
  \centering
  \includegraphics[width=\linewidth]{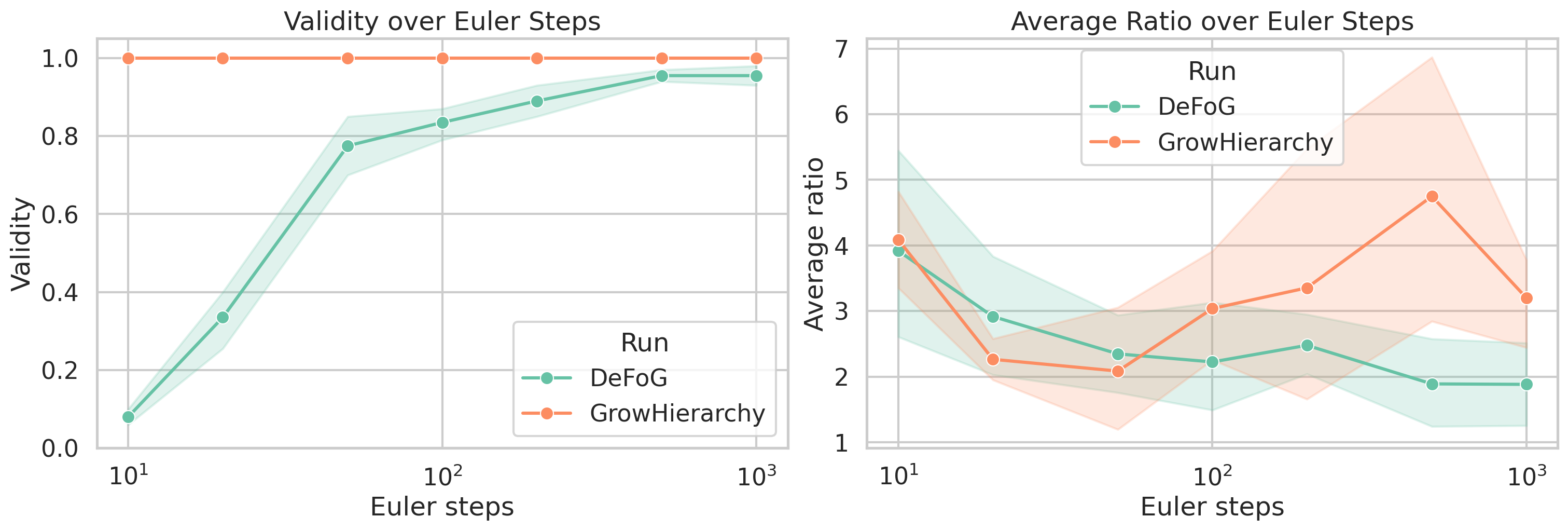}
  \caption{Validity (left) and Ratio (right) on the Tree benchmark
    as a function of the number of Euler integration steps.
    \textsc{GrowHierarchy} preserves $100\%$ validity across three
    orders of magnitude of step budgets, while DeFoG's validity
    collapses below $\sim\!10\%$ at $10$ steps. Average Ratio metric is comparable for both.}
  \label{fig:tree-step-sweep}
\end{figure}
 
\subsection{Molecular Graph Generation: Extended Results}
\label{app:mol-extended}
 
\paragraph{Qm9 datset.}
We include additional results on the QM9 dataset in \cref{tab:qm9} which shows similar performance trends as seen for MOSES and ZINC in the main text.
 
\newcommand{\std}[1]{$_{\pm #1}$}
\newcommand{\gc}{\cellcolor{gray!15}}

\begin{table}[ht]
\centering
\caption{Comparison of DFM and diffusion molecular graph generative models on QM9 (without hydrogens). Best results in \textbf{bold}, second best \underline{underlined}.}
\label{tab:qm9}
\setlength{\tabcolsep}{4pt}
\renewcommand{\arraystretch}{1.1}
\begin{tabular}{cl cccc}
\toprule
& Model 
& Val.\,$\uparrow$ & R.\,Val.\,$\uparrow$ & Uniq.\,$\uparrow$ & FCD\,$\downarrow$ \\
\midrule
& Training set 
& 99.3 & 99.5 & 99.2 & 0.03 \\
\midrule
\multirow{6}{*}{\rotatebox[origin=c]{90}{Fixed-size}}
& DiGress~\citep{digress} 
& 99.0\std{0.0} & -- & 96.2\std{0.1} & -- \\
& DisCo~\citep{disco} 
& 99.3\std{0.6} & -- & -- & -- \\
& Cometh~\citep{cometh} 
& \underline{99.6\std{0.1}} & -- & \underline{96.8\std{0.2}} & 0.25\std{0.01} \\
& GBD~\citep{gbd} 
& \textbf{99.8} & -- & -- & \underline{0.093} \\
& DeFoG~\citep{defog} 
& 99.3\std{0.0} & \underline{99.4\std{0.1}} & 96.3\std{0.3} & 0.12\std{0.00} \\
& \gc BondEdit 
& \gc 99.0\std{0.0} & \gc 99.2\std{0.0} & \gc 96.6\std{0.0} & \gc \textbf{0.085\std{0.00}} \\
\midrule
\multirow{4}{*}{\rotatebox[origin=c]{90}{Variable-size}}
& GrIDDD~\citep{griddd}
& 97.0 & -- & 96.2 & 0.11 \\
& \gc MolEdit 
& \gc 98.8\std{0.1} & \gc 99.2\std{0.1} & \gc 96.0\std{0.1} & \gc 0.17\std{0.01} \\
& \gc AtomGrow
& \gc 98.8\std{0.0} & \gc 99.2\std{0.0} & \gc \textbf{97.0\std{0.0}} & \gc 0.096\std{0.00} \\
& \gc FragGrow
& \gc 99.3\std{0.0} & \gc \textbf{99.7\std{0.0}} & \gc 96.3\std{0.0} & \gc 0.14\std{0.01} \\
\bottomrule
\end{tabular}
\end{table}

\paragraph{Autoregressive comparison on MOSES.}
We additionally compare with GraphINVENT~\citep{graphinvent}, an
autoregressive molecular graph generator. As shown in
\cref{tab:moses-autoregressive}, \textsc{FragGrow} achieves higher
validity ($99.2\%$ versus $96.4\%$) and passes molecular filters more
frequently ($99.3\%$ versus $95.0\%$), while achieving comparable
FCD ($1.21$ versus $1.22$).

\begin{table}[htbp]
\centering
\small
\setlength{\tabcolsep}{5pt}
\caption{Autoregressive comparison on MOSES. GraphINVENT results are
taken from DiGress~\citep{digress}.
Results for all flow-matching/diffusion methods in
\cref{tab:mol-comparison}. }
\label{tab:moses-autoregressive}
\begin{tabular}{lrrrrrr}
\toprule
Model & Validity $\uparrow$ & Uniqueness $\uparrow$ & Filters $\uparrow$
& FCD $\downarrow$ & SNN $\uparrow$ & Scaffold $\uparrow$ \\
\midrule
GraphINVENT & 96.4 & 99.8 & 95.0 & 1.22 & 0.54 & 12.7 \\
\textsc{FragGrow} & 99.2 & 99.9 & 99.3 & 1.21 & 0.55 & 13.7 \\
\bottomrule
\end{tabular}
\end{table}

\paragraph{Sampling-budget efficiency on ZINC.}
Unlike the Tree setting, the molecular action space does not yield
validity by construction: even \textsc{FragGrow} can produce invalid
molecules (e.g.\ valence violations). Nevertheless,
\cref{fig:zinc-step-sweep} shows that a domain-aligned action
space still pays off in the low-budget regime. At $5$ Euler steps,
\textsc{FragGrow} retains $\sim\!78\%$ validity against $\sim\!48\%$
for the factorized baseline \textsc{BondEdit}; at $10$ steps the gap narrows
to $89\%$ vs.\ $84\%$, and both saturate at $\geq\!98\%$ from $25$
steps onward. Each fragment action commits an entire ring system in
a single step, so the same molecule is assembled in far fewer
actions making \textsc{FragGrow} especially attractive whenever sampling
cost matters, e.g.\ as a subroutine in downstream search or
screening loops.

\begin{figure}[t]
  \centering
  \includegraphics[width=0.65\linewidth]{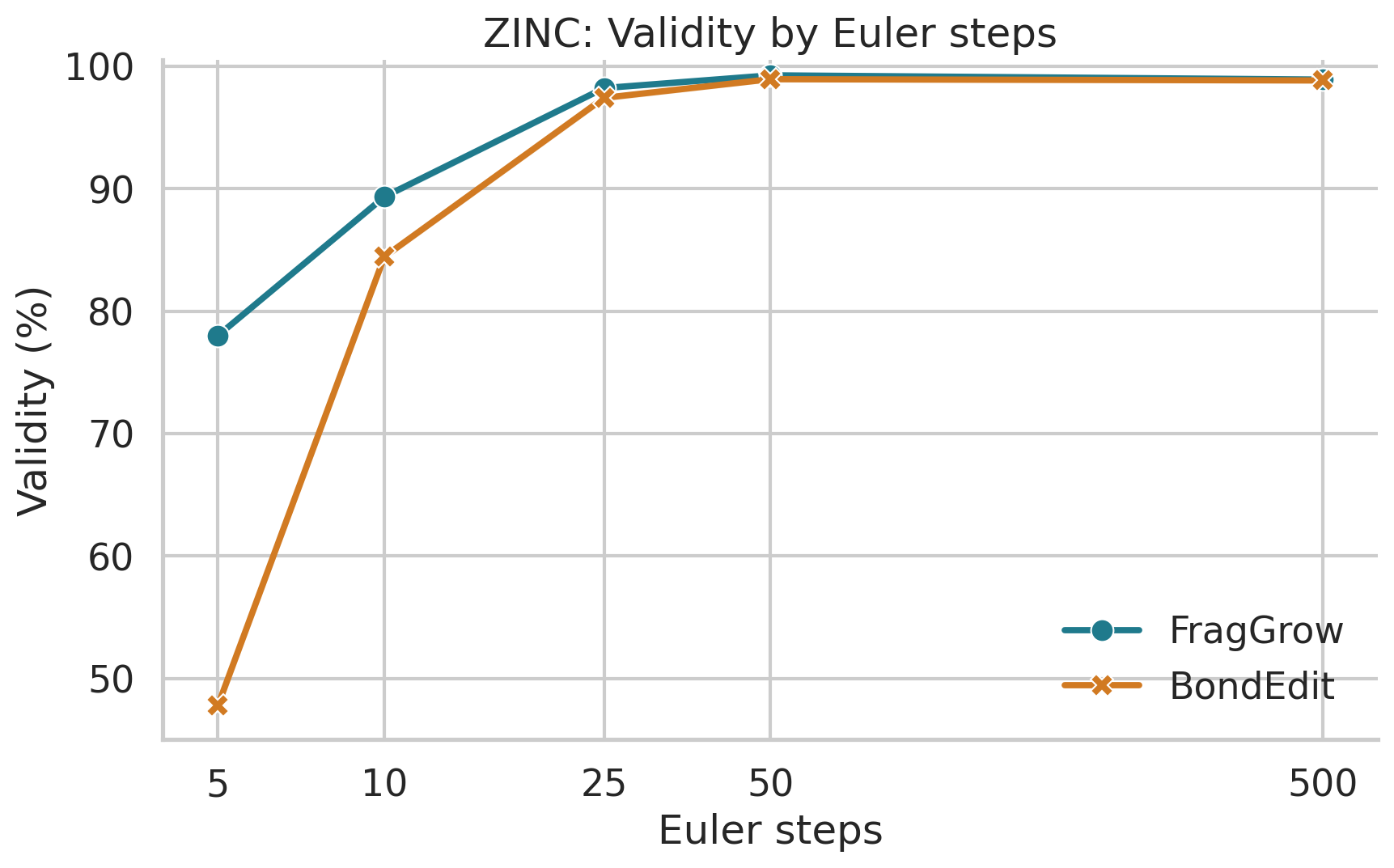}
  \caption{ZINC validity as a function of the number of Euler steps
    for the fragment-based \textsc{FragGrow} and the factorized
    \textsc{BondEdit}. \textsc{FragGrow} maintains substantially
    higher validity in the low-step regime (e.g., $78\%$ vs.\ $48\%$
    at $5$ steps); both methods saturate at $\geq\!98\%$ for
    $\geq\!25$ steps.}
  \label{fig:zinc-step-sweep}
\end{figure}
 
\subsection{Trajectory-Conditioned vs.\ Rao-Blackwellized Targets}
\label{app:rb-ablation}
 
Trajectory matching admits two training targets: the
trajectory-conditioned rate $\rate(\cdot \mid \trajectory)$ defined in
\cref{eq:trajectory-rate}, and its Rao-Blackwellized counterpart
$\rbrate(\cdot \mid \initState, \finalState)$ obtained by marginalizing over
trajectories consistent with the current state and endpoints. \cref{prop:rao-blackwell-consistency} guarantees that
the two targets share the same optimum and expected gradient, while
\cref{prop:rao-blackwell-variance} predicts that the
Rao-Blackwellized estimator has component-wise lower or equal
variance. We test the empirical impact of this variance reduction on
training dynamics.
 
\cref{fig:rb-comparison} compares the two targets on ZINC under
\textsc{AtomGrow}, plotting validity over the first $40$ training
epochs while keeping all other hyperparameters fixed. The Rao-Blackwellized estimator reaches $\geq\!95\%$
validity within the first few epochs and exhibits less
epoch-to-epoch noise; the trajectory-conditioned estimator requires
roughly twice as many epochs to reach the same threshold and remains
 noisier throughout. Both targets converge to comparable final
validity ($\sim\!97$--$99\%$), consistent with
\cref{prop:rao-blackwell-consistency}: the optimum is unchanged, only the
trajectory toward it is smoother. 
 
\begin{figure}[t]
  \centering
  \includegraphics[width=0.65\linewidth]{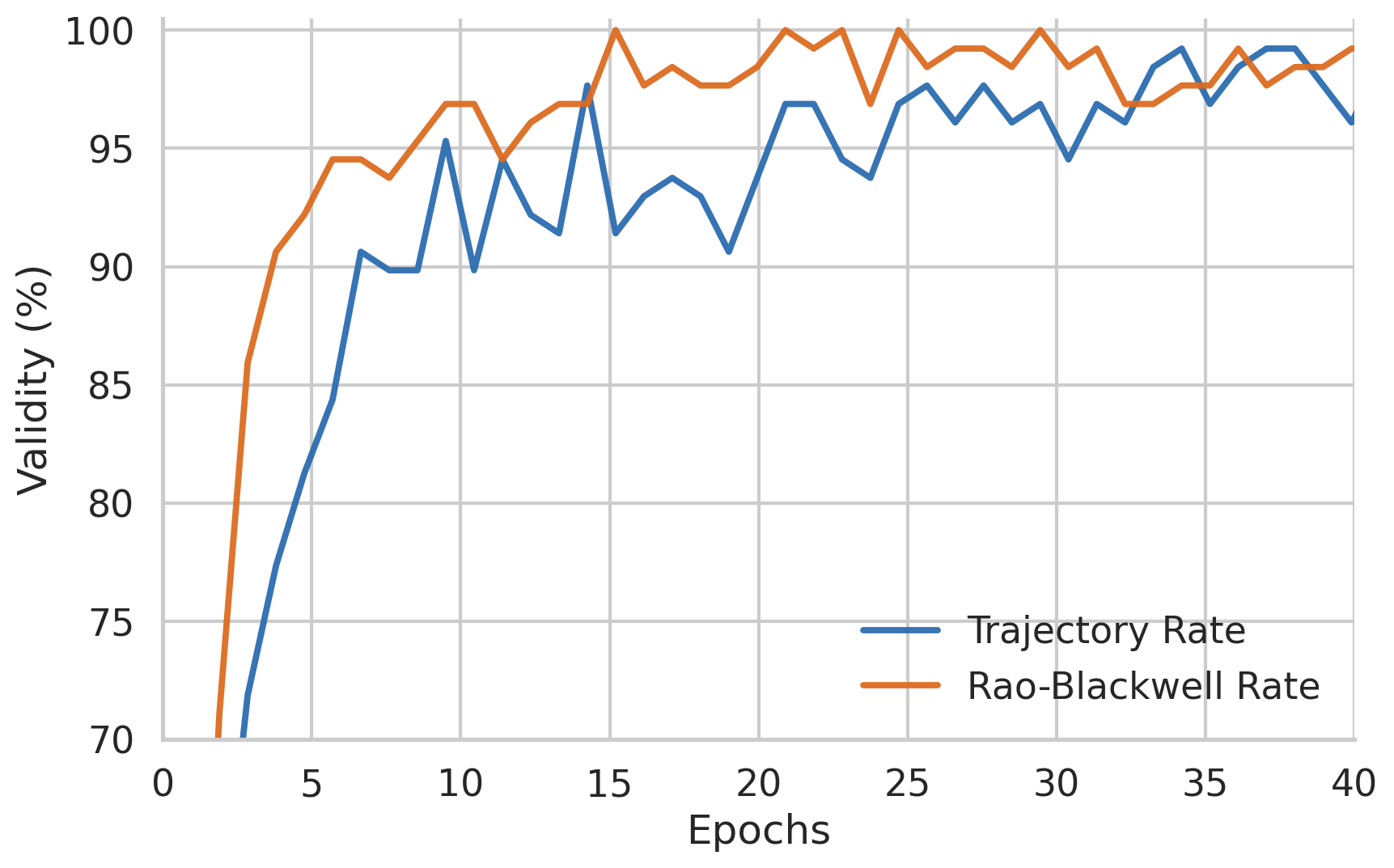}
  \caption{Training dynamics on ZINC under \textsc{AtomGrow}: the
    Rao-Blackwellized target $\bar u_t$ converges
    faster and with lower epoch-to-epoch variance than the
    trajectory-conditioned target $u_t$, while reaching the same
    final validity. The behavior matches the variance-reduction
    guarantee of \cref{prop:rao-blackwell-variance}.}
  \label{fig:rb-comparison}
\end{figure}

\subsection{Effect of the Trajectory Distribution}
\label{app:q-ablation}

We study the effect of the training trajectory distribution $\trajDist$ using
\textsc{AtomGrow} on QM9. We compare three policies for the order of node
insertions: random, breadth-first (BFS), and depth-first (DFS). These policies
jointly determine the insertion order and the  spanning tree, so the
experiment varies their combined effect through $\trajDist$. All three
variants use single-trajectory training targets. The action space, source
distribution, architecture, optimization settings, and sampling procedure
are held fixed; non-tree edges are treated in the same way across variants.

\begin{table}[htbp]
  \centering
  \small
  \setlength{\tabcolsep}{6pt}
  \caption{Effect of the trajectory distribution on \textsc{AtomGrow} on
    QM9, using single-trajectory training targets for all three policies.
    Results are means and standard deviations over five independent
    sampling runs from one trained model per policy.}
  \label{tab:q-ablation}
  \begin{tabular}{lcccc}
    \toprule
    Policy & Validity (\%) $\uparrow$ & Relaxed validity (\%) $\uparrow$
      & Uniqueness (\%) $\uparrow$ & FCD $\downarrow$ \\
    \midrule
    Random & $98.8 \pm 0.0$ & $99.2 \pm 0.0$ & $97.0 \pm 0.0$ & $0.096 \pm 0.00$ \\
    BFS & $98.9 \pm 0.1$ & $99.2 \pm 0.1$ & $96.6 \pm 0.2$ & $0.092 \pm 0.00$ \\
    DFS & $98.8 \pm 0.1$ & $99.2 \pm 0.1$ & $96.8 \pm 0.2$ & $0.098 \pm 0.00$ \\
    \bottomrule
  \end{tabular}
\end{table}

\Cref{tab:q-ablation} shows similar performance across the three choices
of $\trajDist$, suggesting that \textsc{AtomGrow} is robust to these tested
trajectory policies on QM9.

\subsection{Substructure Extension and Linking}
\label{app:downstream}
 
A practical advantage of variable-size, order-invariant trajectory
matching is that the same trained model can address structured
downstream tasks without retraining or task-specific finetuning. We
illustrate this on two such tasks, substructure extension and
substructure linking, and show that \textsc{AtomGrow} accommodates
them out of the box, while the factorized \textsc{BondEdit} baseline
requires task-specific patches and underperforms substantially.
 
\paragraph{Setup.}
For all downstream evaluations and qualitative examples, we use a single \textsc{AtomGrow} model trained
on ZINC with one root atom with probability $0.5$ and two randomly chosen
root atoms with probability $0.5$ (\cref{app:atomgrow}). The same frozen
checkpoint is used for extension and linking, without downstream supervision
or fine-tuning. In unconditional generation on ZINC, this randomized-root model achieves $98.7\%$ validity, $0.53$ scaffold similarity, and $1.46$ FCD. \textsc{BondEdit} also uses a fixed ZINC checkpoint.
For \textsc{AtomGrow}, both tasks reduce to a choice of
$s_I$: extension corresponds to setting $s_I$ to the given
substructure, and linking corresponds to setting $s_I$ to the union
of two disconnected substructures, both treated as roots of the
action DAG.
 
The factorized \textsc{BondEdit} operates on a fixed atom set. We sample
a number of additional atoms, with their features and the bonds between
them drawn from $\noiseDist$, and combine them with the supplied
substructure(s) to form $\initState$. During generation, each substructure's
internal structure is frozen.

 We sweep the Euler start time
$t \in \{0.2, 0.4, 0.6, 0.8\}$ for both methods, which controls how
much of the trajectory remains and hence the typical extent of
generation. Each evaluation uses one sampling run on the original 500-step
Euler time grid, starting at the selected $t$. We sample $N = 64$ completions
per (substructure, $t$) cell and
report the fraction that are simultaneously chemically valid and
yield a connected molecule. Aggregate scores pool all samples across cells,
equivalently averaging the unrounded cell fractions with equal weight.
 
\paragraph{Substructure extension.}
We evaluate on eleven chemically meaningful seed substructures,
deliberately including substructures that do not appear in any ZINC
training molecule to
test out-of-distribution generalization. \cref{fig:extension-heatmap}
reports valid-and-connected fractions across all seeds and start times.
 
\textsc{AtomGrow} achieves valid-and-connected fractions in the
$0.86$--$1.00$ range across (seed, $t$) cells, including the
out-of-distribution substructures. \textsc{BondEdit}, despite the
same backbone and training procedure, plateaus at $0.14$--$0.70$ ---
well below \textsc{AtomGrow} for every setting. 
 
\begin{figure}[t]
  \centering
  \includegraphics[width=0.85\linewidth]{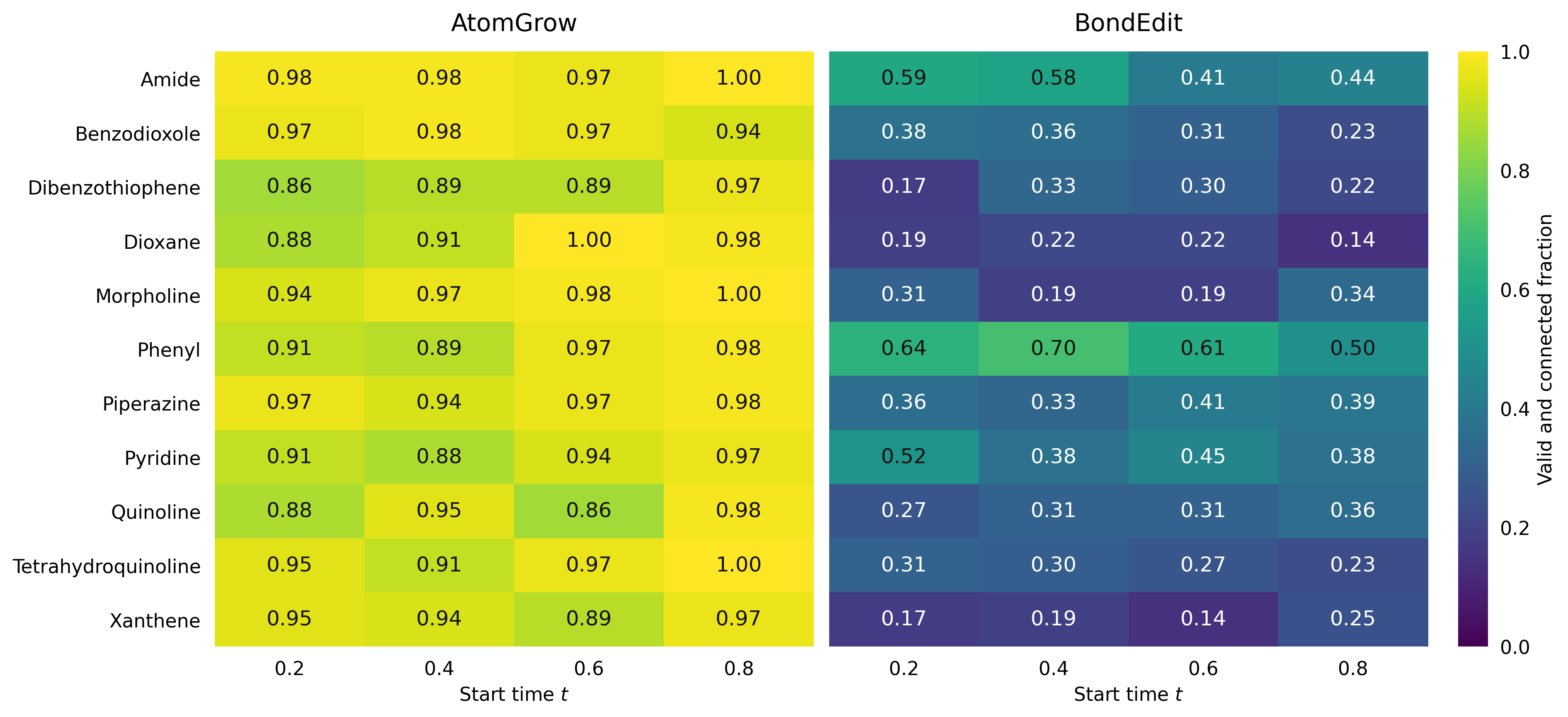}
  \caption{Substructure extension on eleven seed substructures (rows)
    and four Euler start times (columns), reporting the
    valid-and-connected fraction of generated completions.
    \textsc{AtomGrow} (left) accommodates the task by setting $s_I$
    to the seed; \textsc{BondEdit} (right) is adapted by sampling a
    target size and freezing the seed during generation.
    \textsc{AtomGrow} substantially outperforms the factorized
    baseline across all seeds and start times, including
    substructures absent from ZINC training molecules.}
  \label{fig:extension-heatmap}
\end{figure}
 
Beyond the aggregate fractions, the start time provides a natural,
training-free knob over the size of the extension.
\cref{fig:extension-grid} shows valid completions sampled by
\textsc{AtomGrow} for two representative seeds (dibenzothiophene and
quinoline) at four start times: at $t=0.2$ many atoms remain to be
added and the model produces large, decorated extensions; at
$t=0.8$ only a few atoms are added and completions stay close to
the seed (e.g.\ small substituents on the aromatic core). This gives
a single user-controllable lever for trading off seed-faithfulness
against decoration extent --- without retraining or any modification
to the model.
 
\begin{figure}[t]
  \centering
  \subfigure[Dibenzothiophene seed (out-of-distribution: not present in ZINC).]{%
    \includegraphics[width=\linewidth]{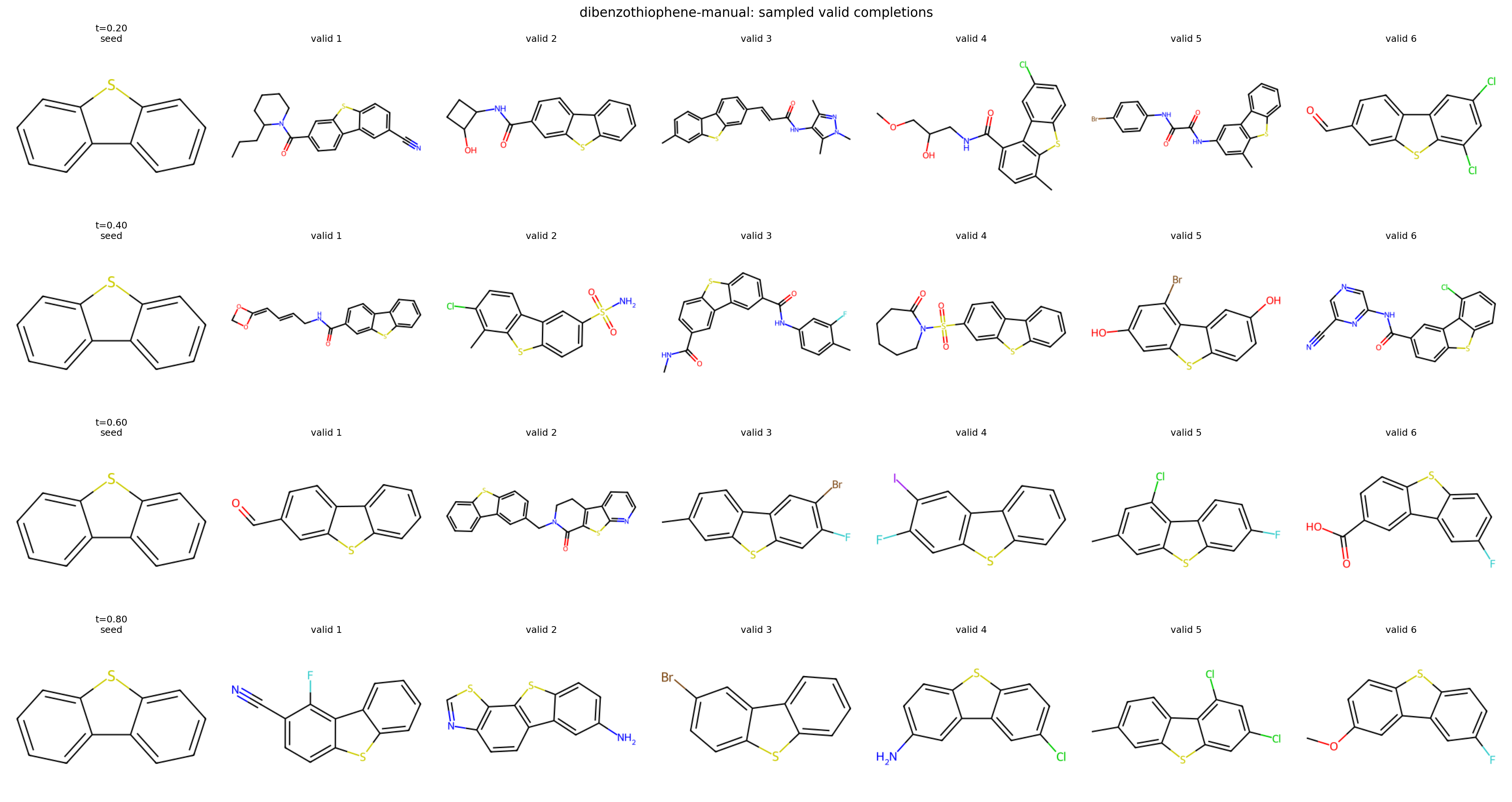}}\\[0.5em]
  \subfigure[Quinoline seed.]{%
    \includegraphics[width=\linewidth]{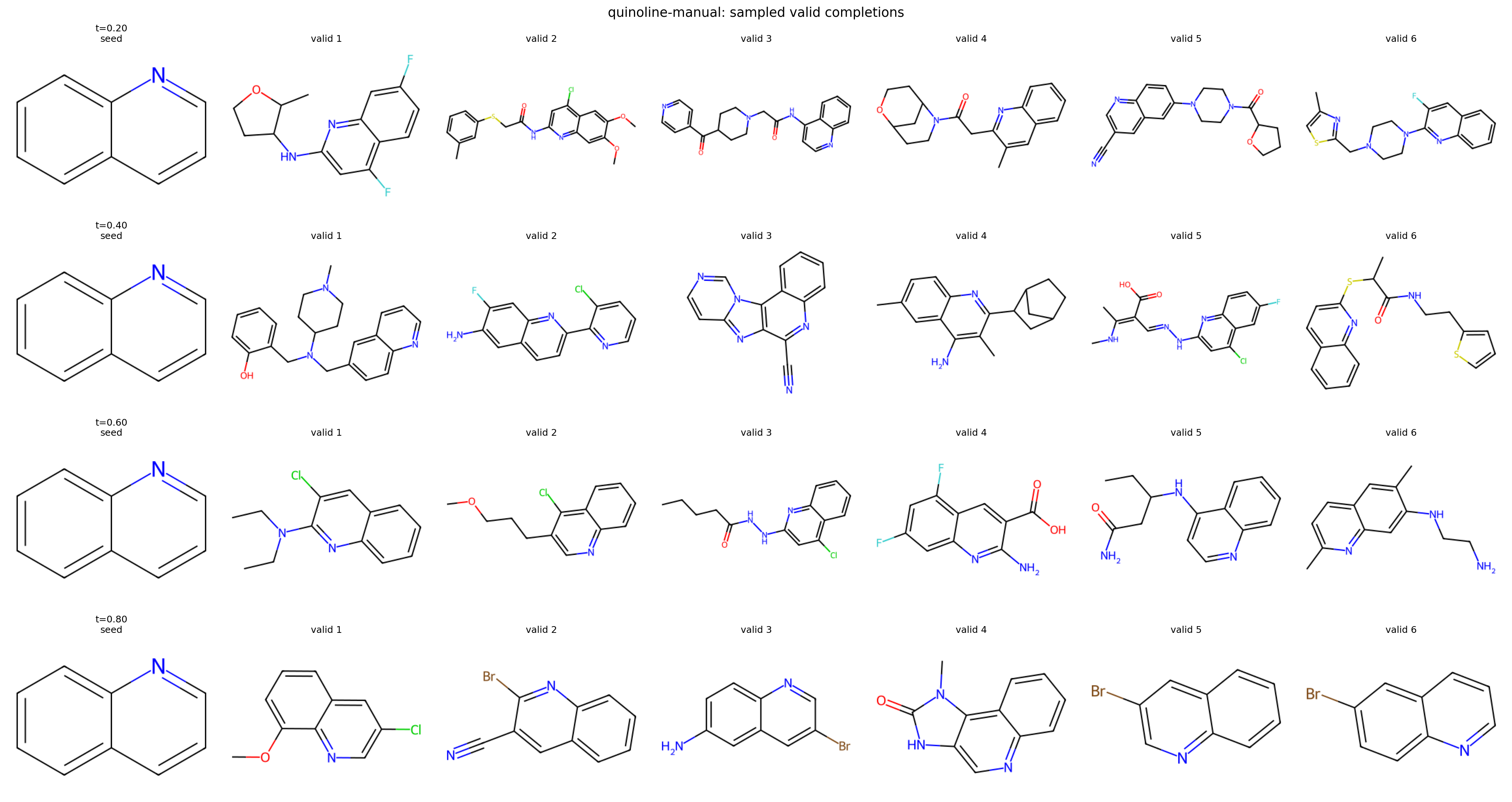}}
  \caption{Valid completions sampled by \textsc{AtomGrow} for two
    seed substructures at four Euler start times $t \in \{0.2, 0.4,
    0.6, 0.8\}$ (one row each). The seed (leftmost column) is
    preserved as a substructure of every completion. Lower $t$
    leaves more of the trajectory remaining and produces larger,
    more elaborately decorated extensions; higher $t$ yields small
    decorations on the seed.}
  \label{fig:extension-grid}
\end{figure}
 
\paragraph{Substructure linking.}
Linking is a strictly harder task: starting from two disjoint
substructures, the model must produce a single molecule containing
both as substructures, which in particular requires forming at
least one connecting path between them. \textsc{AtomGrow}
accommodates this by setting $s_I$ to the union of the two
substructures and treating both as roots of the action DAG; the
remaining atoms and bonds, including any connecting path, are
generated by the standard rollout. \textsc{BondEdit} adapts as in
the extension setting.
 
\cref{fig:linking-heatmap} reports valid-and-connected fractions
on six seed pairs across the same start-time sweep. \textsc{AtomGrow}
attains $0.58$--$0.91$ across all cells and outperforms
\textsc{BondEdit} on every (pair, $t$) combination. For example,
dibenzothiophene + formamide reaches $0.81$ vs.\ $0.33$ at $t=0.8$,
and pyridine + morpholine reaches $0.91$ vs.\ $0.56$ at $t=0.6$.

\begin{figure}[t]
  \centering
  \includegraphics[width=0.85\linewidth]{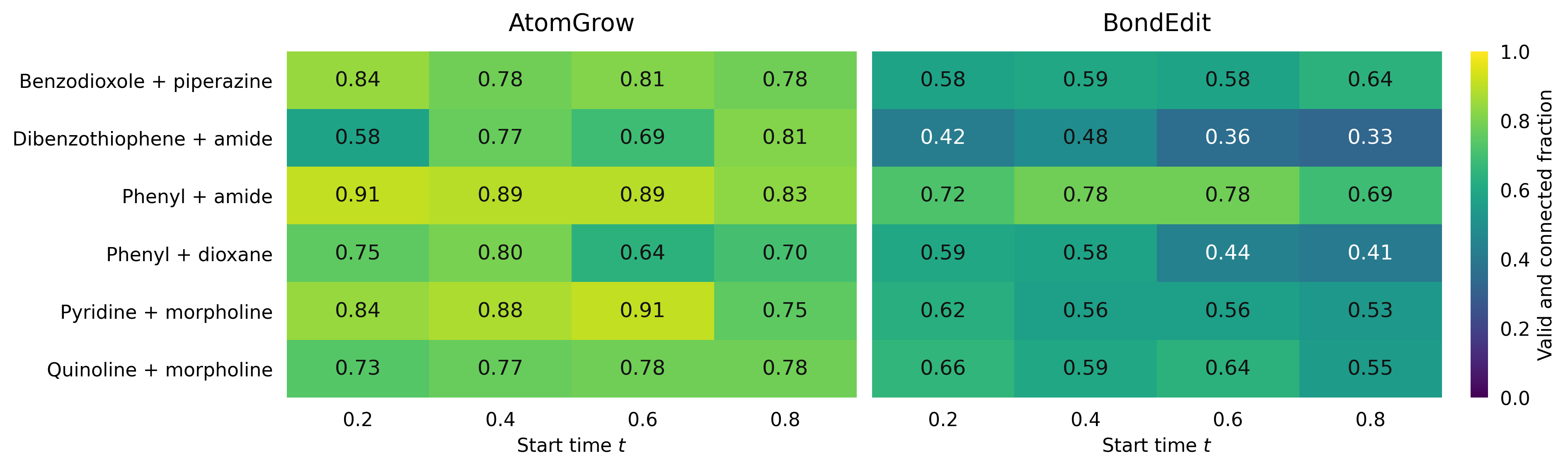}
  \caption{Substructure linking on six seed pairs (rows) and four
    Euler start times (columns), reporting the valid-and-connected
    fraction of generated molecules. \textsc{AtomGrow} (left) sets
    $s_I$ to the union of both seeds and uses both as roots of the
    action DAG; \textsc{BondEdit} (right) is adapted by sampling a
    target size, freezing both seeds, and requiring inserted bonds
    to link them. \textsc{AtomGrow} substantially outperforms the
    factorized baseline across all pairs and start times.}
  \label{fig:linking-heatmap}
\end{figure}
 
\cref{fig:linking-grid} shows qualitative successful linkings
sampled by \textsc{AtomGrow} for two of these seed pairs
(dibenzothiophene + formamide and quinoline + morpholine). Both
seeds are preserved as substructures, and the linker between them
ranges from a single bond or short chain at high $t$ to multi-atom
paths with additional functional groups at low $t$, mirroring
the start-time control observed for extension.
 
\begin{figure}[t]
  \centering
  \subfigure[Dibenzothiophene + formamide.]{%
    \includegraphics[width=\linewidth]{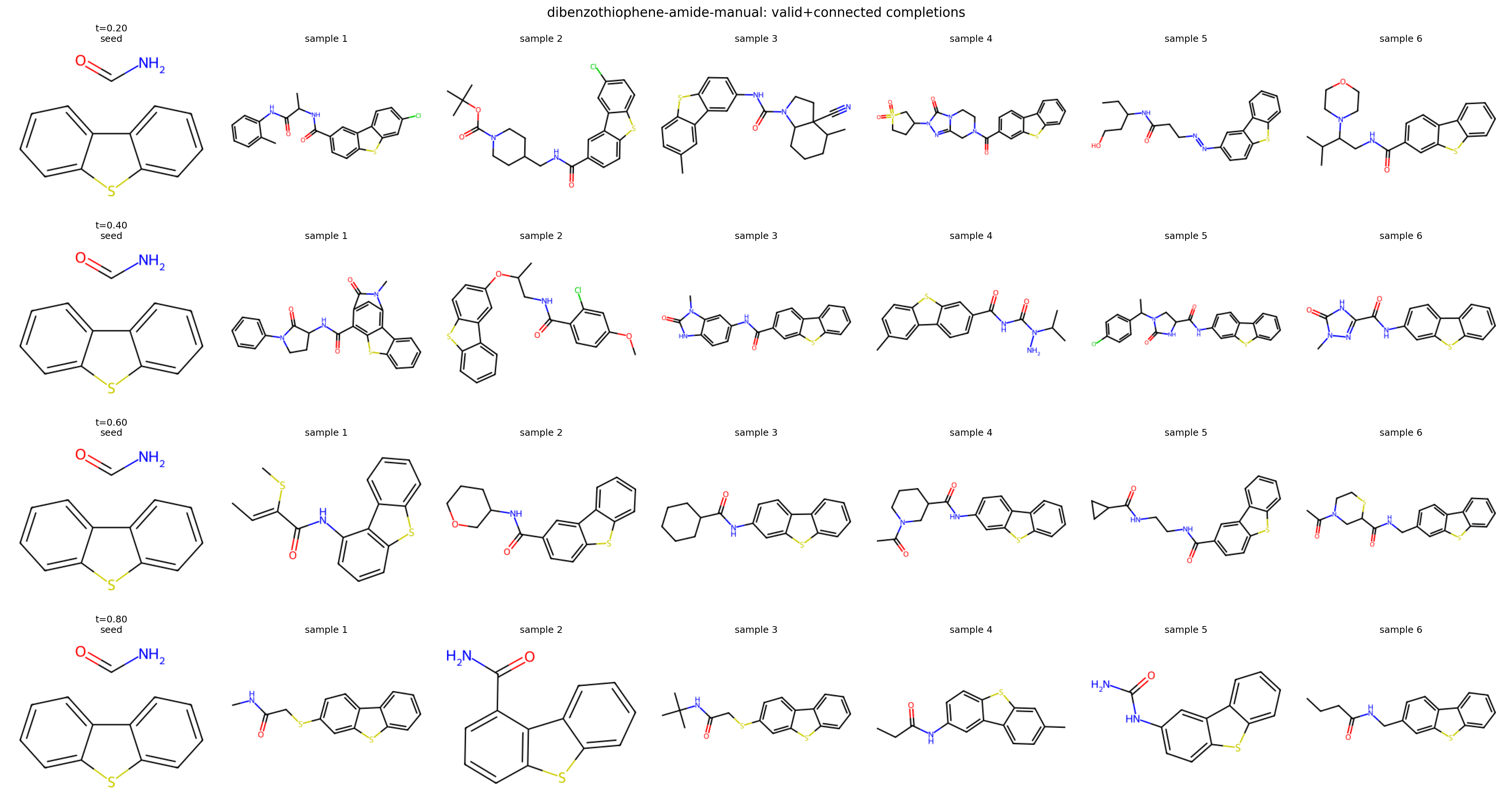}}\\[0.5em]
  \subfigure[Quinoline + morpholine.]{%
    \includegraphics[width=\linewidth]{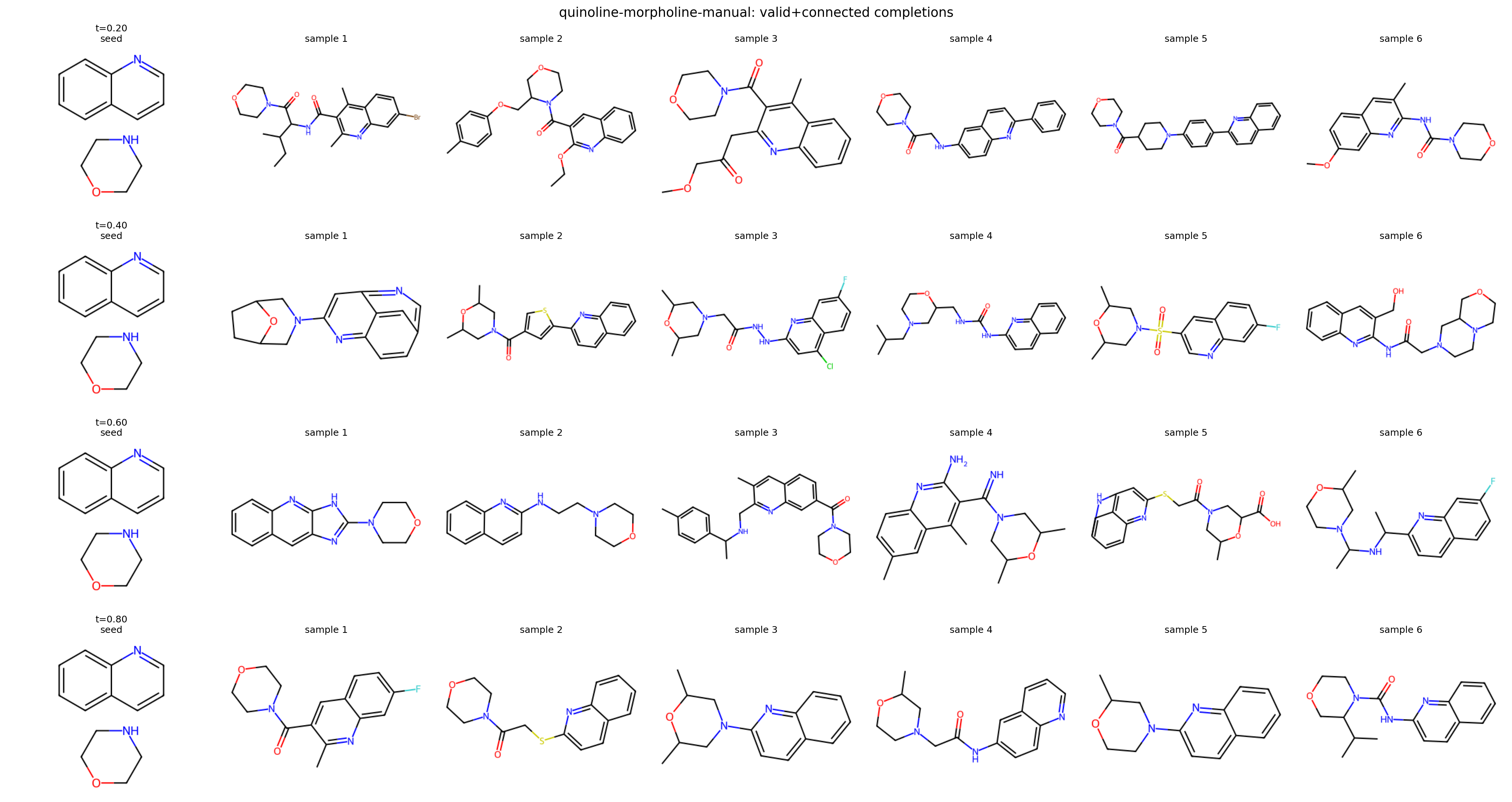}}
  \caption{Valid, connected linkings sampled by \textsc{AtomGrow}
    for two seed pairs at four Euler start times $t \in \{0.2, 0.4,
    0.6, 0.8\}$ (one row each). Both seed substructures (leftmost
    column) appear in every generated molecule, joined by a linker
    whose complexity is controlled by $t$.}
  \label{fig:linking-grid}
\end{figure}

\fi

\end{document}